\documentclass[11pt]{article}
\usepackage{tikz}
\usepackage{graphicx}
\usepackage{amsmath}
\usepackage{amsfonts}
\usepackage{amssymb}
\usepackage{mathrsfs}
\usepackage{fix2col}
\usepackage{latexsym}
\usepackage[font=footnotesize]{subfig}
\usepackage{hyperref}
\usepackage{graphics}
\usepackage{epstopdf}
\usepackage{enumerate}
\usepackage{amsthm}
\usepackage[linesnumbered,ruled,vlined]{algorithm2e}
\usepackage[noend]{algpseudocode}

\newcommand{\vertiii}[1]{{\left\vert\kern-0.25ex\left\vert\kern-0.25ex\left\vert #1 
		\right\vert\kern-0.25ex\right\vert\kern-0.25ex\right\vert}}

\makeatletter

\newtheorem{theorem}{Theorem}
\newtheorem{definition}{Definition}
\newtheorem{lemma}{Lemma}

\newtheorem{remark}{Remark}
\newtheorem{corollary}{Corollary}

\newtheorem{example}{Example}

\usepackage[margin=1in]{geometry}
\usepackage{times}

\usepackage{thmtools}
\usepackage{thm-restate}

\usepackage{listings}

\@ifundefined{showcaptionsetup}{}{%
 \PassOptionsToPackage{caption=false}{subfig}}
\usepackage{subfig}
\makeatother

\usepackage{listings}

\global\long\def\H2{\mathcal{H}_2}

\global\long\def\E1{\mathcal{E}_1}

\begin{document}

\title{\bf Graphical Lasso and Thresholding: Equivalence and Closed-form Solutions}
\author{Salar Fattahi and Somayeh Sojoudi
\thanks{Salar Fattahi is with the Department of Industrial Engineering and Operations Research, University of California, Berkeley. Somayeh Sojoudi is with the Departments of Electrical Engineering and Computer Sciences and Mechanical Engineering as well as the Tsinghua-Berkeley Shenzhen Institute, University of California, Berkeley. This work was supported by the ONR Award N00014-18-1-2526, NSF Award 1808859 and AFSOR Award FA9550-19-1-0055.}}
\date{}
\maketitle

\begin{abstract}
Graphical Lasso (GL) is a popular method for learning the structure of an undirected  graphical model, which is based on an $l_1$ regularization technique. 
{The objective of this paper is to compare the computationally-heavy GL technique  with a numerically-cheap heuristic method that is based on simply thresholding the sample covariance matrix.}
To this end, two notions of sign-consistent and inverse-consistent matrices are developed, and then it is shown that the thresholding and GL methods are equivalent if: (i) the thresholded sample covariance matrix is both sign-consistent and inverse-consistent, and (ii) the gap between the largest thresholded and the smallest un-thresholded entries of the sample covariance matrix is not too small. By building upon this result, it is proved that the GL method---as a conic optimization problem---has an explicit closed-form solution if the thresholded sample covariance matrix has an acyclic structure. This result is then generalized to arbitrary sparse support graphs, where a formula is found to obtain an approximate solution of GL. {Furthermore, it is shown that the approximation error of the derived explicit formula decreases exponentially fast with respect to the length of the minimum-length cycle of the sparsity graph. 
	The developed results are demonstrated on synthetic data, functional MRI data, traffic flows for transportation networks, and massive randomly generated data sets. We show that the proposed method can obtain an accurate approximation of the GL for instances with the sizes as large as $80,000\times 80,000$ (more than 3.2 billion variables) in less than 30 minutes on a standard laptop computer running MATLAB, while other state-of-the-art methods do not converge within 4 hours.}
\end{abstract}

\vspace{2mm}

\section{Introduction}
There  has been a pressing need in developing new and efficient computational methods to analyze and learn the characteristics of high-dimensional data with a structured or randomized nature. Real-world data sets are often overwhelmingly complex, and therefore it is important to obtain a simple description of  the data that can be processed efficiently. In an effort to address this problem, there has been a great deal of interest in {sparsity-promoting} techniques for large-scale optimization problems ~\cite{Coleman90, Bach11, Benson}. These techniques have become essential to the tractability of big-data analyses in many applications, including data mining ~\cite{Garcke01, Muth05, Wu14}, pattern recognition~\cite{Wright10, Qiao10}, human brain functional connectivity~\cite{Sojoudi14}, distributed controller design~\cite{Fardad11, SODC2016}, and compressive sensing~\cite{Candes07, Simon13}. Similar approaches have been used to arrive at a parsimonious estimation of high-dimensional data. However, most of the existing statistical learning techniques in data analytics are contingent upon the availability of a {sufficient} number of samples (compared to the number of parameters), which is difficult to satisfy for many applications~\cite{buhlmann2011statistics, fan2010selective}. To remedy the aforementioned issues, a special attention has been paid to the augmentation of these problems with  sparsity-inducing penalty functions  to obtain sparse and easy-to-analyze solutions. 

Graphical lasso (GL) is one of the most commonly used techniques for estimating the inverse covariance matrix~\cite{friedman2008sparse, banerjee2008model, yuan2007model}.
GL is an optimization problem that shrinks the elements of the inverse covariance matrix towards zero compared to the maximum likelihood estimates, using an $l_1$ regularization.
There is a large body of literature suggesting that the solution of GL is a good estimate for the unknown graphical model, under a suitable choice of the regularization parameter~\cite{friedman2008sparse, banerjee2008model, yuan2007model,liu2010stability, kramer2009regularized, danaher2014joint}. It is known that Graphical Lasso is computationally expensive for large-scale problems. An alternative computationally-cheap heuristic method for estimating graphical models is based on thresholding the sample covariance matrix.


{In this paper, we develop a mathematical framework to analyze the relationship between the GL and thresholding techniques. The paper~\cite{sojoudi2014equivalence111} offers a set of conditions for the equivalence of these two methods, and argues the satisfaction of these conditions in the case where  the regularization coefficient is large or equivalently a sparse graph is sought. 
	Although the conditions derived in~\cite{sojoudi2014equivalence111} shed light on the performance of the GL, they depend on the optimal solution of the GL and cannot be verified without solving the problem. Nonetheless, it is highly desirable to find conditions for the equivalence of the GL and thresholding that are directly  in terms of the sample covariance matrix.} To this end,  two notions of \emph{sign-consistent} and \emph{inverse-consistent} matrices are introduced, and their properties are studied for different types of matrices. It is then shown that the GL and thresholding are equivalent if three conditions are satisfied. The first condition requires a certain matrix formed based on the sample covariance matrix {to have a positive-definite completion}. The second condition requires this matrix to be sign-consistent and inverse-consistent. The third condition needs a separation  between the largest thresholded and the smallest un-thresholded entries of the sample covariance matrix. These conditions can be easily verified for acyclic graphs and are expected to hold for sparse graphs. By building upon  these results,  an explicit closed-form solution is obtained for the GL method in the case where  the thresholded sample covariance matrix has an acyclic support graph. {Furthermore, this result is generalized to sparse support graphs to derive a closed-form formula that can serve either as an approximate solution of the GL or the optimal solution of the GL with a perturbed sample covariance matrix. The approximation error (together with  the corresponding perturbation in the sample covariance matrix) is shown to be related to the lengths of the cycles in the graph. }

The remainder of this paper is organized as follows. The main results are presented in Section~\ref{sec:main}, followed by numerical examples and case studies in Section~\ref{simulations}. Concluding remarks are drawn in Section~\ref{sec:con}. Most of the technical proofs are provided in Appendix.


{\bf Notations:} Lowercase, bold lowercase and uppercase letters are used for scalars, vectors and matrices, respectively (say $x, \mathbf{x}, X$). The symbols $\mathbb{R}^d$, $\mathbb{S}^d$ and $\mathbb{S}^d_+$ are used to denote the sets of $d\times 1$ real vectors, $d\times d$ symmetric matrices and $d\times d$ symmetric  positive-semidefinite matrices, respectively. The notations $\text{trace}(M)$ and $\log\det(M)$ refer to the trace and the logarithm of the determinant of a matrix $M$, respectively. The $(i,j)^{\text{th}}$ entry of the matrix $M$ is denoted by $M_{ij}$. Moreover, $I_d$ denotes the $d\times d$ identity matrix. The sign of a scalar $x$ is shown as $\text{sign}(x)$. The notations $|x|$, $\|M\|_1$ and $\|M\|_F$ denote the absolute value of the scalar $x$, the induced norm-1 and Frobenius norm of the matrix $M$, respectively. The inequalities $M\succeq 0$ and $M\succ 0$ mean that $M$ is positive-semidefinite and positive-definite, respectively. The symbol $\text{sign}(\cdot)$ shows the sign operator. The ceiling function is denoted as $\lceil\cdot\rceil$. The cardinality of a discrete set $\mathcal D$ is denoted as $|\mathcal D|_0$. Given a matrix $M\in\mathbb{S}^d$, define 
\begin{align}\notag
&\|M\|_{1,\mathrm{off}} = \sum_{i=1}^{d}\sum_{j=1}^{d}|M_{ij}|-\sum_{i=1}^{d}|M_{ii}|,\\
&\|M\|_{\max} = \max_{i\not=j}|M_{ij}|.\notag
\end{align}

\vspace{2mm}

\begin{definition}
	Given a symmetric matrix $S\in\mathbb S^d$, the {\bf support graph or sparsity graph} of $S$ is defined as a graph with the vertex set $\mathcal V:=\{1,2,...,d\}$ and the edge set $\mathcal E\subseteq\mathcal V\times \mathcal V$ such that $(i,j)\in\mathcal V$ if and only if $S_{ij}\neq 0$, for every two different vertices $i,j\in\mathcal V$. The support graph of $S$ captures the sparsity pattern of the matrix $S$ and is denoted as $\mathrm{supp}(S)$.
\end{definition}

\vspace{2mm}

\begin{definition}
	Given a graph $\mathcal G$, define $\mathcal G^{(c)}$ as the complement of $\mathcal G$, which is obtained by removing the existing edges of $\mathcal G$ and drawing an edge between every two vertices of $\mathcal G$ that were not originally connected.
\end{definition}

\vspace{2mm}

\begin{definition}
	Given two graphs $\mathcal G_1$ and $\mathcal G_2$ with the same vertex set, $\mathcal G_1$ is called a subgraph of $\mathcal G_2$ if the edge set of $\mathcal G_1$ is a subset of the edge set of $\mathcal G_2$. The notation $\mathcal G_1\subseteq \mathcal G_2$ is used to denote this inclusion.
\end{definition}

{Finally, a symmetric matrix $M$ is said to have a \textbf{positive-definite completion} if there exists a positive-definite $\tilde{M}$ with the same size such that $\tilde{M}_{ij} = {M}_{ij}$ for every $(i,j)\in\mathrm{supp}(M)$.}

\vspace{2mm}


\section{Problem Formulation}

Consider a random vector $\bold x=(x_1,x_2,...,x_d)$ with a  multivariate normal distribution. Let $\Sigma_*\in\mathbb S^d_+$ denote the covariance matrix associated with the vector $\bold x$. The  inverse of the covariance matrix can be used to determine the conditional independence between the  random variables $x_1,x_2,...,x_d$. In particular, if the $(i,j)^{\text{th}}$ entry of $\Sigma_*^{-1}$  is zero for two disparate indices $i$ and $j$, then $x_i$ and $x_j$ are conditionally independent given the rest of the variables. The graph $\mathrm{supp}\big(\Sigma_*^{-1}\big)$ (i.e., the sparsity graph of $\Sigma_*^{-1}$) represents a graphical model  capturing the conditional independence between the elements of $\bold x$. Assume that $\Sigma_*$ is nonsingular and that $\mathrm{supp}\big(\Sigma_*^{-1}\big)$ is a sparse graph. Finding this graph is cumbersome in practice because the exact covariance matrix $\Sigma_*$ is rarely known. More precisely, $\mathrm{supp}\big(\Sigma_*^{-1}\big)$ should be constructed from a given sample covariance matrix (constructed from $n$ samples), as opposed to $\Sigma_*$. Let $\Sigma$ denote an arbitrary $d\times d$ positive-semidefinite matrix, which is provided as an estimate of $\Sigma_*$. Consider the convex optimization problem
\begin{equation}
\label{eq1_o}
\min_{S\in\mathbb S^d_+} -\log\det(S)+\mathrm{trace}(\Sigma S).
\end{equation}
It is easy to verify that the optimal solution of the above problem is equal to $S^{\text{opt}}=\Sigma^{-1}$. However, there are two issues with this solution. First,  since the number of  samples available in many applications is  small or modest compared to the dimension of $\Sigma$, the matrix  $\Sigma$  is ill-conditioned or even singular. Under such circumstances, the equation $S^{\text{opt}}=\Sigma^{-1}$ leads to large or undefined entries for  the optimal solution of \eqref{eq1_o}. Second, although $\Sigma_*^{-1}$ is assumed to be sparse, a small random difference between  $\Sigma_*$ and $\Sigma$  would make $S^{\text{opt}}$ highly dense.
In order to address the aforementioned issues,  consider the   problem
\begin{equation}
\label{eq1}
\min_{S\in\mathbb S^d_+} -\log\det(S)+\mathrm{trace}(\Sigma S)+\lambda \|S\|_{1,\mathrm{off}},
\end{equation}
where $\lambda\in\mathbb R_+$ is a regularization parameter. This problem is referred to as {\it Graphical Lasso} (GL). Intuitively, the  term $\|S\|_{1,\mathrm{off}}$ in the objective function serves as a surrogate for promoting  sparsity among the off-diagonal entries of $S$, while ensuring that the problem is well-defined even with a singular input $\Sigma$. Henceforth, the notation $S^{\text{opt}}$ will be used to denote a solution of the GL instead of the unregularized optimization problem~\eqref{eq1_o}. 

Suppose that it is known {\it a priori} that the true graph $\mathrm{supp}\big(\Sigma_*^{-1}\big)$ has $k$ edges, for some given number $k$. With no loss of generality, assume that all nonzero off-diagonal entries of $\Sigma$ have different magnitudes. Two heuristic methods for finding an estimate of $\mathrm{supp}\big(\Sigma_*^{-1}\big)$ are as follows:
\begin{itemize}
	\item {\bf Graphical Lasso:} { We solve the optimization problem~\eqref{eq1} repeatedly for different values of $\lambda$ until a solution $S^{\text{opt}}$ with exactly $2k$ nonzero off-diagonal entries are found.}
	\item {\bf Thresholding:} Without solving any optimization problem, we simply identify those $2k$ entries of $\Sigma$ that have the largest magnitudes among all off-diagonal entries of $\Sigma$. {We then replace the remaining $d^2-d-2k$ off-diagonal entries of $\Sigma$ with zero and denote the thresholded sample covariance matrix as  $\Sigma_k$}. Note that $\Sigma$ and $\Sigma_k$ have the same diagonal entries.  Finally, we consider the sparsity graph of $\Sigma_k$, namely  $\mathrm{supp}(\Sigma_k)$, as an estimate for $\mathrm{supp}\big(\Sigma_*^{-1}\big)$.
\end{itemize}

\vspace{2mm}

\begin{definition}
	It is said that the {\bf sparsity structures of Graphical Lasso and thresholding are equivalent} if there exists a regularization coefficient $\lambda$ such that $\mathrm{supp}(S^{\text{opt}})=\mathrm{supp}(\Sigma_k)$.
\end{definition}

\vspace{2mm}

Recently, we have verified in several simulations that the GL and thresholding are equivalent for electrical circuits and functional MRI data of 20 subjects, provided  that $k$ is on the order of $n$~\cite{sojoudi2014equivalence111}. This implies that a simple thresholding technique would obtain the same sparsity structure as the computationally-heavy GL technique. {In this paper, it is aimed to understand under what conditions the easy-to-find graph $\mathrm{supp}(\Sigma_k)$ is equal to the hard-to-obtain graph $\mathrm{supp}(S^{\text{opt}})$, without having to solve the GL.} 
Furthermore, we will show that the GL problem has a simple closed-form solution that can be easily derived merely based on the thresholded sample covariance matrix, provided that its underlying graph has an acyclic structure. This result will then be  generalized to obtain an {approximate} solution for the GL in the case where the thresholded sample covariance matrix has  an arbitrary sparsity structure. This closed-form solution converges to the exact solution of the GL as the length of the minimum-length cycle in the support graph of  the thresholded sample covariance matrix grows. 
The derived closed-form solution can be used for two purposes: (1) as a surrogate to the exact solution of the computationally heavy GL problem, and  (2)  as an initial point for common numerical algorithms to numerically solve the GL (see \cite{friedman2008sparse, Hsieh14}).
The above results unveil fundamental properties of the GL in terms of sparsification and  computational complexity. Although conic optimization problems almost never benefit from an exact or inexact explicit formula for their solutions and should be solved numerically, the formula obtained in this paper suggests that sparse GL and related graph-based conic optimization problems may fall into the category of  problems with closed-form solutions  (similar to least squares problems). 

\section{Main Results} \label{sec:main}

In this section, we present the main results of the paper. In order to streamline the presentation, most of the technical proofs are postponed to Appendix. 

\subsection{Equivalence of GL and Thresholding}

{In this subsection, we derive sufficient conditions to guarantee that the GL and thresholding methods result in the same sparsity graph. These conditions are only dependent on $\lambda$ and $\Sigma$, and are expected to hold whenever $\lambda$ is large enough or a sparse graph is sought.}

\vspace{2mm}

\begin{definition} \label{def:dd1}
	
	A  matrix $M\in\mathbb S^d$ is called {\bf inverse-consistent} if there exists a matrix $N\in\mathbb S^d$ with zero diagonal elements such that
	\begin{subequations}
		\begin{align}\notag
		&M+N\succ 0,\\
		&\mathrm{supp}(N)\subseteq\left(\mathrm{supp}(M)\right)^{(c)},\notag\\
		&  \mathrm{supp}\left((M+N)^{-1})\right)\subseteq\mathrm{supp}(M).\notag
		\end{align}
	\end{subequations}
	The matrix $N$ is called {\bf inverse-consistent complement } of $M$ and is denoted as $M^{(c)}$.
	
\end{definition}

\vspace{2mm}

{The next Lemma will shed light on the definition of inverse-consistency by introducing an important class of such matrices that satisfy this property, {namely the \textit{set of matrices with positive-definite completions}.} 
	
	\vspace{2mm}
	
	\begin{lemma} \label{lemma:ll1}
		{Any arbitrary matrix with positive-definite completion} is inverse-consistent and has a unique inverse-consistent complement.
	\end{lemma}
	
	{\noindent \textit{\bf Proof:} 
		Consider the optimization problem
		\begin{subequations}
			\label{eq_p60}
			\begin{align}
			\label{eq_p60a}
			\min_{S\in\mathbb S^n} \qquad &\mathrm{trace}(MS)-\mathrm{logdet}(S)\\
			\label{eq_p6b}
			\text{subject to}\quad & S_{ij}=0,\qquad \forall (i,j)\in (\mathrm{supp}(M))^{(c)}\\
			& S\succeq 0,
			\end{align}
		\end{subequations}
		and its dual
		\begin{subequations}
			\label{eq_p601}
			\begin{align}
			\max_{\Pi\in\mathbb S^n} \qquad &\text{det}(M+\Pi)\\
			\text{subject to}\quad & M+\Pi\succeq 0 \\
			& \rm{supp}(\Pi)\subseteq (\mathrm{supp}(M))^{(c)}\\
			& \Pi_{ii}=0,\qquad i=1,...,d.
			\end{align}
		\end{subequations}
		Note that $\Pi_{ij}$ is equal to the Lagrange multiplier for~\eqref{eq_p6b} and every $(i,j)\in  (\mathrm{supp}(M))^{(c)}$, and is zero otherwise. 
		{Since the matrix $M$ has a positive-definite completion, the dual problem is strictly feasible.} Moreover, $S=I_d$ is a feasible solution of~\eqref{eq_p60}. Therefore, strong duality holds and the primal solution is attainable. On the other hand, the objective function~\eqref{eq_p60a} is strictly convex, which makes the solution of the primal problem unique. Let $S^{\text{opt}}$ denote the globally optimal solution of~\eqref{eq_p60}. It follows from the first-order optimality conditions that
		\begin{equation}\notag
		S^{\text{opt}}=(M+\Pi^{\text{opt}})^{-1}.
		\end{equation}
		This implies that
		\begin{subequations}
			\begin{align}
			&\mathrm{supp}(\Pi^{\text{opt}})\subseteq(\mathrm{supp}(M))^{(c)}\notag\\
			& \mathrm{supp}((M+\Pi^{\text{opt}})^{-1})\subseteq\mathrm{supp}(M)\notag\\
			&M+\Pi^{\text{opt}}\succ 0.\notag
			\end{align}
		\end{subequations}
		As a result, $M\in\mathbb S^d$  is inverse-consistent and $\Pi^{\text{opt}}$ is its complement. To prove the uniqueness of the inverse-consistent complement of $M$, let $\Pi$ denote an arbitrary complement of $M$. It follows from Definition~\ref{def:dd1} and the first-order optimality conditions that $(M+\Pi)^{-1}$ is a solution of~\eqref{eq_p60}. Since $S^{\text{opt}}$  is the unique solution of~\eqref{eq_p60}, it can be concluded that $\Pi=\Pi^{\text{opt}}$. This implies that $M$ has a unique inverse-consistent complement.~\hfill$\blacksquare$}
	
	\vspace{2mm}
	
	{\begin{remark}
			Two observations can be made based on Lemma~\ref{lemma:ll1}. First, the positive-definiteness of a matrix is sufficient to guarantee that it belongs to the cone of matrices with positive-definite completion. Therefore, positive-definite matrices are inverse-consistent. Second, upon existence, the inverse-consistent complement of a matrix with positive-definite completion is equal to the difference between the matrix and its unique \textbf{maximum determinant completion}. 
	\end{remark}}
	
	\vspace{2mm}
	
	\begin{definition}  \label{def:dd2}
		
		{An inverse-consistent matrix $M$} is called {\bf sign-consistent} if the $(i,j)$ entries of $M$ and $(M+M^{(c)})^{-1}$ are nonzero and have opposite signs for every $(i,j)\in\mathrm{supp}(M)$.
		
	\end{definition}
	
	\vspace{2mm}

	\begin{example}[{\bf An inverse- and sign-consistent matrix}] {\rm To illustrate Definitions~\ref{def:dd1} and \ref{def:dd2}, consider the matrix
			\begin{equation}\notag
			M=\left[\begin{array}{cccc}
			1& 0.3 &0  & 0\\
			0.3 & 1 & -0.4 &0\\
			0 & -0.4 & 1 & 0.2\\
			0 & 0 & 0.2 & 1
			\end{array}\right].
			\end{equation}
			The graph $\mathrm{supp}(M)$ is a path graph with the vertex set $\{1,2,3,4\}$ and the edge set $\{(1,2),(2,3),(3,4)\}$. To show that $M$ is inverse-consistent, let the matrix $M^{(c)}$ be chosen as
			\begin{equation}\notag
			M^{(c)}=\left[\begin{array}{cccc}
			0& 0 &-0.120  & -0.024\\
			0& 0& 0 &-0.080\\
			-0.120 & 0 & 0 & 0\\
			-0.024 & -0.080 & 0 &0
			\end{array}\right].
			\end{equation}
			The inverse matrix $(M+M^{(c)})^{-1}$ is equal to
			\begin{equation}\notag
			\left[\begin{array}{cccc}
			\frac{1}{0.91}& \frac{-0.3}{0.91} &0  & 0\\
			\frac{-0.3}{0.91} & 1+\frac{0.09}{0.91}+\frac{0.16}{0.84} & \frac{0.4}{0.84} &0\\
			0 & \frac{0.4}{0.84} & 1+\frac{0.16}{0.84}+\frac{0.04}{0.96}  & \frac{-0.2}{0.96}\\
			0 & 0 & \frac{-0.2}{0.96} & \frac{1}{0.96}
			\end{array}\right].
			\end{equation}
			Observe that: 
			\begin{itemize}
				\item $M$ and $M+M^{(c)}$ are both positive-definite.
				\item The sparsity graphs of $M$ and $M^{(c)}$ are complements of each other.
				\item The sparsity graphs of $M$ and $(M+M^{(c)})^{-1}$ are identical.
				\item  The nonzero off-diagonal entries of $M$ and $(M+M^{(c)})^{-1}$ have opposite signs. 
			\end{itemize}
			The above  properties imply that $M$ is both  inverse-consistent and sign-consistent, and $M^{(c)}$ is its complement.}
	\end{example}
	
	\begin{definition}
		\label{def:dd3}
		Given a graph $\mathcal G$ and a scalar $\alpha$, define  $\beta(\mathcal G,\alpha)$ as the maximum of $\|M^{(c)}\|_{\max}$ over {all matrices $M$ with positive-definite completions and} with the diagonal entries all equal to 1  such that $\mathrm{supp}(M)=\mathcal G$ and  $\|M\|_{\max}\leq \alpha$. 
	\end{definition}
	
	\vspace{2mm}
	
	{Consider the dual solution $\Pi^{\text{opt}}$ introduced in the proof of Lemma~\ref{lemma:ll1} and note that it is a function of $M$. Roughly speaking, the function $\beta(\mathcal G,\alpha)$ in the above definition provides an upper bound on $\|\Pi^{\text{opt}}\|_{\max}$ {over all matrices $M$ with positive-definite completions and} with the diagonal entries equal to 1  such that $\mathrm{supp}(M)=\mathcal G$ and  $\|M\|_{\max}\leq \alpha$.
		As will be shown later, this function will be used as a \textit{certificate} to verify the optimality conditions for the GL.}
	
	Since $\Sigma_*$ is non-singular and we have a finite number of samples, the elements of the upper triangular part of $\Sigma$ (excluding its diagonal elements) are all nonzero and distinct with probability one. Let $\sigma_1,\sigma_2,...,\sigma_{d(d-1)/2}$ denote the absolute values of those upper-triangular entries such that
	\begin{equation}\notag
	\sigma_1>\sigma_2>...>\sigma_{d(d-1)/2}>0.
	\end{equation}
	
	\begin{definition}
		\label{def:kkka}
		Consider an arbitrary positive regularization parameter $\lambda$ that  does not belong to the discrete set $\{\sigma_1,\sigma_2,...,\sigma_{d(d-1)/2}\}$.  Define the index $k$ associated with $\lambda$ as 
		an integer number satisfying the relation $\lambda \in (\sigma_{k+1},\sigma_{k})$. If $\lambda$ is greater than $\sigma_1$, then $k$ is set to 0.
	\end{definition}
	Throughout this paper, the index $k$  refers to the number introduced in Definition \ref{def:kkka}, which depends on $\lambda$.

	\vspace{2mm}
	
	{\begin{definition}\label{residue}
			Define the {\bf residue of $\Sigma$ relative to $\lambda$ } as a matrix $\Sigma^{\mathrm{res}}(\lambda)\in\mathbb S^d$ such that the $(i,j)$ entry of $\Sigma^{\mathrm{res}}(\lambda)$ is equal to 
			$
			\Sigma_{ij}-\lambda\times \mathrm{sign}(\Sigma_{ij})
			$
			if $i\not=j$ and $|\Sigma_{ij}|>\lambda$, and it is equal to 0 otherwise. Furthermore, define {\bf normalized residue of $\Sigma$ relative to $\lambda$} as
			\begin{equation}\notag
			\tilde{\Sigma}^{\mathrm{res}}(\lambda) = D^{-1/2}\times\Sigma^{\mathrm{res}}(\lambda)\times D^{-1/2},
			\end{equation}
			where $D$ is diagonal matrix with $D_{ii} = \Sigma_{ii}$ for every $i\in\{1,...,d\}$.
	\end{definition}}
	
	\vspace{2mm}
	
	{Notice that $\Sigma^{\mathrm{res}}(\lambda)$ is in fact the soft-thresholded sample covariance matrix with the threshold $\lambda$. }
	For notational simplicity, we will use $\Sigma^{\text{res}}$ or $\tilde{\Sigma}^{\mathrm{res}}$ instead of $\Sigma^{\text{res}}(\lambda)$ or $\tilde{\Sigma}^{\mathrm{res}}(\lambda)$ whenever the equivalence is implied by the context.
	One of the main theorems of this paper is presented below.\vspace{2mm}

	{\begin{theorem} \label{thm:tt1}
			The sparsity structures of the thresholding and GL methods are equivalent if the following conditions are satisfied: 
			\begin{itemize}
				\item {\bf Condition 1-i:} $I_d+\tilde\Sigma^{\mathrm{res}}$ has a positive-definite completion.
				\item {\bf Condition 1-ii:} $I_d+\tilde\Sigma^{\mathrm{res}}$ is sign-consistent.
				\item {\bf Condition 1-iii:} The relation
				\begin{equation}\notag
				\beta\left(\mathrm{supp}(\Sigma^{\mathrm{res}}),\|\tilde{\Sigma}^{\mathrm{res}}\|_{\max}\right)\leq \underset{
					\begin{subarray}{c}
					i\not=j\\
					|\Sigma_{ij}|\leq\lambda
					\end{subarray}
				} {\min}\frac{\lambda-|\Sigma_{ij}|}{\sqrt{\Sigma_{ii}\Sigma_{jj}}}
				\end{equation}
				holds.
			\end{itemize}
	\end{theorem}}
	
	\vspace{2mm}
	
	{A number of observations can be made based on Theorem~\ref{thm:tt1}. First note that, due to Lemma~\ref{lemma:ll1}, Condition~(1-i) guarantees that $I_d+\tilde\Sigma^{\mathrm{res}}$ is inverse-consistent; in fact it holds when $I_d+\tilde\Sigma^{\mathrm{res}}$ itself is positive-definite. Note that the positive-definiteness of $I_d+\tilde\Sigma^{\mathrm{res}}$ is guaranteed to hold if the eigenvalues of the normalized residue of the matrix  $\Sigma$ relative to $\lambda$ are greater than $-1$.} Recall that $\lambda\in(\sigma_{k+1}, \sigma_{k})$ for some integer $k$ and the off-diagonal entries of $I_d+\tilde\Sigma^{\mathrm{res}}$ are in the range $[-1,1]$.} In the case where the number $k$ is significantly smaller than $d^2$, the residue matrix has many zero entries. Hence, the satisfaction of Condition~(1-i) is expected for a large class of residue matrices; { this will be verified extensively in our case studies on the real-world and synthetically generated data sets. 
	Specifically, this condition is automatically satisfied if $I_d+\tilde\Sigma^{\mathrm{res}}$ is diagonally dominant.}
{Conditions~(1-ii) and~(1-iii) of Theorem~\ref{thm:tt1} are harder to check. These conditions depend on the support graph of the residue matrix $\tilde\Sigma^{\mathrm{res}}$ and/or how small the nonzero entries of $\tilde\Sigma^{\mathrm{res}}$ are. The next two lemmas further analyze these conditions to show that they are expected to be satisfied for large $\lambda$.}

\vspace{2mm}

\begin{lemma} \label{thm:tt4}
	Given an arbitrary graph $\mathcal G$, there is a strictly positive constant number $\zeta(\mathcal G)$  such that
	\begin{equation}
	\label{eq_p5}
	\beta(\mathcal G,\alpha)\leq \zeta(\mathcal G) \alpha^2,\qquad \forall\ \alpha\in(0, 1)
	\end{equation}
	{and therefore, Condition (1-iii) is reduced to 
		\begin{equation}\notag
		\zeta(\mathrm{supp}(\Sigma^{\mathrm{res}}))\times\underset{
			\begin{subarray}{c}
			k\not=l\\
			|\Sigma_{kl}|>\lambda
			\end{subarray}
		} {\max}\left(\frac{|\Sigma_{kl}|-\lambda}{\sqrt{\Sigma_{kk}\Sigma_{ll}}}\right)^2\leq \underset{
			\begin{subarray}{c}
			i\not=j\\
			|\Sigma_{ij}|\leq\lambda
			\end{subarray}
		} {\min}\frac{\lambda-|\Sigma_{ij}|}{\sqrt{\Sigma_{ii}\Sigma_{jj}}}.
		\end{equation}}
\end{lemma}

{\begin{lemma}\label{l_sign}
		Consider a { matrix $M$ with a positive-definite completion} and with unit diagonal entries. Define $\alpha = \|M\|_{\max}$ and $\mathcal{G} = \mathrm{supp}(M)$. There exist strictly positive constant numbers $\alpha_0(\mathcal{G})$ and $\gamma(\mathcal{G})$ such that $M$ is sign-consistent if $\alpha\leq\alpha_0(\mathcal{G})$ and the absolute value of the off-diagonal nonzero entries of $M$ is lower bounded by $\gamma(\mathcal{G})\alpha^2$. This implies that  Condition (i-ii) is satisfied if $\|\tilde{\Sigma}^{\mathrm{res}}\|_{\max}\leq\alpha_0(\mathrm{supp}(\Sigma^{\mathrm{res}}))$ and 
		\begin{equation}\label{eq18}
		\gamma(\mathrm{supp}(\Sigma^{\mathrm{res}}))\times\underset{
			\begin{subarray}{c}
			k\not=l\\
			|\Sigma_{kl}|>\lambda
			\end{subarray}
		} {\max}\left(\frac{|\Sigma_{kl}|-\lambda}{\sqrt{\Sigma_{kk}\Sigma_{ll}}}\right)^2\leq \underset{
			\begin{subarray}{c}
			i\not=j\\
			|\Sigma_{ij}|>\lambda
			\end{subarray}
		} {\min}\frac{|\Sigma_{ij}|-\lambda}{\sqrt{\Sigma_{ii}\Sigma_{jj}}}.
		\end{equation}
\end{lemma}}

\vspace{2mm}
{For simplicity of notation, define $r = \frac{\max_{i}\Sigma_{ii}}{\min_{j}\Sigma_{jj}}$ and $\Sigma_{\max} = \max_{i}\Sigma_{ii}$.
	Assuming that $\|\tilde{\Sigma}^{\mathrm{res}}\|_{\max}\leq\alpha_0(\mathrm{supp}(\Sigma^{\mathrm{res}}))$, Conditions~(1-ii) and (1-iii) of Theorem~\ref{thm:tt1} are guaranteed to be satisfied if 
	\begin{equation}\label{eq19}
	\zeta(\mathrm{supp}(\Sigma^{\mathrm{res}}))\leq\frac{1}{r^2}\cdot \frac{ \frac{\lambda-\sigma_{k+1}}{\Sigma_{\max}}}{\left(\frac{\sigma_1-\lambda}{\Sigma_{\max}}\right)^2},\qquad \gamma(\mathrm{supp}(\Sigma^{\mathrm{res}}))\leq\frac{1}{r^2}\cdot \frac{ \frac{\sigma_{k}-\lambda}{\Sigma_{\max}}}{\left(\frac{\sigma_1-\lambda}{\Sigma_{\max}}\right)^2},
	\end{equation}
	which is equivalent to
	\begin{equation}\notag
	\max\left\{\gamma(\mathrm{supp}(\Sigma^{\mathrm{res}})), \zeta(\mathrm{supp}(\Sigma^{\mathrm{res}}))\right\}\leq \frac{2}{r^2}\cdot\frac{ \frac{\sigma_k-\sigma_{k+1}}{\Sigma_{\max}}}{\left(\frac{2\sigma_1-\sigma_k-\sigma_{k+1}}{\Sigma_{\max}}\right)^2}.
	\end{equation}
	for the choice $\lambda=\frac{\sigma_k+\sigma_{k+1}}{2}$. Consider the set
	\begin{equation}\notag
	\mathcal T= \big\{|\Sigma_{ij}|\ \big|\ i=1,2,...,d-1,\ j=i+1,...,d\big\}.
	\end{equation}
	This set has $\frac{d(d-1)}{2}$ elements. The cardinality of $\{\sigma_1,...,\sigma_{d-1}\}$, as a subset of $\mathcal T$, is smaller than the cardinality of $\mathcal T$ by a factor of $\frac{d}{2}$. Combined with the fact that $|\sigma_i|<\Sigma_{\max}$ for every $i = 1,...,\frac{d(d-1)}{2}$, this implies that the term $\frac{2\sigma_1-\sigma_{d-1}-\sigma_{d}}{\Sigma_{\max}}$ is expected to be small and its square is likely to be much smaller than 1, provided that the  elements of $\mathcal T$ are sufficiently spread. If the number $(2\sigma_1-\sigma_{d-1}-\sigma_{d})$ is relatively smaller than the gap $\sigma_{d-1}-\sigma_{d}$ and $k=O(d)$, then \eqref{eq_p5} and as a result Conditions~(1-ii) and~(1-iii) would be satisfied. The satisfaction of this condition will be studied for acyclic graphs in the next section. }

\subsection{Closed-form Solution: Acyclic Sparsity Graphs}

In the previous subsection, we provided a set of sufficient conditions for the equivalence of the GL and thresholding methods. Although these conditions are merely based on the known parameters of the problem, i.e., the regularization coefficient and sample covariance matrix, their verification is contingent upon knowing the value of $\beta(\mathrm{supp}(\Sigma^{\mathrm{res}}),\|\tilde{\Sigma}^{\mathrm{res}}\|_{\max})$ and whether  $I_d+\tilde\Sigma^{\text{res}}$ is sign-consistent {and has a positive-definite completion}. The objective of this part is to greatly simplify the conditions in the case where the thresholded sample covariance matrix has an acyclic support graph. 
{First, notice that if $I_d+\tilde\Sigma^{\text{res}}$ is positive-definite, it has a trivial positive-definite completion}. Furthermore, we will prove that $\zeta(\text{supp}(\Sigma^{\mathrm{res}}))$  in Lemma~\ref{thm:tt4} is equal to 1 when $\text{supp}(\Sigma^{\mathrm{res}})$ is acyclic. {This reduces Condition~(1-iii) to the simple inequality
	\begin{equation}\notag
	\|\tilde{\Sigma}^{\mathrm{res}}\|_{\max}^2\leq \underset{
		\begin{subarray}{c}
		i\not=j\\
		|\Sigma_{ij}|\leq\lambda
		\end{subarray}
	} {\min}\frac{\lambda-|\Sigma_{ij}|}{\sqrt{\Sigma_{ii}\Sigma_{jj}}},
	\end{equation}
	which can be verified efficiently and is expected to hold in practice (see Section~\ref{simulations}).}
Then, we will show that the sign-consistency of $I_d+\tilde\Sigma^{\text{res}}$ is automatically implied {by the fact that it has a positive-definite completion} if  $\text{supp}(\Sigma^{\text{res}})$ is acyclic.

\vspace{2mm}

\begin{lemma}\label{l_beta}
	Given an arbitrary acyclic graph $\mathcal G$, the relation
	\begin{equation}\label{beta}
	\beta(\mathcal{G},\alpha) \leq \alpha^2
	\end{equation}
	holds for every $0\leq\alpha<1$. Furthermore, strict equality holds for \eqref{beta} if $\mathcal{G}$ includes a path of length at least 2.
\end{lemma}

{\noindent{\bf Sketch of the Proof:}
	In what follows, we will provide a sketch of the main idea behind the proof of Lemma~\ref{l_beta}. The detailed analysis can be found in the Appendix. Without loss of generality, one can assume that $\mathcal{G}$ is connected. Otherwise, the subsequent argument can be made for every connected component of $\mathcal{G}$. Consider a matrix $M$ that satisfies the conditions delineated in Definition~\ref{def:dd3}, i.e. 1) {it has a positive-definite completion} and hence, is inverse-consistent (see Lemma~\ref{lemma:ll1}), 2) it has unit diagonal entries, 3) the absolute value of its off-diagonal elements is upper bounded by $\alpha$, and 4) $\mathrm{supp}(M) = \mathcal{G}$. The key idea behind the proof of Lemma~\ref{l_beta} lies in the fact that, due to the acyclic structure of $\mathcal{G}$, one can explicitly characterize the inverse-consistent complement of $M$. In particular, it can be shown that the inverse-consistent complement of $M$ has the following explicit formula: for every $(i,j)\not\in\mathcal{G}$, $M^{(c)}_{ij}$ is equal to the multiplication of the off-diagonal elements of $M$ corresponding to the edges in the unique path between the nodes $i$ and $j$ in $\mathcal{G}$. This key insight immediately results in the statement of Lemma~\ref{l_beta}: the length of the path between nodes $i$ and $j$ is lower bounded by 2 and therefore, $M^{(c)}_{ij}\leq \alpha^2$. Furthermore, it is easy to see that if $\mathcal{G}$ includes a path of length at least 2, $M$ can be chosen such that for some $(i,j)\not\in\mathcal{G}$, we have $M^{(c)}_{ij}= \alpha^2$.~\hfill$\blacksquare$
	
	Lemma~\ref{l_beta} is at the core of our subsequent arguments. It shows that the function $\beta(\mathcal{G},\alpha)$ has a simple and explicit formula since its inverse-consistent complement can be easily obtained. Furthermore, it will be used to derive \textit{approximate} inverse-consistent complement of the matrices with sparse, but not necessarily acyclic support graphs.}

\vspace{2mm}

\begin{lemma} \label{thm:tt2}
	Condition~(1-ii) of Theorem~\ref{thm:tt1} is implied by its Condition~(1-i) if the graph  $\mathrm{supp}(\Sigma^{\mathrm{res}})$ is acyclic.
\end{lemma}

{\noindent{\bf Proof:}
	Consider an { arbitrary matrix $M\in\mathbb S^d$ with a positive-definite completion}. It suffices to show that if $\mathrm{supp}(M)$ is acyclic, then $M$ is  sign-consistent. To this end, 
	consider the matrix $\Pi^{\text{opt}}$ introduced in the proof of Lemma~\ref{lemma:ll1}, which is indeed the unique inverse-consistent complement of $M$. For an arbitrary pair $(i,j)\in\mathrm{supp}(M)$, define a diagonal matrix $\Phi\in\mathbb S^n$ as follows:
	\begin{itemize}
		\item Consider the graph $\mathrm{supp}(M)\backslash\{(i,j)\}$, which is obtained from the acyclic graph $\mathrm{supp}(M)$ by removing its edge $(i,j)$. The resulting graph is disconnected because there is no path between nodes $i$ and $j$. 
		\item Divide the disconnected graph $\mathrm{supp}(M)\backslash\{(i,j)\}$ into two groups 1 and 2 such that group 1 contains node $i$ and group 2 includes node 2.
		\item For every $l\in\{1,...,n\}$, define $\Phi_{ll}$ as 1 if $l$ is in group~1, and as -1 otherwise.   
	\end{itemize}
	In light of Lemma~\ref{lemma:ll1}, $(M+\Pi)^{-1}$  is the unique solution of~\eqref{eq_p60}. Similarly, $\Phi(M+\Pi)^{-1}\Phi$  is a feasible point for~\eqref{eq_p60}. As a result, the following inequality must hold
	\begin{equation}
	\begin{aligned}\notag
	&\bigg\{\mathrm{trace}(M(M+\Pi^{\text{opt}})^{-1})-\mathrm{logdet}((M+\Pi^{\text{opt}})^{-1})\bigg\}\\
	& - \bigg\{\mathrm{trace}(M\Phi(M+\Pi^{\text{opt}})^{-1}\Phi)-\mathrm{logdet}(\Phi(M+\Pi^{\text{opt}})^{-1}\Phi)\bigg\}<0.
	\end{aligned}
	\end{equation}
	It is easy to verify that the left side of the above inequality is equal to twice the product of the $(i,j)$ entries of $M$ and $(M+\Pi)^{-1}$. This implies that the $(i,j)$ entries of $M$ and $(M+\Pi)^{-1}$ have opposite signs. As a result, $M$ is sign-consistent.~\hfill$\blacksquare$}

\vspace{2mm}

\begin{definition}\label{T}
	Define  $T(\lambda)$ as a $d\times d$ symmetric matrix whose $(i,j)^{\text{th}}$ entry is equal to $\Sigma_{ij}+\lambda\times \text{sign}(S^{\text{opt}}_{ij})$ for every $(i,j)\in\text{supp}(S^{\text{opt}}),$ and it is equal to zero otherwise.
\end{definition}

The next result of this paper is a consequence of Lemmas \ref{l_beta} and \ref{thm:tt2} and Theorem~\ref{thm:tt1}. 

%

{\begin{theorem}\label{thm4}
		Assume that the graph $\text{supp}(S^{\text{opt}})$ is acyclic and the matrix $D+T(\lambda)$ is positive-definite. Then, the relation $\mathcal{E}^{\text{opt}}\subseteq \mathcal{E}^{\mathrm{res}}$ holds and the optimal solution $S^{\text{opt}}$ of the GL can be computed via the explicit formula
		\begin{equation}\label{S_opt2}
		S^{\text{opt}}_{ij} = \left\{
		\begin{array}{ll}
		\frac{1}{\Sigma_{ii}}\left(1+\underset{
			(i,m)\in\mathcal{E}^{\text{opt}}
		}{\sum}\frac{({\Sigma^{\mathrm{res}}_{im}})^2}{{\Sigma_{ii}\Sigma_{mm}}-({\Sigma^{\mathrm{res}}_{im}})^2}\right) & \text{if}\quad i=j,\\
		\frac{-\Sigma^{\mathrm{res}}_{ij}}{{\Sigma_{ii}\Sigma_{jj}}-({\Sigma^{\mathrm{res}}_{ij}})^2} & \text{if}\quad(i,j)\in \mathcal{E}^{\text{opt}}\vspace{1mm},\\
		
		0 & \text{otherwise},
		\end{array} 
		\right.
		\end{equation}
		where $\mathcal{E}^{\text{opt}}$ and $\mathcal{E}^{\mathrm{res}}$ denote the edge sets of $\mathrm{supp}(S^{\text{opt}})$ and $\mathrm{supp}(\Sigma^{\mathrm{res}})$, respectively.
\end{theorem}}


When the regularization parameter $\lambda$ is large, the graph $\text{supp}(S^{\text{opt}})$ is expected to be sparse and possibly acyclic. In this case, the matrix  $T(\lambda)$ is sparse with small nonzero entries. If $D+T(\lambda)$ is positive-definite and $\text{supp}(S^{\text{opt}})$ is acyclic, Theorem \ref{thm4} reveals two important properties of the solution of the GL: 1) its support graph is contained in  the sparsity graph of the thresholded sample covariance matrix, and 2) the entries of this matrix  can be found using the explicit formula  \eqref{S_opt2}. However, this formula requires to know  the locations of the nonzero elements of  $S^{\text{opt}}$. 
In what follows, we will replace the assumptions of the above theorem with easily verifiable rules that are independent from the optimal solution $S^{\text{opt}}$ or the locations of its nonzero entries. Furthermore, it will be shown that these conditions are expected to hold when $\lambda$ is large enough, i.e., if a  sparse matrix $S^{\text{opt}}$ is sought.
{\begin{theorem}\label{th2}
		Assume that the following conditions are satisfied:
		\begin{itemize}
			\vspace{-2mm}
			\item {\bf Condition 2-i.} The graph $\text{supp}(\Sigma^{\mathrm{res}})$ is acyclic.
			\vspace{-2mm}
			\item {\bf Condition 2-ii.} $I_d+\tilde\Sigma^{\mathrm{res}}$ is positive-definite.\label{C2}
			\vspace{-2mm}
			\item {\bf Condition 2-iii.} $\|\tilde{\Sigma}^{\mathrm{res}}\|_{\max}^2\leq \underset{
				\begin{subarray}{c}
				i\not=j\\
				|\Sigma_{ij}|\leq\lambda
				\end{subarray}
			} {\min}\frac{\lambda-|\Sigma_{ij}|}{\sqrt{\Sigma_{ii}\Sigma_{jj}}}$\nonumber.\label{C3}
		\end{itemize}
		Then, the sparsity pattern of the optimal solution $S^{\text{opt}}$ corresponds to the sparsity pattern of $\Sigma^{\mathrm{res}}$ and, in addition,  $S^{\text{opt}}$ can be obtained via the explicit formula  \eqref{S_opt2}.
\end{theorem}}
The above theorem states that if a sparse graph is sought, then as long as some easy-to-verify conditions are met, there is an  explicit formula for the optimal solution.  
It will later be shown  that Condition (2-i) is exactly or approximately satisfied if the regularization coefficient is sufficiently large. Condition (2-ii) implies that the eigenvalues of the normalized residue of $\Sigma$ with respect to $\lambda$ should be greater than  -1. This condition is expected to be automatically satisfied since most  of the elements of $\tilde\Sigma^{\text{res}}$ are equal to zero {and the nonzero elements have small magnitude}. In particular, this condition is satisfied if $I_d+\tilde\Sigma^{\text{res}}$ is diagonally dominant. 
{Finally, using~\eqref{eq18}, it can be verified that  Condition (2-iii) is satisfied if
	\begin{equation}\label{eq27}
	\frac{\left(\frac{2\sigma_1-\sigma_k-\sigma_{k+1}}{\Sigma_{\max}}\right)^2}{\frac{\sigma_k-\sigma_{k+1}}{\Sigma_{\max}}}\leq\frac{2}{r^2}.
	\end{equation}
	Similar to the arguments made in the previous subsection,~\eqref{eq27} shows that  Condition (2-iii) is satisfied if $\frac{2\sigma_1-\sigma_k-\sigma_{k+1}}{\Sigma_{\max}}$ is small. This is expected to hold  in practice since the choice of $\lambda$ entails that $2\sigma_1-\sigma_k-\sigma_{k+1}$ is much smaller than $\Sigma_{\max}$. Under such circumstances, one can use Theorem \ref{th2} to obtain the solution of the GL without having to solve \eqref{eq1} numerically.}

Having computed the  sample covariance matrix, we will next show that checking the conditions in Theorem \ref{th2} and finding $S^{\text{opt}}$ using \eqref{S_opt2} can all be carried out efficiently.
\begin{corollary}\label{cor2}
	Given $\Sigma$ and $\lambda$, the total time complexity of checking the conditions in Theorem \ref{th2} and finding $S^{\text{opt}}$ using \eqref{S_opt2} is $\mathcal{O}(d^2)$.
\end{corollary}

Another line of work has been devoted to studying the connectivity structure of the optimal solution of the GL. In particular, \cite{Mazumdar12} and \cite{Witten11} have shown that the connected components induced by thresholding the covariance matrix and those in the support graph of the optimal solution of the  GL lead to the same vertex partitioning. Although this result does not require any particular condition, it cannot provide any information about the edge structure of the support graph and one needs to  solve \eqref{eq1} for each connected component using an iterative algorithm, which may take up to $\mathcal{O}(d^3)$ per iteration~\cite{friedman2008sparse, banerjee2008model, Mazumdar12}. Corollary~\ref{cor2} states that this complexity could be reduced significantly for sparse graphs.

{\begin{remark}
		The results introduced in Theorem~\ref{thm:tt1} can indeed be categorized as a set of ``safe rules'' that correctly determine sparsity pattern of the optimal solution of the GL. These rules are subsequently reduced to a set of easily verifiable conditions in Theorem~\ref{th2} to safely obtain the correct sparsity pattern of the acyclic components in the optimal solution. On the other hand, there is a large body of literature on simple and cheap safe rules to pre-screen and simplify the sparse learning and estimation problems, including Lasso, logistic regression, support vector machine, group Lasso, etc~\cite{ghaoui2010safe, tibshirani2012strong, fercoq2015mind, ndiaye2015gap}. Roughly speaking, these methods are based on constructing a sequence of \textit{safe regions} that encompass the optimal solution for the dual of the problem at hand. These safe regions, together with the Karush–-Kuhn–-Tucker (KKT) conditions, give rise to a set of rules that facilitate inferring the sparsity pattern of the optimal solution. Our results are similar to these methods since we also analyze the special structure of the KKT conditions and resort to the dual of the GL to obtain the correct sparsity structure of the optimal solution. However, according to the seminal work~\cite{ndiaye2015gap}, most of the developed results on safe screening rules rely on strong Lipschitz assumptions on the objective function; an assumption that is violated in the GL. This calls for a new machinery to derive theoretically correct rules for this problem; a goal that is at the core of Theorems~\ref{thm:tt1} and~\ref{th2}. 
\end{remark}}

\subsection{Approximate Closed-form Solution:  Sparse Graphs}

In the preceding subsection, it was shown that, under some mild assumptions, the GL has an explicit closed-form solution if the support graph of the thresholded sample covariance matrix is acyclic. In this part, a similar approach will be taken to find {approximate} solutions of the GL with an arbitrary underlying sparsity graph. {In particular, by closely examining the hard-to-check conditions of Theorem~\ref{thm:tt1}, a set of simple and easy-to-verify surrogates will be introduced which give rise to an approximate closed-form solution for the general sparse GL.} Furthermore, we will derive a strong upper bound on the approximation error and show that it decreases exponentially fast with respect to the length of the minimum-length cycle in the support graph of the thresholded sample covariance matrix. Indeed, the formula obtained earlier  for acyclic graphs could be regarded as a by-product of this generalization since the length of the minimum-length cycle can be considered as infinity for such graphs. The significance of this result is twofold:

\begin{itemize}
	\item Recall that the support graph corresponding to the optimal solution of the GL is sparse (but not necessarily acyclic) for a large regularization coefficient. In this case, the approximate error is provably small and the derived closed-form solution can serve as a good approximation for the exact solution of the GL. This will later be demonstrated  in different simulations. 
	
	\item The performance and  runtime of numerical (iterative)  algorithms for solving the GL heavily depend on their initializations.  It is known that if the initial point is chosen close enough to the optimal solution, these algorithms converge to the optimal solution in just a few iterations~\cite{friedman2008sparse, Hsieh14, Richard17}. {The approximate closed-form solution designed in this paper can be used as an initial point for the existing numerical algorithms to significantly improve their runtime.} 
\end{itemize}

The proposed approximate solution for the GL with an arbitrary support graph has the following form:
{\begin{equation}\label{S_opt3}
	A_{ij} = \left\{
	\begin{array}{ll}
	\frac{1}{\Sigma_{ii}}\left(1+\underset{
		(i,m)\in\mathcal{E}^{\text{opt}}
	}{\sum}\frac{({\Sigma^{\mathrm{res}}_{im}})^2}{{\Sigma_{ii}\Sigma_{mm}}-({\Sigma^{\mathrm{res}}_{im}})^2}\right) & \text{if}\quad i=j,\\
	\frac{-\Sigma^{\mathrm{res}}_{ij}}{{\Sigma_{ii}\Sigma_{jj}}-({\Sigma^{\mathrm{res}}_{ij}})^2} & \text{if}\quad(i,j)\in \mathcal{E}^{\text{res}},\vspace{1mm}\\
	
	0 & \text{otherwise}.
	\end{array} 
	\right.
	\end{equation}}
The definition of this matrix does not make any assumption on the structure of the graph $\mathcal{E}^{\text{res}}$. {Recall that $\Sigma^{\mathrm{res}}$ in the above formula  is the shorthand notation for $\Sigma^{\mathrm{res}}(\lambda)$. As a result, the matrix $A$ is a function of $\lambda$. To prove that the above matrix is an approximate solution of the GL, a few steps need to be taken. First, recall that---according to the proof of Lemma~\ref{l_beta}---it is possible  to explicitly build the inverse-consistent complement of the thresholded sample covariance matrix if its sparsity graph is acyclic.} This matrix serves as a \textit{certificate} to confirm that the explicit solution \eqref{S_opt3} indeed satisfies the KKT conditions for the GL. By adopting a similar approach, it will then be proved that if the support graph of the thresholded sample covariance matrix is sparse, but not necessarily acyclic, one can find an approximate inverse-consistent complement of the proposed closed-form solution to approximately satisfy the KKT conditions.

\begin{definition}
	Given a number $\epsilon\geq 0$, a $d\times d$ matrix $B$ is called an \textbf{$\epsilon$-relaxed inverse} of matrix $A$ if $A\times B = I_d + E$ such that $|E_{ij}|\leq \epsilon$ for every $(i,j)\in \{1,2,...,d\}^2$.
\end{definition}

The next lemma offers optimality (KKT) conditions for the unique solution of the GL.
\begin{lemma}[\cite{sojoudi2014equivalence111}]\label{expsol} 
	A matrix $S^{\text{opt}}$ is the optimal solution of the GL if and only if it satisfies the following conditions for every $i,j\in\{1,2,...,d\}$
	\begin{subequations}\label{optcon}
		\begin{align}
		& (S^{\text{opt}})^{-1}_{ij} = \Sigma_{ij}&& \text{if}\quad i=j,\\
		& (S^{\text{opt}})^{-1}_{ij} = \Sigma_{ij}+\lambda\times \text{\rm sign}(S^{\text{opt}}_{ij})&& \text{if}\quad S^{\text{opt}}_{ij}\not=0,\\
		& \Sigma_{ij}-\lambda\leq (S^{\text{opt}})^{-1}_{ij} \leq \Sigma_{ij}+\lambda&& \text{if}\quad S^{\text{opt}}_{ij}=0,
		\end{align}
	\end{subequations}
	where $(S^{\text{opt}})^{-1}_{ij}$ denotes the $(i,j)^{\text{th}}$ entry of $(S^{\text{opt}})^{-1}$.
\end{lemma}

The following definition introduces a relaxed version of the first-order optimality conditions given in~\eqref{optcon}.

\begin{definition}
	Given a number $\epsilon\geq 0$, it is said that the $d\times d$ matrix $A$ satisfies the \textbf{$\epsilon$-relaxed KKT conditions} for the GL problem if there exists a $d\times d$ matrix $B$ such that
	\begin{itemize}
		\item $B$ is an {$\epsilon$-relaxed inverse} of the matrix $A$.
		\item The pair $(A, B)$ satisfies the conditions
		\begin{subequations}\label{optcon2}
			\begin{align}
			& B_{ij} = \Sigma_{ij}\hspace{3cm} &&\text{if}\quad i=j,\label{eps1}\\
			& |B_{ij} - \left(\Sigma_{ij}+\lambda\times \mathrm{sign}(A_{ij})\right)|\leq \epsilon&& \text{if}\quad A_{ij}\not=0,\label{eps2}\\
			& |B_{ij}-\Sigma_{ij}| \leq \lambda+\epsilon&& \text{if}\quad A_{ij}=0.\label{eps3}
			\end{align}
		\end{subequations}
	\end{itemize}
\end{definition}
By leveraging the above definition, the objective is to prove that the explicit solution introduced in \eqref{S_opt3} satisfies the $\epsilon$-relaxed KKT conditions for some number $\epsilon$ to be defined later. 

\begin{definition}
	{Given a graph $\mathcal{G}$}, define the function $c(\mathcal{G})$ as the length of the minimum-length cycle of $\mathcal{G}$ (the number $c(\mathcal{G})$ is set to $+\infty$ if $\mathcal{G}$ is acyclic). Let $\text{\rm deg}(\mathcal{G})$ refer to the maximum degree of $\mathcal{G}$.
	Furthermore, define $\mathcal{P}_{ij}(\mathcal{G})$ as the set of all simple paths between nodes $i$ and $j$ in $\mathcal{G}$, and denote the maximum of $|\mathcal{P}_{ij}(\mathcal{G})|_0$ over all pairs $(i,j)$  as $P_{\max}(\mathcal{G})$.
\end{definition}

{Define $\Sigma_{\max}$ and $\Sigma_{\min}$ as the maximum and minimum diagonal elements of $\Sigma$, respectively.
	
	\begin{theorem}\label{thm:approx}
		Under the assumption $\lambda<\sigma_1$, the explicit solution \eqref{S_opt3} satisfies the $\epsilon$-relaxed KKT conditions for the GL with $\epsilon$ chosen as
		\begin{align}\label{epsilon}
		\epsilon =   \max\left\{\Sigma_{\max},\sqrt{\frac{\Sigma_{\max}}{\Sigma_{\min}}}\right\}\cdot\delta\cdot (P_{\max}(\text{\rm supp}(\Sigma^{\mathrm{res}}))-1)\cdot \left(\|\tilde{\Sigma}^{\mathrm{res}}\|_{\max}\right)^{\left\lceil\frac{c(\text{\rm supp}(\Sigma^{\mathrm{res}}))}{2}\right\rceil},
		\end{align}
		where 
		\begin{equation}\label{delta}
		\delta = 1+\frac{\mathrm{deg}(\text{\rm supp}(\Sigma^{\mathrm{res}}))\cdot\|\tilde{\Sigma}^{\mathrm{res}}\|_{\max}^2}{1-\|\tilde{\Sigma}^{\mathrm{res}}\|_{\max}^2}+\frac{\left(\mathrm{deg}(\text{\rm supp}(\Sigma^{\text{res}}))-1\right)}{1-\|\tilde{\Sigma}^{\mathrm{res}}\|_{\max}^2},
		\end{equation}
		if  the following conditions are satisfied:
		\begin{itemize}
			\vspace{-2mm}
			\item {\bf Condition 3-i.} $I_d+\tilde\Sigma^{\mathrm{res}}$ is positive-definite.\label{C22}
			\vspace{-2mm}
			\item {\bf Condition 3-ii.} $\|\tilde{\Sigma}^{\mathrm{res}}\|_{\max}^2\leq \underset{
				\begin{subarray}{c}
				i\not=j\\
				(i,j)\not\in\mathrm{supp}(\Sigma^{\mathrm{res}})
				\end{subarray}
			} {\min}\frac{\lambda-|\Sigma_{ij}|}{\sqrt{\Sigma_{ii}\Sigma_{jj}}}$.\label{C32}
		\end{itemize}
\end{theorem}}


The number $\epsilon$ given in Theorem~\ref{thm:approx} is comprised of different parts:

{\begin{itemize}
		\item $\|\tilde{\Sigma}^{\mathrm{res}}\|_{\max}$: Notice that $\|\tilde{\Sigma}^{\mathrm{res}}\|_{\max}$ is strictly less than 1 and $\lambda$ is large when a sparse graph is sought. Therefore, $\|\tilde{\Sigma}^{\mathrm{res}}\|_{\max}$ is expected to be small for sparse graphs. Under this assumption, we have $0\leq \|\tilde{\Sigma}^{\mathrm{res}}\|_{\max}\ll 1$.
		\item $c(\text{\rm supp}(\Sigma^{\text{res}}))$: It is straightforward to verify that $c(\text{\rm supp}(\Sigma^{\text{res}}))$ is a non-decreasing function of $\lambda$. This is due to the fact that as $\lambda$ increases, $\Sigma^{\text{res}}(\lambda)$ becomes sparser and this results in a support graph with fewer edges. In particular,  if $d\geq 3$, then $c(\text{\rm supp}(\Sigma^{\text{res}}))=3$ for $\lambda = 0$ and $c(\text{\rm supp}(\Sigma^{\text{res}}))= +\infty$ for $\lambda = \sigma_1$ almost surely.
		\item $P_{\max}(\text{\rm supp}(\Sigma^{\text{res}}))$ and $\text{\rm deg}(\text{\rm supp}(\Sigma^{\text{res}}))$: These two parameters are also non-decreasing functions of $\lambda$ and  likely to be small for large $\lambda$. For a small $\lambda$, the numbers  $P_{\max}(\text{\rm supp}(\Sigma^{\text{res}}))$ and $\text{\rm deg}(\text{\rm supp}(\Sigma^{\text{res}}))$ could be on the order of $\mathcal{O}(d!)$ and $\mathcal{{O}}(d)$, respectively. However, these values are expected to be small for sparse graphs. In particular, it is easy to verify that for nonempty and acyclic graphs, $P_{\max}(\text{\rm supp}(\Sigma^{\mathrm{res}}))=1$.
\end{itemize}}
The above observations imply that if $\lambda$ is large enough and the support graph of $\Sigma^{\text{res}}$ is sparse, \eqref{S_opt3} serves as a good approximation of the optimal solution of the GL. In other words, it results from \eqref{epsilon}  that if $\text{supp}(\Sigma^{\text{res}})$ has a structure that is close to an acyclic graph, i.e., it has only a few cycles with moderate lengths, we have $\epsilon\approx 0$. In Section~\ref{simulations}, we will present  illustrative examples to show the accuracy of the closed-form approximate solution with respect to the size of the cycles in the sparsity graph. 

Consider the matrix  $A$ given in \eqref{S_opt3}, {and let  $\mu_{\min}(A)$ and $\mu_{\max}(A)$ denote its minimum and maximum eigenvalues, respectively. If $\lambda = \sigma_1$, then $A = D^{-1}$ (recall that $D$ collects the diagonal entries of $\Sigma$) and subsequently $\mu_{\min}(A) >0$.} Since $\mu_{\min}(\cdot)$ is a continuous function of $\lambda$, there exists a number $\lambda_0$   in the interval $(0,1)$ such that the matrix $A$ (implicitly defined based on $\lambda$) is positive-definite for every $\lambda\geq \lambda_0$.
The following theorem further elaborates on the connection between the closed-form formula and the optimal solution of the GL.

{\begin{theorem}\label{cor:perturb}
		There exists an strictly positive number $\lambda_0$ such that, for every $\lambda\geq \lambda_0$, the matrix $A$ given in \eqref{S_opt3} is the optimal solution of the GL problem after replacing $\Sigma$ with some perturbed matrix $\hat{\Sigma}$ that satisfies the inequality
		\begin{equation}\label{eq:perturb}
		\|\Sigma-\hat{\Sigma}\|_{2}\leq d_{\max}(A)\left(\frac{1}{\mu_{\min}(A)}+1\right) \epsilon,
		\end{equation}
		where $d_{\max}(A)$ is the maximum vertex cardinality of the connected components in the graph $\mathrm{supp}(A)$ and  $\epsilon$ is given in \eqref{epsilon}. Furthermore,~\ref{eq:perturb} implies that
		\begin{equation}\label{optgap}
		f(A)-f^*\leq \left(\mu_{\max}(A)+\mu_{\max}(S^{\mathrm{opt}})\right)d_{\max}(A)\left(\frac{1}{\mu_{\min}(A)}+1\right)\epsilon,
		\end{equation}
		where $f(A)$ and $f^*$ are the objective functions of the GL evaluated at $A$ and the optimal solution, respectively.  
\end{theorem}}

{As mentioned before, if a sparse solution is sought for the GL, the regularization coefficient would be large and this helps with the satisfaction of the inequality $\lambda\geq \lambda_0$. In fact, it will be shown through different simulations that $\lambda_0$ is small in practice and hence, this condition is not restrictive. Under this circumstance, Theorem \ref{cor:perturb} states that the easy-to-construct matrix $A$ is 1) the exact optimal solution of the GL  problem with a perturbed sample covariance matrix, and 2) it is the approximate solution of the GL with the original sample covariance matrix. The magnitudes of this perturbation and approximation error are a function of $d_{\max}(A)$, $\mu_{\min}(A)$, $\mu_{\max}(A)$, $\mu_{\max}(S^{\mathrm{opt}})$, and $\epsilon$. Furthermore, it should be clear that $A$ and $\epsilon$ are functions of $\lambda$ and $\Sigma$ (we dropped this dependency for simplicity of notation). Recall that the disjoint components (or the vertex partitions) of $\text{supp}(A)$ satisfy a nested property: given $1\geq\lambda_1>\lambda_2\geq0$, the components of $\text{supp}(A)$ for $\lambda=\lambda_1$ are nested within the components of $\text{supp}(A)$ for $\lambda=\lambda_2$ (see {\cite{Mazumdar12}} for a simple proof of this statement). This implies that $d_{\max}(A)$ is a decreasing function of $\lambda$. In particular,  it can be observed that $d_{\max}(A) = d$ if $\lambda = 0$ and $d_{\max}(A) = 1$ if $\lambda = \sigma_1$. Now, consider $\mu_{\min}(A)$, $\mu_{\max}(A)$, and $\mu_{\max}(S^{\mathrm{opt}})$. First, note that if $\lambda = \sigma_1$,  then $A = S^{\mathrm{opt}} = D^{-1}$. Furthermore, it is easy to verify that both $A$ and $S^{\mathrm{opt}}$ are continuous functions of $\lambda$. Therefore, for large values of $\lambda$, $\mu_{\min}(A)$, $\mu_{\max}(A)$, and $\mu_{\max}(S^{\mathrm{opt}})$ are expected to be close to $1/\Sigma_{\max}$, $1/\Sigma_{\min}$, and $1/\Sigma_{\min}$, respectively. In addition, as discussed earlier, $\epsilon$ is a decreasing function of $\lambda$ and vanishes when $\lambda$ is large enough. Based on these observations, it can be concluded that the upper bound presented in \eqref{eq:perturb} is small if $\lambda$ is chosen to be large. }

Notice that although the aforementioned value of $\epsilon$ in \eqref{epsilon} and the upper bound in \eqref{eq:perturb} were essential in the study of the effect of the sparsity of the support graph on the accuracy of the presented closed-form solution, they are conservative in practice. These numbers may be tightened significantly for  specific sample covariance matrices. We will further discuss the approximation error of the closed-form solution in Section~\ref{simulations}.

\paragraph{Warm-start algorithm}
{As delineated before, one of the main strengths of the proposed closed-form solution is that it can be used as an initial point (warm-start) for the numerical algorithms specialized for solving the GL. To this goal, the following warm-start procedure is proposed.}

\vspace{2mm}
{\begin{algorithm}[H]
		\SetAlgoLined
		\KwData{data samples ($\mathbf{x}$), and regularization coefficient ($\lambda$)}
		\KwResult{Solution of the GL ($S^{\mathrm{opt}}$)}
		Obtain the residue matrix $\Sigma^{\mathrm{res}}$ based on Definition~\ref{residue} and the closed-form solution $A$ from~\eqref{S_opt3}\;
		\For{each component $i$ in $\mathrm{supp}(\Sigma^{\mathrm{res}})$}{
			\eIf{Conditions 2-i, 2-ii, 2-iii are satisfied}
			{$S^{\mathrm{opt}}[i]\gets A[i]$\;}{
				Find $S^{\mathrm{opt}}[i]$ by numerically solving the GL for component $i$ with initial point $A[i]$\;
			}
		}
		\caption{Warm-start algorithm}\label{algo1}
	\end{algorithm}
	\noindent In the above algorithm, $S^{\mathrm{opt}}[i]$ and $A[i]$ are the submatrices of $S^{\mathrm{opt}}$ and $A$ corresponding to the $i^{\text{th}}$ component of $\mathrm{supp}(\Sigma_{\mathrm{res}})$. The warm-start algorithm is based on the key fact that the GL decomposes over the disjoint components of $\mathrm{supp}(\Sigma_{\mathrm{res}})$~\cite{Mazumdar12, Witten11}. In particular, in the first step, the warm-start algorithm obtains the residue matrix according to Definition~\ref{residue}. Next, for every disjoint component of the residue matrix, if its support graph is acyclic and the conditions of Theorem~\ref{th2} are satisfied, then the corresponding component in $S^{\mathrm{opt}}$ is found using the closed-form solution~\eqref{S_opt2}. Otherwise, this closed-form solution is provided as an initial point to a numerical algorithm, such as GLASSO and QUIC~\cite{friedman2008sparse, Hsieh14}, in order to boost the runtime of solving the GL for the considered component. The results of the warm-start algorithm will be evaluated in the next section.}

{\begin{remark}
		The statistical analysis of the GL entails that $\lambda$ should converge to zero as the number of samples grows to infinity. It is worthwhile to mention that our results may not be applicable in the high sampling regime, where $\lambda$ is close to zero and consequently the thresholded sample covariance matrix is dense. However, notice that the main strength of the GL lies in the high dimensional-low sampling regime where $n$ is much smaller than $d$ and is in the order of $\log d$. Under such circumstances, the proposed explicit formula results in highly accurate solutions for the GL. In fact, it will be shown through massive-scale simulations that in practice, the required conditions on $\lambda$---such as the positive-definiteness of $I_d+\tilde\Sigma^{\mathrm{res}}$---for the validity of the presented results are much more relaxed than the known conditions on $\lambda$ to guarantee the statistical consistency of the GL.
\end{remark}}

\section{Numerical Results}\label{simulations}

{In this section, we will demonstrate the effectiveness of the proposed methods on synthetically generated data, as well as on real data collected from the brain networks and transportation systems.}

{\subsection{Case Study on Synthetic Data}}

%

Given a nonnegative number $\omega$, consider an arbitrary  sample covariance matrix $\Sigma$ with the following properties:
\begin{itemize}
	\item Its diagonal elements are normalized to 1.
	\vspace{-2mm}
	\item The entries corresponding to an arbitrary spanning tree of $\text{supp}(\Sigma)$ belong to the union of the intervals $[-0.85,-0.95]$ and $[0.85,0.95]$.
	\vspace{-2mm}
	\item The off-diagonal entries that do not belong to the spanning tree are in the interval $[-0.85+\omega, 0.85-\omega]$.
\end{itemize}
The goal is to find conditions on $\lambda$, $\omega$ and the size of the covariance matrix such that Theorem \ref{th2} can be used to obtain a closed-form solution for the GL problem. One can choose the value of $\lambda$ to be greater than $\sigma_{d}$ to ensure that the graph $\text{supp}(\Sigma^{\text{res}})$ is acyclic. In particular, if we pick $\lambda$ in the interval $(\sigma_{d},\sigma_{d-1})$, the graph $\text{supp}(\Sigma^{\text{res}})$ becomes a spanning tree. 

{Select $\lambda$ as $0.85-\epsilon$ for a sufficiently small number $\epsilon$ and consider Condition (2-ii) in Theorem \ref{th2}. One can easily verify that $I_d+\Sigma^{\text{res}}$ is positive-definite if the inequality $
	\frac{1}{\text{deg}(v)}>(\sigma_1-\lambda)^2
	$ holds
	for every node $v$ in $\text{supp}(\Sigma^{\text{res}})$, where $\text{deg}(v)$ is the degree of node $v$. This condition is guaranteed to be satisfied for all possible acyclic graphs if $(\mathrm{deg}(v))(0.95-0.85)^2<1$ or equivalently $\mathrm{deg}(v)\leq 100$ for every node $v$. Regarding Condition (2-iii), it can be observed that the relation $(\sigma_1-\lambda)^2\leq \lambda-\sigma_{k+1}$ holds if $(0.95-0.85)^2<0.85-(0.85-\omega)$. This implies that the inequality $\omega>0.01$ guarantees the satisfaction of Condition (2-iii) for every acyclic graph $\text{supp}(\Sigma^{\text{res}})$. In other words, one can find the optimal solution of the GL problem  using the explicit formula in Theorem \ref{th2} as long as:  1) a spanning tree structure for the optimal solution of the GL problem is sought, 2) the degree of each node in the spanning tree is not greater than 100, and (3) the difference between $\sigma_{d-1}$ and $\sigma_d$ is greater than $0.01$. Note that Condition (2)  is conservative and can be dropped for certain types of graphs (e.g., path graphs). In practice, the positive-definiteness of  $I_d+\Sigma^{\text{res}}$ is not restrictive; we have verified that this matrix is positive-definite for randomly generated instances with the sizes up to $d = 200,000$ even when $\mathrm{deg}(v)> 100$. }

\begin{figure*}
	\centering
	\includegraphics[width=.55\columnwidth]{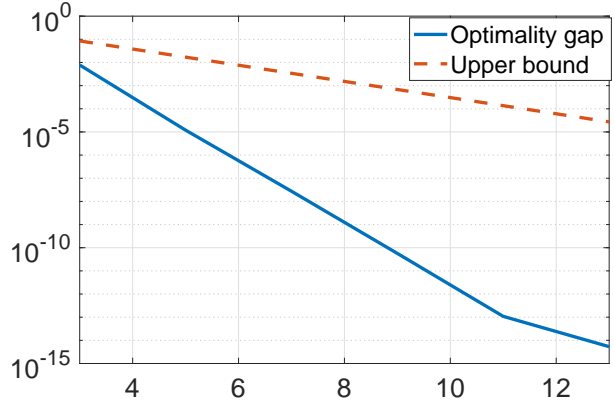}
	\hspace{0.1cm}
	\caption{ \footnotesize {The optimality gap between the closed-form and optimal solutions for the GL}}
	\label{Error_cycle}
\end{figure*}

{Now, consider the following modifications in the experiment:
	\begin{itemize}
		\item The elements of $\Sigma$ corresponding to a cycle of length $d$ are randomly set to $-0.8$ or $0.8$ with equal probability.
		\vspace{-2mm}
		\item The off-diagonal entries that do not correspond to the above cycle are in the interval $[-0.7, 0.7]$.
	\end{itemize}
	If $\lambda$ is chosen as $0.75$, then the graph $\text{supp}(\Sigma^{\text{res}})$ coincides with a cycle of length $d$. Furthermore, $I_d+\Sigma^{\text{res}}$ is diagonally dominant and hence positive-definite for every $d$. Figure~\ref{Error_cycle} shows the optimality gap of the proposed closed-form solution and its derived theoretical upper bound (i.e. the left and right hand sides of ~\eqref{optgap}, respectively) with respect to the length of the cycle $d$ in log-linear scale. (note that  $\text{deg}(\text{\rm supp}(\Sigma^{\text{res}}))$ and $P_{\max}(\text{\rm supp}(\Sigma^{\text{res}}))$ in~\eqref{optgap} are both equal to 2).
	Two important observations can be made based on this figure.
	\begin{itemize}
		\item In practice, the performance of the derived closed-form solution is significantly better than its theoretical upper bounds. In fact, this error is less than $10^{-6}$ when the length of the minimum-length cycle is at least 6. The high accuracy of the closed-form solution will become more evident in the subsequent case studies on large-scale problems.
		\item It can be seen that the logarithm of the optimality gap is approximately a linear function of the cycle length. This matches the behavior of the theoretical bounds introduced in Theorems~\ref{thm:approx} and~\ref{cor:perturb}: the approximation error is exponentially decreasing with respect to the length of the minimum-length cycle.
\end{itemize}}

\subsection{Case Study on Brain Networks}

\begin{figure*}\label{key}
	\centering
	\subfloat[Number of edges]{\label{Exact_edge_fMRI_combine}
		\includegraphics[width=.5\columnwidth]{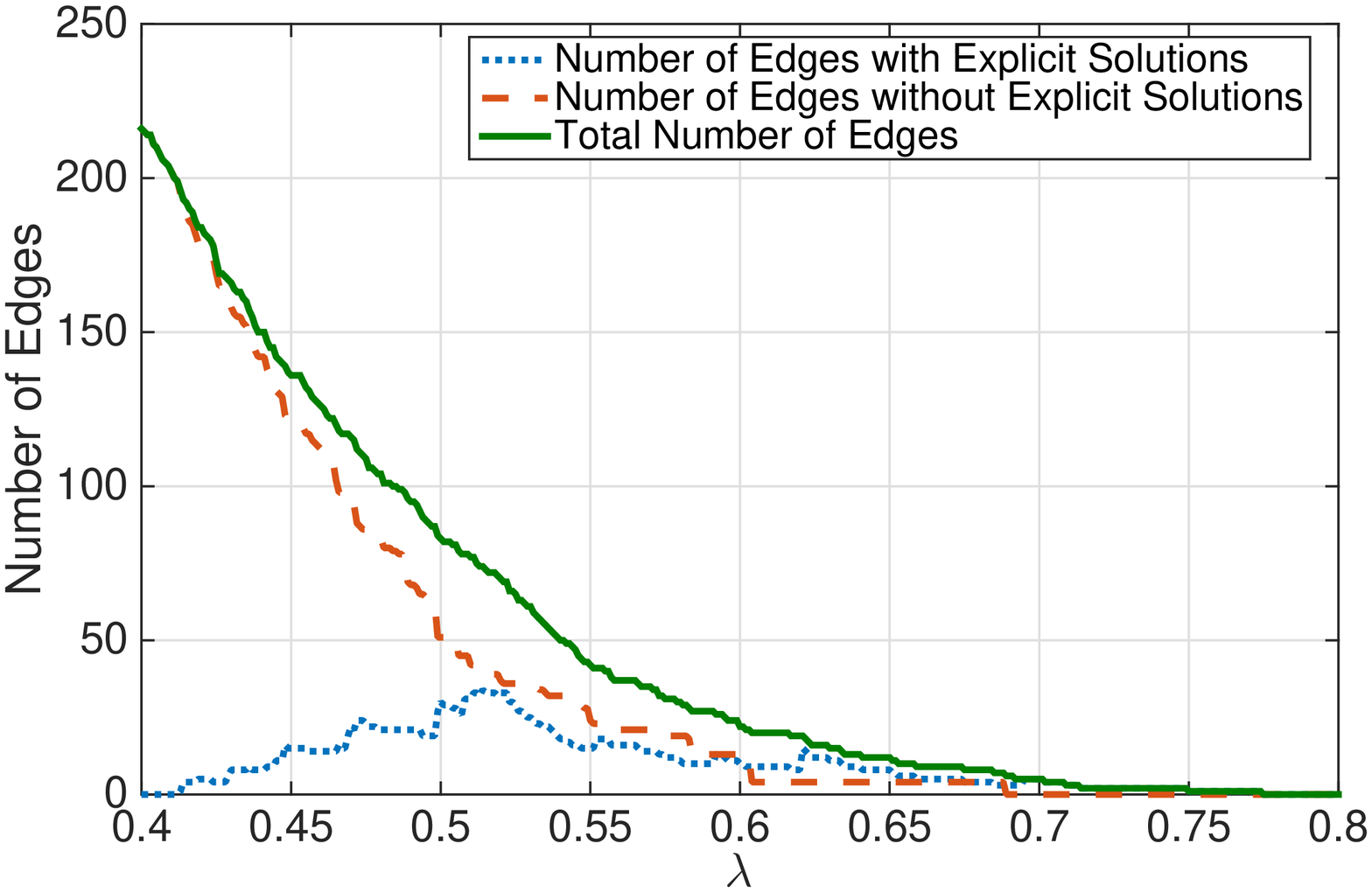}}
	\hspace{-0.6cm}
	\subfloat[Number of nodes]{\label{Exact_node_fMRI_combine}
		\includegraphics[width=.5\columnwidth]{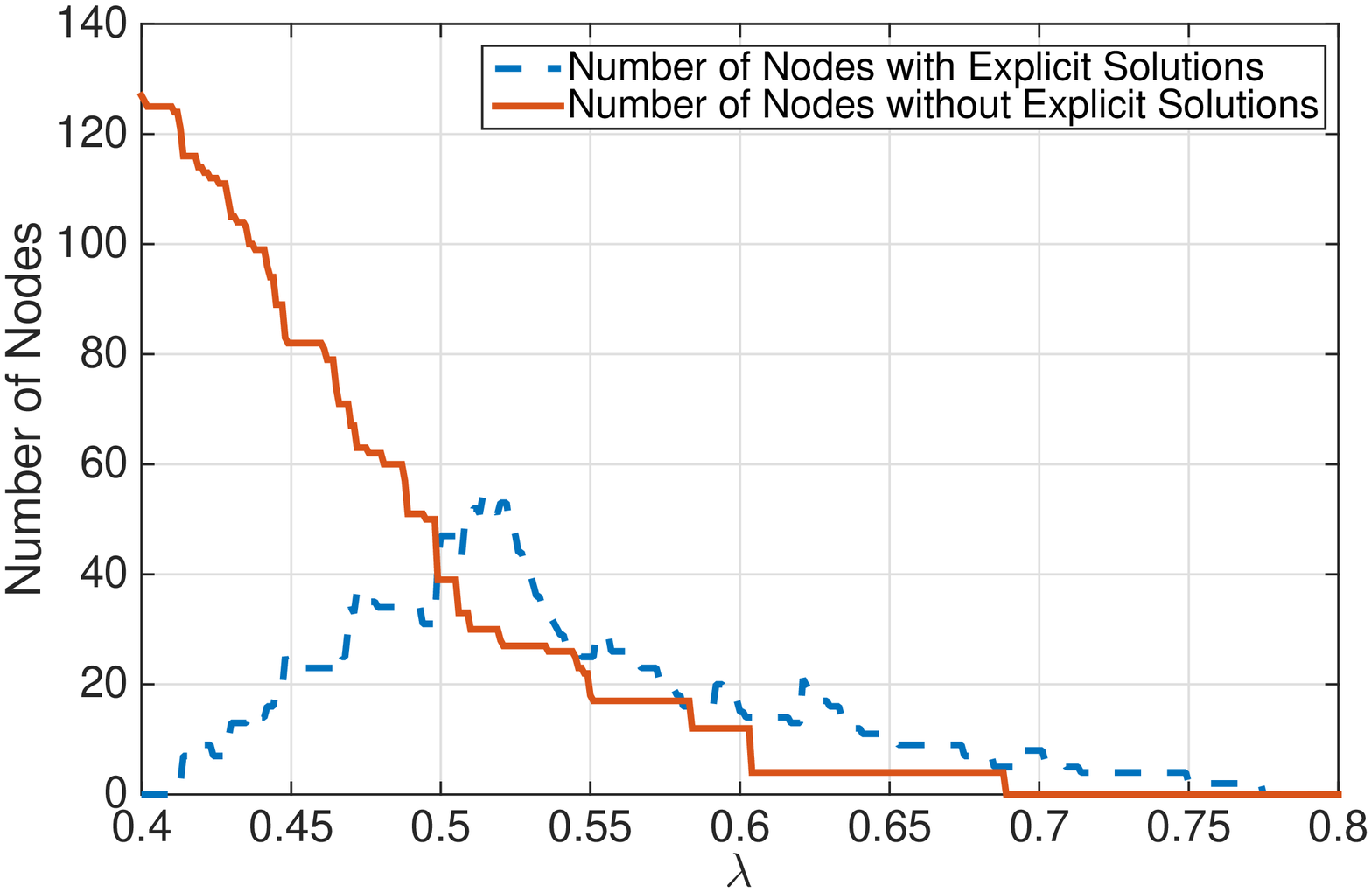}}
	
	\subfloat[Number of nonzeros and mismatches]{\label{Nonzero_fMRI_combine}
		\includegraphics[width=.5\columnwidth]{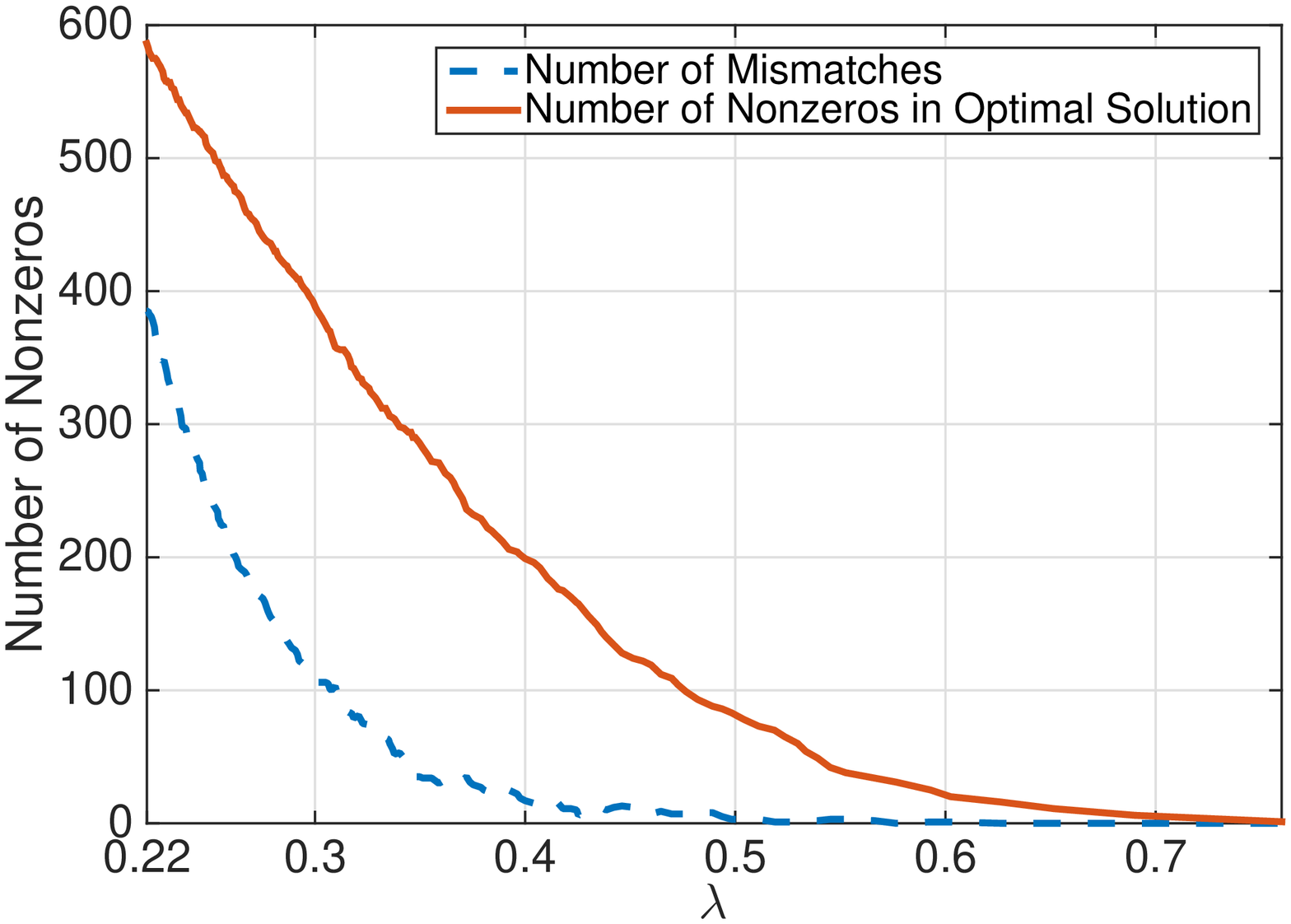}}
	\hspace{-0.6cm}
	\subfloat[Minimum eigenvalue]{\label{MinEig_fMRI_combine}
		\includegraphics[width=.5\columnwidth]{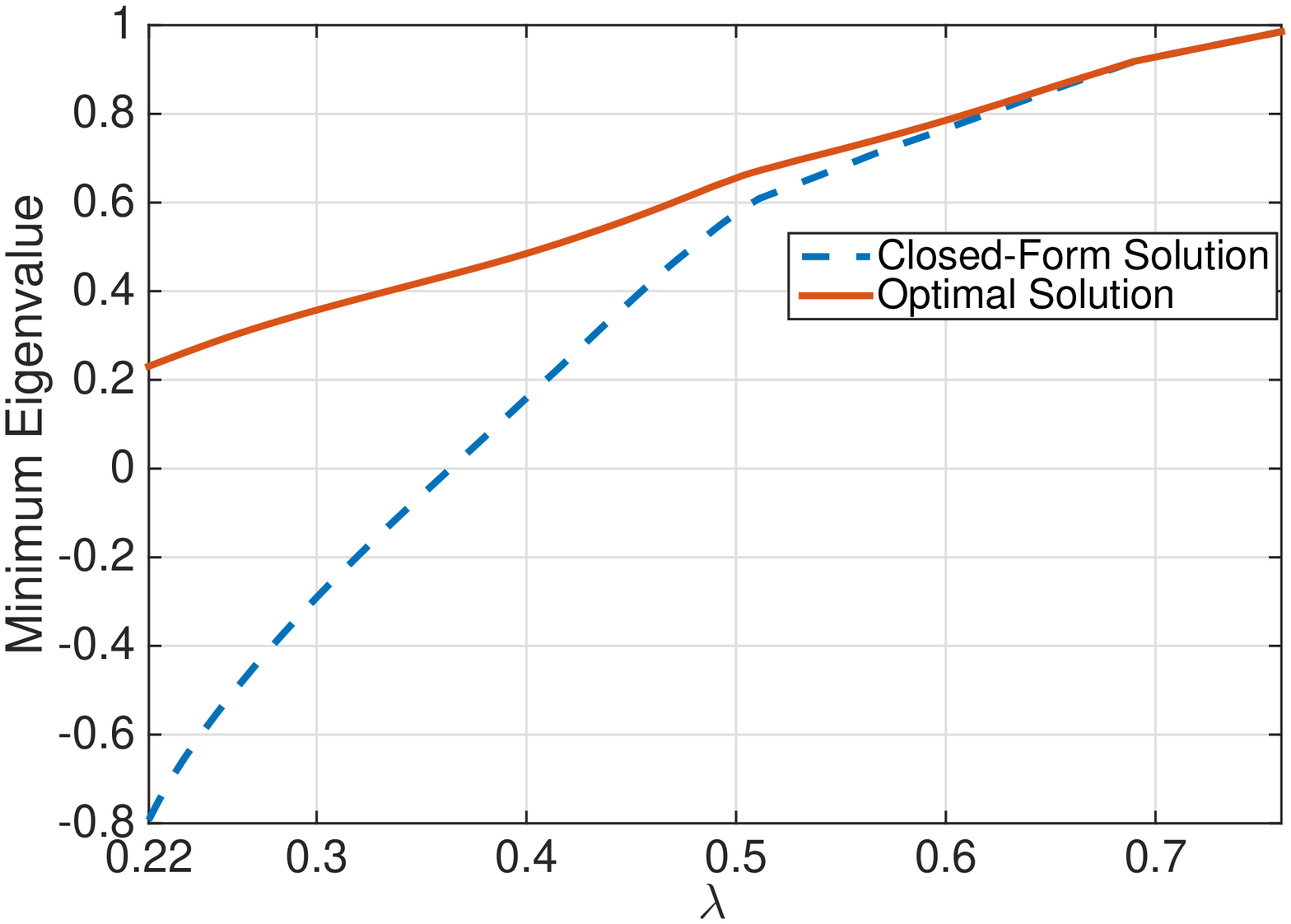}}
	
	\subfloat[2-norm of the approximation error]{\label{Normdiff_fMRI_combine}
		\includegraphics[width=.5\columnwidth]{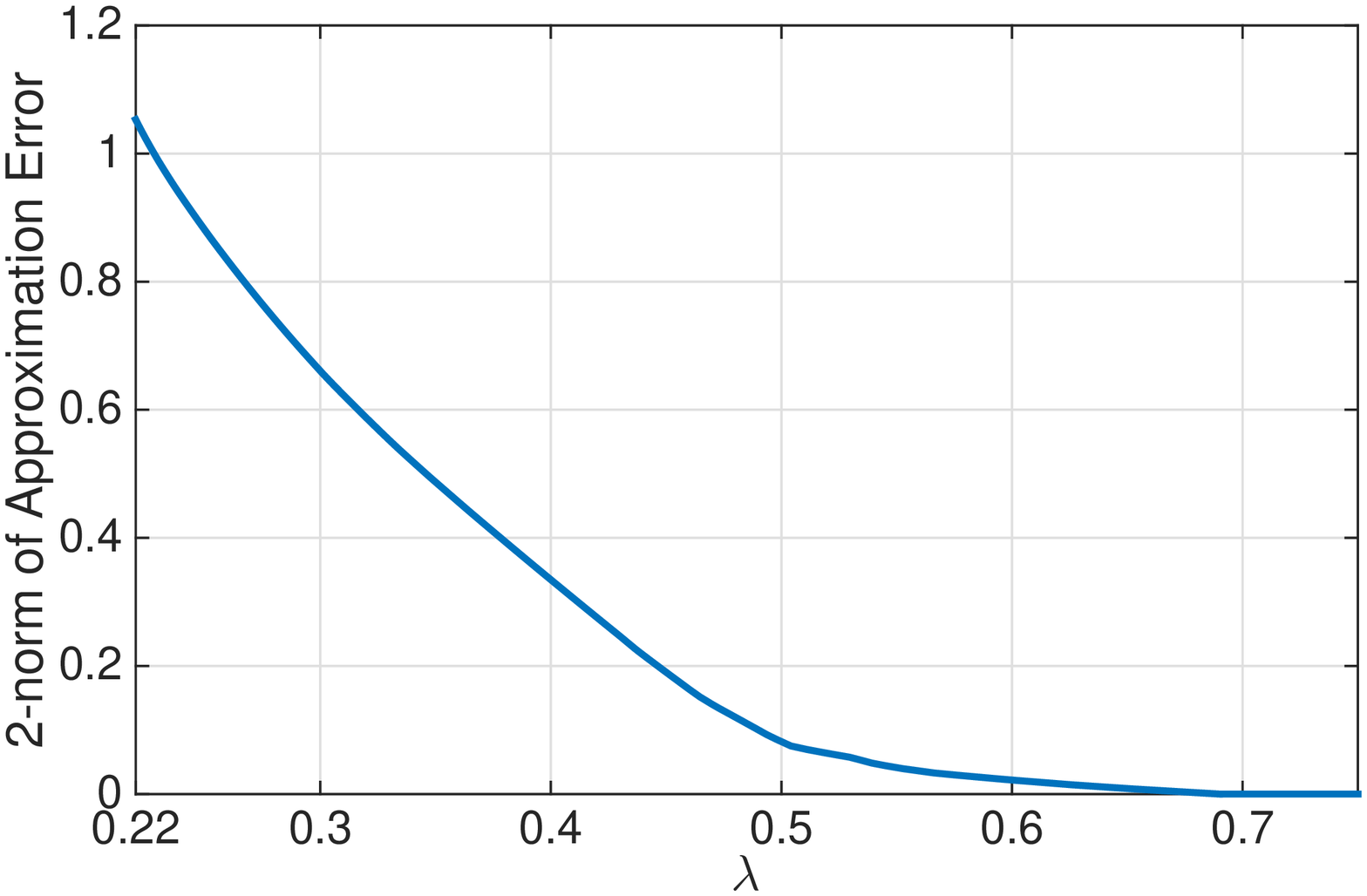}}
	\hspace{-0.6cm}
	\subfloat[Similarity degree]{\label{Similarity_fMRI_combine}
		\includegraphics[width=.5\columnwidth]{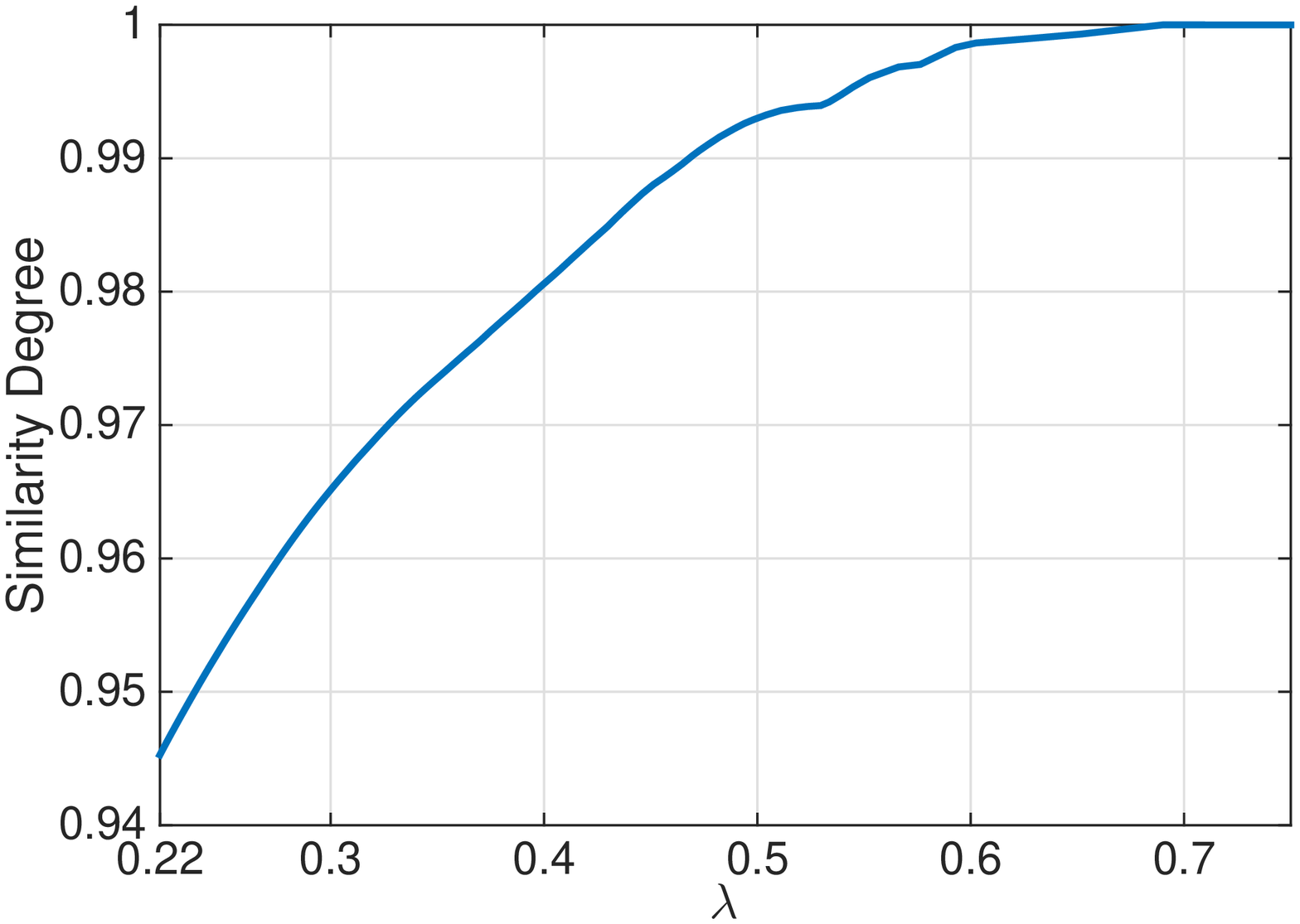}}

	\caption{ \footnotesize a) Number of edges in the sparsity graph of the closed-form approximate solution  whose corresponding entries are guaranteed to be equal to those in the sparsity graph of the optimal solution due to  Theorem~\ref{th2}. b) Number of nodes that belong to the components for which the corresponding submatrices of the optimal solution have explicit formulas. c) Number of edges in the sparsity graph of the optimal solution, compared to the number of mismatches. d) Minimum eigenvalues of the optimal and  closed-form approximate solutions. e) The 2-norm of the difference between the optimal and approximate solutions. f) The similarity degree between the optimal and  approximate solutions.}
\end{figure*}

Consider the problem of estimating the brain functional connectivity network based on a set of resting state functional MRI (fMRI) data  collected from 20 individual subjects~\cite{Vertes12}.  The data  for each subject correspond to disjoint brain activities and are correlated due to the underlying functional connectivity structure of the brain. In order to represent these dependencies, each disjoint region of the brain can be considered as a node and the correlation between two different regions can be resembled by an edge between the nodes. The data set for each subject consists of 134 samples of low frequency oscillations taken from 140 different cortical brain regions. We construct a normalized sample covariance matrix by combining the data sets of all 20  subjects (note that the data for each individual is limited and not informative enough, but the combined data provides rich information about the brain network). The goal is to use the GL to estimate the underlying functional connectivity network of different regions of the brain based on the obtained $140\times 140$ sample covariance matrix. We study the thresholded sample covariance matrix and the derived closed-form solution for different values of the regularization coefficient in order to analyze their accuracy.


Figure~\ref{Exact_edge_fMRI_combine} shows the number of edges in the sparsity graph of the thresholded sample covariance matrix that belong to those  connected components satisfying the conditions in Theorem~\ref{th2}. The formula derived in this paper is able to find the optimal values of the  entries of the solution corresponding to these edges. {It can be observed that if $\lambda$ is greater than 0.51, then almost half of the edges in the sparsity graph of the optimal solution can be found using the proposed explicit formula. This is due to the fact that the corresponding entries in the residue matrix belong to the acyclic components of its sparsity graph and satisfy the conditions of Theorem~\ref{th2}.}
Figure~\ref{Exact_node_fMRI_combine} depicts the number of nodes that belong to the components (with  sizes greater than 1) for which the corresponding submatrices of the solution of the GL have an explicit formula.   Note that those entries in the optimal solution that correspond to isolated nodes are trivially equal to 0. Therefore, in order to better reflect the significance of the derived solution, we have only considered the components with at least two nodes. It can be observed that if $\lambda$ is greater than 0.5, then the number of nodes belonging to the components with explicit formula is greater than the number of those nodes associated with inexact closed-form solutions. Figure~\ref{Nonzero_fMRI_combine} demonstrates the number of edges in the sparsity graph of the optimal solution, together with the number of mismatches in the edge sets of the sparsity graphs of the optimal and thresholded solutions. Notice that the number of  mismatches is less than $10\%$ when $\lambda$ is greater than 0.35 and  is almost 0 when $\lambda$ is greater than 0.5. 

Figure~\ref{MinEig_fMRI_combine} shows the minimum eigenvalues of the optimal and closed-form approximate solutions for different values of $\lambda$. The approximate solution is positive-definite when $\lambda$ is greater than 0.37. This implies that $\lambda_0$ in Corollary~\ref{cor:perturb} is equal to $0.37$. Figures~\ref{Normdiff_fMRI_combine} and~\ref{Similarity_fMRI_combine} depict the 2-norm of the approximation error (the difference between the optimal and closed-form approximate solutions) and the similarity degree between these two solutions, which is defined as
\begin{equation}\notag
\text{similarity degree}= \frac{\text{trace}(\tilde S^{\text{opt}}\times \tilde A)}{\| \tilde S^{\text{opt}}\|_F\times  \|\tilde A\|_F},
\end{equation}
where $\tilde S^{\text{opt}} = S^{\text{opt}}-I_d$ and $\tilde A = A-I_d$. Subtracting the identity matrix from $A$ and $S^{\text{opt}}$ is due to the observation that both matrices have diagonal entries close to 1 when the support graph is sparse. This leads to an artificially inflated similarity degree between $A$ and $S^{\text{opt}}$. Therefore, in order to have a better assessment of the similarity between the closed-form and optimal solutions, we measure the similarity between $A$ and $S^{\text{opt}}$ after softening the effect of their diagonal entries. The similarity degree of 1 means that the optimal and approximate solutions are exactly equal.

It can be observed that the approximation error is small and the similarity degree is high for a wide range of values of $\lambda$. For instance, if $\lambda$ is greater than 0.4, then the 2-norm of the approximation error is less than $0.37$ and the similarity degree is greater than $0.98$. For these values of $\lambda$, the number of edges in the sparsity graph of the optimal solution ranges from 200 to 0. In all of these cases,  the structure and  values of the optimal solution can be estimated efficiently, without solving the optimization problem numerically.

\subsection{Case Study on Transportation Networks}

\begin{figure*}
	\centering
	\subfloat[Number of nonzeros and mismatches]{\label{Nonzero_Region3_1Week_March6to12}
		\includegraphics[width=.45\columnwidth]{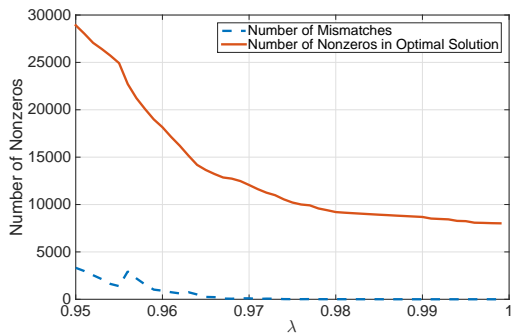}}
	\hspace{-0.6cm}
	\subfloat[Minimum eigenvalue]{\label{MinEig_Region3_1Week_March6to12}
		\includegraphics[width=.45\columnwidth]{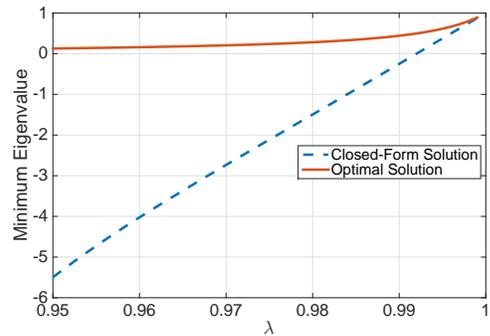}}
	
	\subfloat[2-norm of the approximation error]{\label{Normdiff_Region3_1Week_March6to12}
		\includegraphics[width=.45\columnwidth]{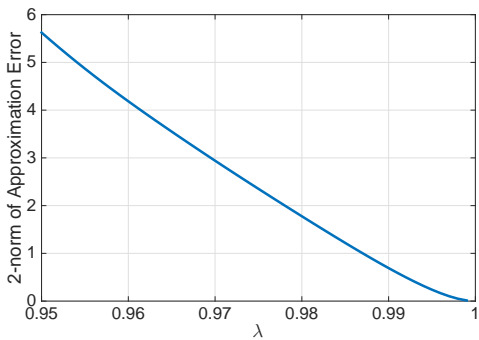}}
	\hspace{-0.6cm}
	\subfloat[Similarity degree]{\label{Similarity_Region3_1Week_March6to12}
		\includegraphics[width=.45\columnwidth]{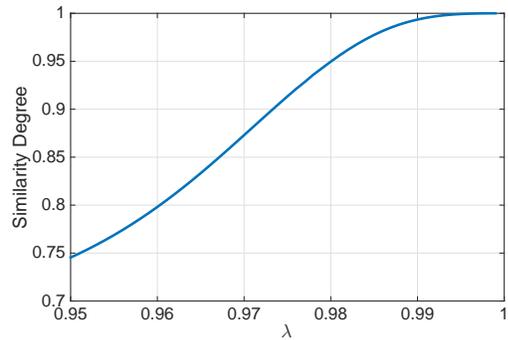}}

	\caption{ \footnotesize a) Number of edges in the sparsity graph of the optimal solution, compared to the number of mismatches. b) Minimum eigenvalues of the optimal and closed-form approximate solutions. c) The 2-norm of the difference between the  optimal and approximate solutions. d) The similarity degree between the optimal and approximate solutions.}
	\label{Trans}
\end{figure*}

In recent years, the problem of short- and long-term traffic flow prediction and control has attracted much attention in Intelligent Transportation Systems (ITSs)~\cite{Lino01}. 
Estimating the correlation between the traffic flows on different links of a transportation network is one of the crucial steps toward the traffic congestion control in the network; it can also serve as an initial block in different traffic forecasting methods. Substantial research has been devoted to extracting these dependencies and performing predictions based on the measured data (see \cite{Yin02, Rafegh17, Sun12} and the references therein). In this case study, the objective is to construct a sparse matrix representing the conditional covariance between the traffic flows of different links in the network. The data  is collected from the Caltrans Performance Measurement System (PeMS) database, which consists of traffic information of freeways on the a statewide scale across California~\cite{PeMS}. We consider the data  measured by the stations deployed in District 3 of California, which is collected and aggregated every 5 minutes from 1277 stations during March 6$^{\text{th}}$ to March 12$^{\text{th}}$ of the year 2017 (one-week interval). Due to the malfunctioning of some of the detectors, a non-negligible portion of the traffic flows was missing from the raw data set. Therefore, the following steps were taken before solving the GL problem in order to obtain a useful representation of the raw data:
\begin{itemize}
	\item Since 228 stations did not have sufficient number of measurements during the one-week period, they were removed from the sampled data.
	\item In a few stations, the detectors did not measure the traffic flow for some periods of time. For these data samples, we used a linear interpolation method  to estimate the missing values.
\end{itemize}
After performing the aforementioned data-cleaning steps, a $1049\times 1049$ normalized sample covariance matrix was constructed from the combined 2016 data samples (288 samples for each day of the week).  In Figure~\ref{Trans}, the accuracy of the thresholding technique and its corresponding closed-form approximate solution is compared to the optimal solution of the GL problem for different values of the regularization coefficient.

Since the number of entries in the upper triangular part of the sample covariance matrix is large (roughly 550,000 entries), we have only considered large values of $\lambda$ in order to obtain a sparse solution for the GL. Figure~\ref{Nonzero_Region3_1Week_March6to12} shows the number of edges in the sparsity graph of the optimal solution, compared to the number of mismatches between the edge sets of the sparsity graphs of the optimal and closed-form solutions. It can be observed that as $\lambda$ increases, the support graph of the optimal solution becomes sparser and the number of mismatches decreases. In particular, the number of mismatches is almost zero if $\lambda$ is chosen to be greater than $0.97$. Figure~\ref{MinEig_Region3_1Week_March6to12} depicts the minimum eigenvalues of the optimal and closed-form approximate solutions of the GL with respect to $\lambda$. The approximate solution becomes positive-definite if $\lambda$ is greater than $0.991$. Furthermore, Figures~\ref{Normdiff_Region3_1Week_March6to12} and~\ref{Similarity_Region3_1Week_March6to12} show that, for those values of $\lambda$ between $0.991$ and $0.999$, the 2-norm of the approximation error is between $0.5$ and $0.01$, and that the similarity degree is greater than $0.99$. For this range of $\lambda$, the number of edges in the sparsity graph of the optimal solution is  $7.82$ to $7.40$ times higher the number of nodes. 

{\subsection{Case Study on Large-Scale Data}
	
	In this case study, we evaluate the performance of the proposed closed-form solution on massive randomly generated data sets. Given $d$ (the dimension of each sample) and similar to~\cite{Hsieh14} and~\cite{yun2011coordinate}, a sparse inverse covariance matrix is generated for each test case according to the following procedure: first, a sparse matrix $U\in\mathbb{R}^{d\times d}$ is generated whose nonzero elements are randomly set to $+1$ or $-1$, with equal probability. Then, the inverse covariance matrix is set to $UU^{\top}+2I$. Depending on the test case, the number of nonzero elements in $U$ is controlled so that the resulted inverse covariance matrix has approximately $5d$ or $10d$ nonzero elements. $n = d/2$ number of i.i.d. samples are drawn from the corresponding multivariate Gaussian distribution in all experiments, except for the largest test case with $d = 80000$. This instance has more than $3.2$ billion variables and only $n = 20000$ samples are collected to solve the GL due to the memory limitations. Furthermore, the regularization coefficient is chosen such that the estimated solution has approximately the same number of nonzero elements as the ground truth.
	
	Table~\ref{Table} reports the runtime of the closed-form solution, compared to two state-of-the-art methods for solving the GL, namely QUIC~\cite{Hsieh14} and GLASSO~\cite{friedman2008sparse} algorithms, as well as elementary estimator~\cite{yang2014elementary}. The GLASSO is the most widely-used algorithm for the GL, while the QUIC algorithm is commonly regarded as the fastest available solver for this problem. The elementary estimator is recently proposed in lieu of the GL to remove its computational burden, while preserving its desired high-dimensional properties. We use the source codes for latest versions of QUIC and GLASSO in our simulations. In particular, we use the QUIC 1.1 (available in \url{http://bigdata.ices.utexas.edu/software/1035/}) which is implemented in $C\!+\!+$ with MATLAB interface. The GLASSO is downloaded from \url{http://statweb.stanford.edu/~tibs/glasso/} and is implemented in FORTRAN with MATLAB interface. We implemented the elementary estimator and the proposed closed-form solution in MATLAB using its sparse package. A time limit of 4 hours is considered in all experiments. Table~\ref{Table} has the following columns:
	
	\begin{itemize}
		\item $d$: The dimension of the samples.
		\item $m$: The number of nonzero elements in the true inverse covariance matrix.
		\item Closed-form: The runtime of the proposed method.
		\item QUIC-C and GLASSO-C: The runtime of the QUIC and GLASSO without initialization.
		\item QUIC-W and GLASSO-W: The runtime of the QUIC and GLASSO using the warm-start Algorithm~\ref{algo1}.
		\item Elem.: The runtime of the elementary estimator.
	\end{itemize}
	
	\begin{table}\small\centering
		\begin{tabular}{ c|c||c||c|c||c|c||c}
			$d$ & $m$ & Closed-Form & QUIC-C & QUIC-W & GLASSO-C & GLASSO-W & Elem. \\  
			\hline
			$2000$ & $9894$ & $0.1$ & $2.0$ & $1.4$ & $42.8$ & $13.5$ & $0.2$\\
			\hline
			$2000$ & $20022$ & $0.1$ & $3.0$ & $2.1$ & $43.8$ & $15.3$ & $0.2$\\
			\hline
			$4000$ & $20094$ & $0.5$ & $13.9$ & $7.5$ & $460.8$ & $135.1$ & $2.1$\\
			\hline
			$4000$ & $40382$ & $0.5$ & $21.5$ & $12.0$ & $467.6$ & $156.2$ & $2.9$\\
			\hline
			$8000$ & $40218$ & $2.5$ & $78.7$ & $49.3$ & $3675.1$ & $1011.2$ & $11.3$\\
			\hline
			$8000$ & $79890$ & $2.5$ & $111.7$ & $88.4$ & $3784.3$ & $1278.8$ & $22.2$\\
			\hline
			$12000$ & $60192$ & $7.8$ & $243.8$ & $153.1$ & $\star$ & $3233.0$ & $31.8$\\
			\hline
			$12000$ & $119676$ & $7.4$ & $333.6$ & $251.0$ & $\star$ & $3437.2$ & $70.2$\\
			\hline
			$16000$ & $80064$ & $17.1$ & $570.0$ & $322.8$ & $\star$ & $6545.0$ & $67.2$\\
			\hline
			$16000$ & $160094$ & $18.5$ & $787.4$ & $616.4$ & $\star$ & $9960.8$ & $174.8$\\
			\hline
			$20000$ & $99954$ & $39.4$ & $1266.5$ & $539.4$ & $\star$ & $\star$ & $107.8$\\
			\hline
			$20000$ & $200018$ & $37.4$ & $1683.8$ & $1392.5$ & $\star$ & $\star$ & $211.5$\\
			\hline
			$40000$ & $200290$ & $495.4$ & $\star$ & $\star$ & $\star$ & $\star$ & $\star$\\ 
			\hline
			$80000$ & $401798$ & $1450.4$ & $\star$ & $\star$ & $\star$ & $\star$ & $\star$\\ 
		\end{tabular}
		\caption{ \footnotesize {The runtime of different methods.}}\label{Table}
	\end{table}
	
	In all of the test cases, the resulted closed-form solution is positive-definite and hence, feasible.
	It can be seen that the proposed method significantly outperforms QUIC, GLASSO and elementary estimator in terms of its runtime. In particular, the presented method is on average $6$, $36$, and $951$ times faster than elementary, QUIC, and GLASSO methods, respectively, provided that they can obtain the solution within the predefined time limit. Furthermore, for the cases where the GL can be solved to optimality using QUIC, the relative optimality gap of the closed-form solution, i.e., $(f(A)-f^*)/f^*$, is $2.1\times 10^{-3}$ on average. For the cases with $d = 40000$ and $d = 80000$, none of these methods converge to a meaningful solution, while the proposed method can obtain an accurate solution in less than 30 minutes. On the other hand, the warm-start Algorithm~\ref{algo1} accompanied by QUIC and GLASSO yields up to $2.35$ and $4.45$ times speedups in their runtime, respectively. Moreover, the warm-start algorithm doubles the size of the instances that are solvable using the GLASSO.
	
	Table~\ref{Table2} compares the accuracy of the estimated inverse covariance matrix using different methods. This table includes the following columns:
	\begin{itemize}
		\item $\ell_F$: The Frobenius norm of the difference between the true and estimated inverse covariance matrices, normalized by the Frobenius norm of the true inverse covariance matrix.
		\item TPR and FPR: The true positive rate (TPR) and false positive rate (FPR) defined as
		\begin{align}
		& \text{TPR} = \frac{\left|(i,j): i\not=j, S_{ij}\not= 0, (\Sigma_*^{-1})_{ij}\not= 0\right|_0}{\left|(i,j): i\not=j, (\Sigma_*^{-1})_{ij}\not= 0\right|_0},\notag\\
		& \text{FPR} = \frac{\left|(i,j): i\not=j, S_{ij}\not= 0, (\Sigma_*^{-1})_{ij}= 0\right|_0}{\left|(i,j): i\not=j, (\Sigma_*^{-1})_{ij}= 0\right|_0},\notag
		\end{align} 
		where $S$ corresponds to the explicit formula, the optimal solution of the GL, or the elementary estimator.
	\end{itemize}
	
	It can be seen that, while the elementary estimator has slightly better estimation error, its TPR is significantly outperformed by the those of the GL and closed-form solutions. Furthermore, it can be seen that the closed-form estimator has almost the same accuracy as the optimal solution of the GL. The superiority of the proposed closed-form solution over the other methods becomes more evident in the larger instances, where it (almost) exactly recovers the true sparsity pattern of the inverse covariance matrix and results in small estimation error, while becoming the only viable method for estimating the inverse covariance matrix.}

{Finally, we show that the requirement $\lambda\geq\lambda_0$ in Theorem~\ref{cor:perturb} does not impose any restriction on the practicality of this theorem under the {finite-sampling} regime. In particular, we show that in practice, the lower bound $\lambda_0$ on $\lambda$ is {significantly smaller} than the theoretical value of $\lambda$ that is derived for the high-dimensional consistency of the GL. 
	To this goal, we compare $\lambda_0$ with the theoretical value of $\lambda$ introduced in the seminal paper~\cite{ravikumar2011high}. In particular,~\cite{ravikumar2011high} shows that the following value for $\lambda$ is sufficient to guarantee consistency
	\begin{equation}\label{lambda_t}
	\lambda = \frac{8}{\alpha}\sqrt{128(1+4\sigma^2)^2\max(\Sigma_*)_{ii}^2}\sqrt{\frac{\log d+\log 4}{n}},
	\end{equation}
	where $\alpha$ is the mutual incoherence parameter, $\sigma$ is the sub-Gaussian parameter of normalized random variables, and $(\Sigma_*)_{ii}$ is the $i^{\text{th}}$ diagonal element of the true covariance matrix. Figure~\ref{lambda} shows the values for $\lambda_0$, theoretical $\lambda$ defined as~\eqref{lambda_t}, and $\lambda$ used in our simulations with respect to the dimension of the problem. On average, $\lambda_0$ is 640 and 6 times smaller than the theoretical and used $\lambda$, respectively. Furthermore, Figure~\ref{density} shows the density (the number of nonzero elements, normalized by the total number of entries) of the thresholded sample covariance matrix when $\lambda$ is set to $\lambda_0$, compared to the density of the true inverse covariance matrix. Note that when $\lambda = \lambda_0$, the density of the thresholded sample covariance matrix is close to $0.3$ on average while the average density of the true inverse covariance matrix is less than $0.0009$. Based on these simulations, one can infer that $\lambda_0$ is an under-estimator for the values of the regularization coefficient that correctly promote sparsity in the estimated solution, and the requirement $\lambda\geq \lambda_0$ is extremely mild for large-scale instances of the GL.}

\begin{figure*}
	\centering
	\subfloat[Different values of $\lambda$]{\label{lambda}
		\includegraphics[width=.45\columnwidth]{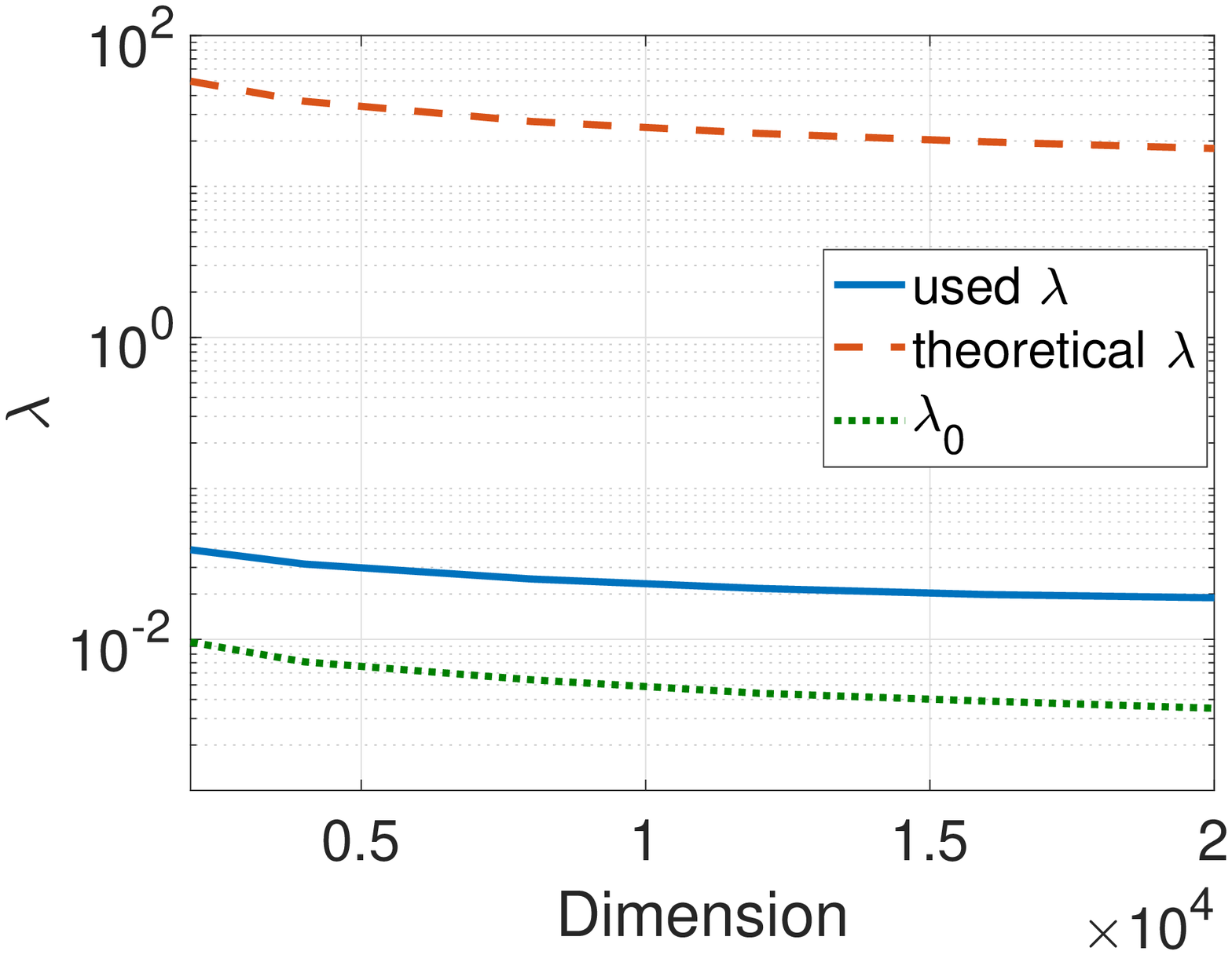}}
	\hspace{0.1cm}
	\subfloat[Density of the thresholded sample covariance matrix]{\label{density}
		\includegraphics[width=.45\columnwidth]{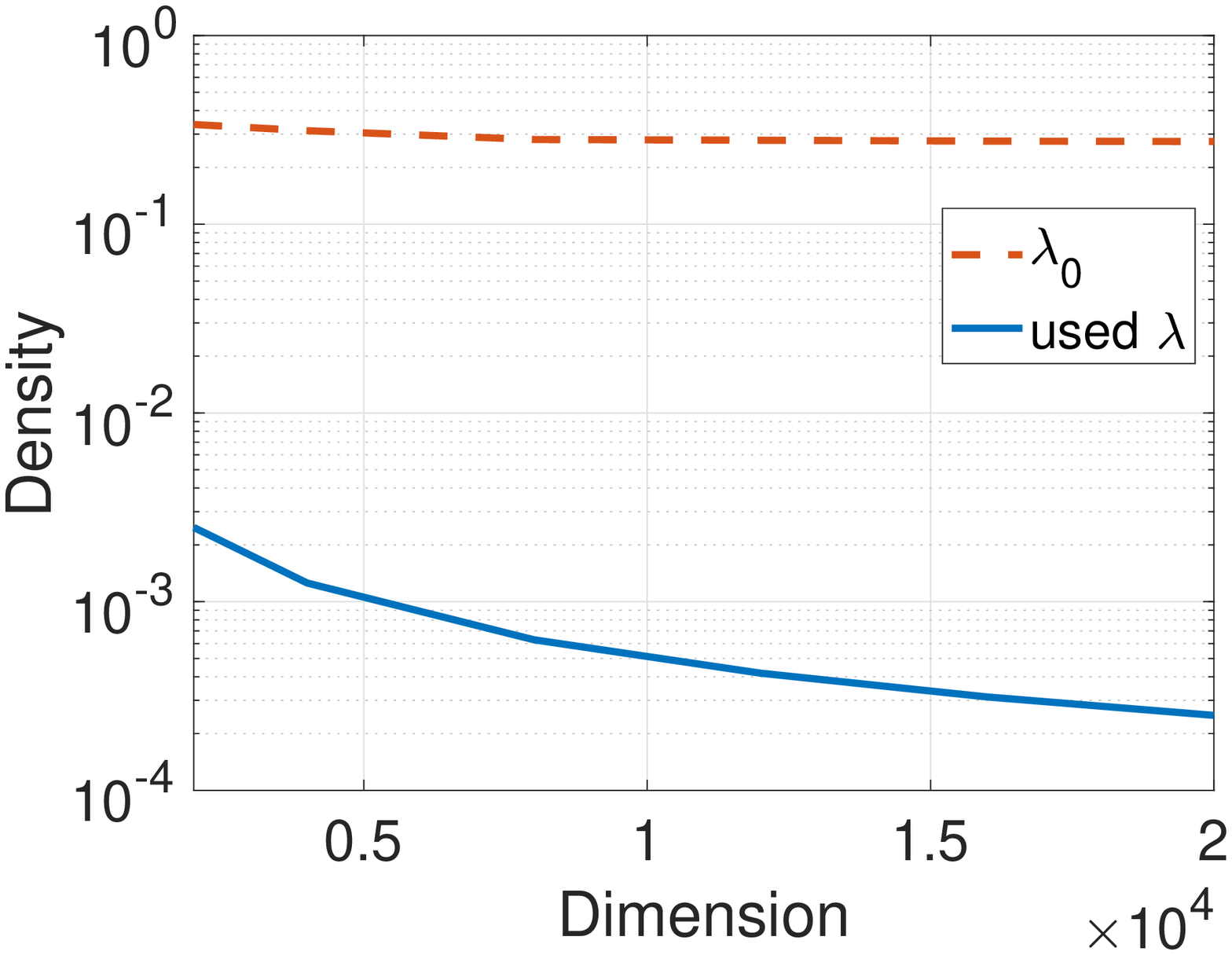}}
	\caption{ \footnotesize a) Values of $\lambda_0$, used $\lambda$, and theoretical $\lambda$, b) The density of the true inverse and thresholded sample covariance matrices.}
	\label{min_eigen}
\end{figure*}

\begin{table}\centering
	\begin{tabular}{ c|c||c|c|c||c|c|c||c|c|c}
		& & \multicolumn{3}{|c||}{Closed-Form} & \multicolumn{3}{|c||}{Graphical Lasso} & \multicolumn{3}{|c}{Elementary} \\
		\hline
		$d$ & $m$ & $\ell_F$ & TPR & FPR & $\ell_F$ & TPR & FPR & $\ell_F$ & TPR & FPR \\ 
		\hline
		$2000$ & $9894$ & $0.41$ & $0.71$ & $0.00$ & $0.41$ & $0.71$ & $0.00$ & $0.40$ & $0.63$ & $0.00$\\
		\hline
		$2000$ & $20022$ & $0.50$ & $0.59$ & $0.00$ & $0.65$ & $0.59$ & $0.00$ & $0.49$ & $0.34$ & $0.01$\\
		\hline
		$4000$ & $20094$ & $0.39$ & $0.83$ & $0.00$ & $0.38$ & $0.84$ & $0.00$ & $0.37$ & $0.76$ & $0.00$\\
		\hline
		$4000$ & $40382$ & $0.48$ & $0.74$ & $0.00$ & $0.48$ & $0.75$ & $0.00$ & $0.48$ & $0.54$ & $0.00$\\
		\hline
		$8000$ & $40218$ & $0.36$ & $0.92$ & $0.00$ & $0.35$ & $0.93$ & $0.00$ & $0.33$ & $0.87$ & $0.00$\\
		\hline
		$8000$ & $79890$ & $0.45$ & $0.87$ & $0.00$ & $0.44$ & $0.88$ & $0.00$ & $0.44$ & $0.71$ & $0.00$\\
		\hline
		$12000$ & $60192$ & $0.33$ & $0.96$ & $0.00$ & $0.32$ & $0.97$ & $0.00$ & $0.30$ & $0.93$ & $0.00$\\
		\hline
		$12000$ & $119676$ & $0.43$ & $0.93$ & $0.00$ & $0.41$ & $0.94$ & $0.00$ & $0.42$ & $0.81$ & $0.00$\\
		\hline
		$16000$ & $80064$ & $0.32$ & $0.97$ & $0.00$ & $0.30$ & $0.98$ & $0.00$ & $0.28$ & $0.96$ & $0.00$\\
		\hline
		$16000$ & $160094$ & $0.42$ & $0.95$ & $0.00$ & $0.40$ & $0.96$ & $0.00$ & $0.40$ & $0.86$ & $0.00$\\
		\hline
		$20000$ & $99954$ & $0.31$ & $0.99$ & $0.00$ & $0.30$ & $0.99$ & $0.00$ & $0.28$ & $0.96$ & $0.00$\\
		\hline
		$20000$ & $200018$ & $0.41$ & $0.96$ & $0.00$ & $0.39$ & $0.97$ & $0.00$ & $0.39$ & $0.89$ & $0.00$\\
		\hline
		$40000$ & $200290$ & $0.28$ & $1.00$ & $0.00$ & $\star$ & $\star$ & $\star$ & $\star$ & $\star$ & $\star$\\
		\hline
		$80000$ & $401798$ & $0.27$ & $1.00$ & $0.00$ & $\star$ & $\star$ & $\star$ & $\star$ & $\star$ & $\star$\\
		
	\end{tabular}
	\caption{ \footnotesize {The accuracy of different methods.}}\label{Table2}
\end{table}

\section{Conclusions}
\label{sec:con}

Graphical Lasso (GL) is a popular  method for finding the conditional independence between the entries of a random vector. This technique  aims at  learning the sparsity pattern of the  inverse covariance matrix from a limited number of samples, based on the regularization of a positive-definite matrix. Motivated by the computational complexity of solving the GL for large-scale problems, this paper provides conditions under which the GL behaves the same as the simple method of thresholding the sample covariance matrix. The conditions make direct use of the sample covariance matrix and are not based on the solution of the GL. 
More precisely, it is shown that the GL and thresholding techniques are equivalent if: (i) a certain matrix formed based on the sample covariance matrix is both sign-consistent and inverse-consistent, and (ii)  the gap between the largest thresholded and the smallest un-thresholded entries of the sample covariance matrix is not too small. Although the GL is believed to  be a difficult conic optimization problem, it is proved that it indeed has a closed-form solution in the case where the sparsity pattern of the solution is known to be acyclic. This result is then extended to general sparse graphs and an explicit formula is derived as an approximate solution of the GL, where the approximation error is also quantified in terms of the structure of the sparsity graph.  {The significant speedup and graceful scalability of the proposed explicit formula compared to other state-of-the-art methods is showcased on different real-world and randomly generated data sets. }

\vskip 0.2in
\bibliographystyle{IEEEtran}
\bibliography{reference_GL_VS_TH}

\begin{thebibliography}{10}
\providecommand{\url}[1]{#1}
\csname url@samestyle\endcsname
\providecommand{\newblock}{\relax}
\providecommand{\bibinfo}[2]{#2}
\providecommand{\BIBentrySTDinterwordspacing}{\spaceskip=0pt\relax}
\providecommand{\BIBentryALTinterwordstretchfactor}{4}
\providecommand{\BIBentryALTinterwordspacing}{\spaceskip=\fontdimen2\font plus
\BIBentryALTinterwordstretchfactor\fontdimen3\font minus
  \fontdimen4\font\relax}
\providecommand{\BIBforeignlanguage}[2]{{%
\expandafter\ifx\csname l@#1\endcsname\relax
\typeout{** WARNING: IEEEtran.bst: No hyphenation pattern has been}%
\typeout{** loaded for the language `#1'. Using the pattern for}%
\typeout{** the default language instead.}%
\else
\language=\csname l@#1\endcsname
\fi
#2}}
\providecommand{\BIBdecl}{\relax}
\BIBdecl

\bibitem{Coleman90}
T.~F. Coleman and Y.~Li, Eds., \emph{Large-scale numerical optimization}.\hskip
  1em plus 0.5em minus 0.4em\relax SIAM, 1990, vol.~46.

\bibitem{Bach11}
F.~Bach, R.~Jenatton, J.~Mairal, and G.~Obozinski, ``Optimization with
  sparsity-inducing penalties,'' \emph{Foundations and
  Trends$^{\textregistered}$ in Machine Learning}, vol.~4, no.~1, pp. 1--106,
  2012.

\bibitem{Benson}
S.~J. Benson, Y.~Ye, and X.~Zhang, ``Solving large-scale sparse semidefinite
  programs for combinatorial optimization,'' \emph{SIAM Journal on
  Optimization}, vol.~10, no.~2, pp. 443--461, 2000.

\bibitem{Garcke01}
J.~Garcke, M.~Griebel, and M.~Thess, ``Data mining with sparse grids,''
  \emph{Computing}, vol.~67, no.~3, pp. 225--253, 2001.

\bibitem{Muth05}
S.~Muthukrishnan, ``Data streams: Algorithms and applications,''
  \emph{Foundations and Trends$^{\textregistered}$ in Theoretical Computer
  Science}, vol.~1, no.~2, pp. 117--236, 2005.

\bibitem{Wu14}
X.~Wu, X.~Zhu, G.~Q. Wu, and W.~Ding, ``Data mining with big data,'' \emph{IEEE
  Transactions on Knowledge and Data Engineering}, vol.~26, no.~1, pp. 97--107,
  2014.

\bibitem{Wright10}
J.~Wright, Y.~Ma, J.~Mairal, G.~Sapiro, T.~S. Huang, and S.~Yan, ``Sparse
  representation for computer vision and pattern recognition,''
  \emph{Proceedings of the IEEE}, vol.~98, no.~6, pp. 1031--1044, 2010.

\bibitem{Qiao10}
L.~Qiao, S.~Chen, and X.~Tan, ``Sparsity preserving projections with
  applications to face recognition,'' \emph{Pattern Recognition}, vol.~43,
  no.~1, pp. 331--341, 2010.

\bibitem{Sojoudi14}
S.~Sojoudi and J.~Doyle, ``Study of the brain functional network using
  synthetic data,'' \emph{52nd Annual Allerton Conference on Communication,
  Control, and Computing (Allerton)}, pp. 350--357, 2014.

\bibitem{Fardad11}
M.~Fardad, F.~Lin, and M.~R. Jovanovi{\'c}, ``Sparsity-promoting optimal
  control for a class of distributed systems,'' \emph{American Control
  Conference}, pp. 2050--2055, 2011.

\bibitem{SODC2016}
S.~Fattahi and J.~Lavaei, ``On the convexity of optimal decentralized control
  problem and sparsity path,'' in \emph{American Control Conference (ACC),
  2017}.\hskip 1em plus 0.5em minus 0.4em\relax IEEE, 2017, pp. 3359--3366.

\bibitem{Candes07}
E.~Candes and J.~Romberg, ``Sparsity and incoherence in compressive sampling,''
  \emph{Inverse Problems}, vol.~23, no.~3, pp. 969--985, 2007.

\bibitem{Simon13}
S.~Foucart and H.~Rauhut, \emph{A mathematical introduction to compressive
  sensing}.\hskip 1em plus 0.5em minus 0.4em\relax Basel: Birkh{\"a}user, 2013,
  vol.~1, no.~3.

\bibitem{buhlmann2011statistics}
P.~B{\"u}hlmann and S.~Van De~Geer, \emph{Statistics for high-dimensional data:
  methods, theory and applications}.\hskip 1em plus 0.5em minus 0.4em\relax
  Springer Science \& Business Media, 2011.

\bibitem{fan2010selective}
J.~Fan and J.~Lv, ``A selective overview of variable selection in high
  dimensional feature space,'' \emph{Statistica Sinica}, vol.~20, no.~1, p.
  101, 2010.

\bibitem{friedman2008sparse}
J.~Friedman, T.~Hastie, and R.~Tibshirani, ``Sparse inverse covariance
  estimation with the graphical lasso,'' \emph{Biostatistics}, vol.~9, no.~3,
  pp. 432--441, 2008.

\bibitem{banerjee2008model}
O.~Banerjee, L.~El~Ghaoui, and A.~d'Aspremont, ``Model selection through sparse
  maximum likelihood estimation for multivariate gaussian or binary data,''
  \emph{Journal of Machine Learning Research}, vol.~9, pp. 485--516, 2008.

\bibitem{yuan2007model}
M.~Yuan and Y.~Lin, ``Model selection and estimation in the gaussian graphical
  model,'' \emph{Biometrika}, vol.~94, no.~1, pp. 19--35, 2007.

\bibitem{liu2010stability}
H.~Liu, K.~Roeder, and L.~Wasserman, ``Stability approach to regularization
  selection (stars) for high dimensional graphical models,'' in \emph{Advances
  in Neural Information Processing Systems}, 2010, pp. 1432--1440.

\bibitem{kramer2009regularized}
N.~Kr{\"a}mer, J.~Sch{\"a}fer, and A.-L. Boulesteix, ``Regularized estimation
  of large-scale gene association networks using graphical gaussian models,''
  \emph{BMC bioinformatics}, vol.~10, no.~1, p. 384, 2009.

\bibitem{danaher2014joint}
P.~Danaher, P.~Wang, and D.~M. Witten, ``The joint graphical lasso for inverse
  covariance estimation across multiple classes,'' \emph{Journal of the Royal
  Statistical Society: Series B (Statistical Methodology)}, vol.~76, no.~2, pp.
  373--397, 2014.

\bibitem{sojoudi2014equivalence111}
S.~Sojoudi, ``Equivalence of graphical lasso and thresholding for sparse
  graphs,'' \emph{Journal of Machine Learning Research}, vol.~17, no. 115, pp.
  1--21, 2016.

\bibitem{Hsieh14}
C.-J. Hsieh, M.~A. Sustik, I.~S. Dhillon, and P.~Ravikumar, ``Quic: quadratic
  approximation for sparse inverse covariance estimation,'' \emph{Journal of
  Machine Learning Research}, vol.~15, no.~1, pp. 2911--2947, 2014.

\bibitem{Mazumdar12}
R.~Mazumder and T.~Hastie, ``Exact covariance thresholding into connected
  components for large-scale graphical lasso,'' \emph{Journal of Machine
  Learning Research}, vol.~13, pp. 781--794, 2012.

\bibitem{Witten11}
D.~M. Witten, J.~H. Friedman, and N.~Simon, ``New insights and faster
  computations for the graphical lasso,'' \emph{Journal of Computational and
  Graphical Statistics}, vol.~20, no.~4, pp. 892--900, 2011.

\bibitem{ghaoui2010safe}
L.~E. Ghaoui, V.~Viallon, and T.~Rabbani, ``Safe feature elimination for the
  lasso and sparse supervised learning problems,'' \emph{arXiv preprint
  arXiv:1009.4219}, 2010.

\bibitem{tibshirani2012strong}
R.~Tibshirani, J.~Bien, J.~Friedman, T.~Hastie, N.~Simon, J.~Taylor, and R.~J.
  Tibshirani, ``Strong rules for discarding predictors in lasso-type
  problems,'' \emph{Journal of the Royal Statistical Society: Series B
  (Statistical Methodology)}, vol.~74, no.~2, pp. 245--266, 2012.

\bibitem{fercoq2015mind}
O.~Fercoq, A.~Gramfort, and J.~Salmon, ``Mind the duality gap: safer rules for
  the lasso,'' \emph{arXiv preprint arXiv:1505.03410}, 2015.

\bibitem{ndiaye2015gap}
E.~Ndiaye, O.~Fercoq, A.~Gramfort, and J.~Salmon, ``Gap safe screening rules
  for sparse multi-task and multi-class models,'' in \emph{Advances in Neural
  Information Processing Systems}, 2015, pp. 811--819.

\bibitem{Richard17}
R.~Y. Zhang and J.~Lavaei, ``Modified interior-point method for
  large-and-sparse low-rank semidefinite programs,'' \emph{56th IEEE Conference
  on Decision and Control}, 2017.

\bibitem{Vertes12}
P.~E. V{\'e}rtes, A.~F. Alexander-Bloch, N.~Gogtay, J.~N. Giedd, J.~L.
  Rapoport, and E.~T. Bullmore, ``Simple models of human brain functional
  networks,'' \emph{Proceedings of the National Academy of Sciences}, vol. 109,
  no.~15, pp. 5868--5873, 2012.

\bibitem{Lino01}
L.~Figueiredo, I.~Jesus, J.~T. Machado, J.~R. Ferreira, and J.~M.~D. Carvalho,
  ``Towards the development of intelligent transportation systems,'' \emph{IEEE
  Intelligent Transportation Systems}, pp. 1206--1211, 2001.

\bibitem{Yin02}
H.~Yin, S.~C. Wong, J.~Xu, and C.~K. Wong, ``Urban traffic flow prediction
  using a fuzzy-neural approach,'' \emph{Transportation Research Part C:
  Emerging Technologies}, vol.~10, no.~2, pp. 85--98, 2017.

\bibitem{Rafegh17}
H.~Nassiri and R.~Aghamohammadi, ``A new analytic neuro-fuzzy model for work
  zone capacity estimation,'' \emph{Transportation Research Board 96th Annual
  Meeting}, vol.~17, no. 06061, 2017.

\bibitem{Sun12}
R.~H. Sun, Shiliang and Y.~Gao, ``Network-scale traffic modeling and
  forecasting with graphical lasso and neural networks,'' \emph{Journal of
  Transportation Engineering}, vol. 138, no.~11, pp. 1358--1367, 2012.

\bibitem{PeMS}
\BIBentryALTinterwordspacing
(2017) {\rm Caltrans Performance Management System (PeMS)}. [Online].
  Available: \url{http://pems.dot.ca.gov}
\BIBentrySTDinterwordspacing

\bibitem{yun2011coordinate}
S.~Yun and K.-C. Toh, ``A coordinate gradient descent method for
  $l_1$-regularized convex minimization,'' \emph{Computational Optimization and
  Applications}, vol.~48, no.~2, pp. 273--307, 2011.

\bibitem{yang2014elementary}
E.~Yang, A.~C. Lozano, and P.~K. Ravikumar, ``Elementary estimators for
  graphical models,'' in \emph{Advances in neural information processing
  systems}, 2014, pp. 2159--2167.

\bibitem{ravikumar2011high}
P.~Ravikumar, M.~J. Wainwright, G.~Raskutti, B.~Yu \emph{et~al.},
  ``High-dimensional covariance estimation by minimizing ℓ1-penalized
  log-determinant divergence,'' \emph{Electronic Journal of Statistics},
  vol.~5, pp. 935--980, 2011.

\bibitem{Ravindra93}
R.~K. Ahuja, T.~L. Magnanti, and J.~B. Orlin, \emph{Network flows: theory,
  algorithms, and applications}.\hskip 1em plus 0.5em minus 0.4em\relax
  Pearson, 1993.

\bibitem{Lieven15}
L.~Vandenberghe and M.~S. Andersen, ``Chordal graphs and semidefinite
  optimization,'' \emph{Foundations and Trends$^{\textregistered}$ in
  Optimization}, vol.~1, no.~4, pp. 241--433, 2015.

\end{thebibliography}

\appendix
\section*{Appendix}\label{Appendix}

In what follows, the omitted technical proofs will be presented. A number of lemmas are required for this purpose. 
\vspace{2mm}

{Before presenting the proof of Theorem~\ref{thm:tt1}, consider the \textit{normalized GL}, defined as
	\begin{equation}
	\label{eq_corr}
	\min_{S\in\mathbb S^d_+} -\log\det(S)+\mathrm{trace}(\tilde\Sigma S)+\sum_{i\not=j}\tilde{\lambda}_{ij}|S_{ij}|,
	\end{equation}
	where $\tilde{\Sigma}$ is the normalized sample covariance, i.e., $\tilde{\Sigma}_{ij} = \frac{\Sigma_{ij}}{\sqrt{\Sigma_{ii}\Sigma_{jj}}}$ for every $(i,j)\in\{1,2,...,d\}^2$ (also known as sample correlation matrix). Similarly, $\tilde{\lambda}_{ij}$ is defined as $\frac{\lambda}{\sqrt{\Sigma_{ii}\Sigma_{jj}}}$. Upon denoting the optimal solution of the normalized GL as $\tilde{S}$, we consider the relationship between $\tilde{S}$ and $S^{\mathrm{opt}}$. Recall that $D$ is defined as a matrix collecting the diagonal elements of $\Sigma$.
	\begin{lemma}\label{l26}
		We have $S^{\mathrm{opt}} = D^{-1/2} \tilde{S} D^{-1/2}$.
	\end{lemma}
	
	\begin{proof}
		Notice that the GL~\eqref{eq1} can be re-written as follows
		\begin{equation}\label{eq11}
		\min_{{S}\in\mathbb S^d_+} -\log\det({S})+\mathrm{trace}(\tilde{\Sigma}D^{1/2}{S}D^{1/2})+ \sum_{i\not=j}\lambda|{S}_{ij}|,
		\end{equation}
		where we have used the equality
		\begin{equation}\notag
		\mathrm{trace}({\Sigma}{S}) = \mathrm{trace}(D^{1/2}\tilde{\Sigma}D^{1/2}{S}) = \mathrm{trace}(\tilde{\Sigma}D^{1/2}{S}D^{1/2}).
		\end{equation}
		Upon defining 
		\begin{equation}\label{Stilde}
		\tilde{S} = D^{1/2} S D^{1/2}
		\end{equation}
		and following some algebra, one can verify that~\eqref{eq11} is equivalent to
		\begin{equation}
		\label{RGL4}
		\min_{\tilde{S}\in\mathbb S^d_+} -\log\det(\tilde{S})+\mathrm{trace}(\tilde{\Sigma} \tilde{S})+ \sum_{i\not=j}\tilde{\lambda}_{ij}|\tilde{S}_{ij}|+\log\det(D).
		\end{equation}
		Dropping the constant term in~\eqref{RGL4} gives rise to the normalized GL~\eqref{eq_corr}. Therefore, $S^{\mathrm{opt}} = D^{-1/2} \tilde{S} D^{-1/2}$ holds in light of~\ref{Stilde}. This completes the proof. 
\end{proof}}

{\noindent \textit{\bf Proof of Theorem~\ref{thm:tt1}}
	Note that, due to the Definition~\ref{residue} and Lemma~\ref{l26}, $\tilde{\Sigma}^{\mathrm{res}}$ and $\tilde{S}$ have the same sparsity pattern as ${\Sigma}^{\mathrm{res}}$ and ${S}^{\mathrm{opt}}$, respectively.
	Therefore, it suffices to show that the sparsity structures of $\tilde{\Sigma}^{\mathrm{res}}$ and $\tilde{S}$ are the same.  
	
	To verify this, we focus on the optimality conditions for optimization~\eqref{eq_corr}. Define $M$ as $I_d+\tilde{\Sigma}^{\mathrm{res}}$. Due to Condition~(1-i) and Lemma~\ref{lemma:ll1}, $M$ is inverse-consistent and has a unique inverse-consistent complement, which is denoted by $N$. First, will show that $(M+N)^{-1}$ is the optimal solution of~\eqref{eq_corr}. For an arbitrary pair $(i,j)\in\{1,...,d\}^2$, the KKT conditions, introduced in Lemma~\ref{expsol}, imply that one of the following cases holds:
	\begin{itemize}
		\item[1)] $i=j$: We have $(M+N)_{ij}= M_{ii}=\tilde{\Sigma}_{ii}$.
		\item[2)] $(i,j)\in \mathrm{supp}(\tilde{\Sigma}^{\mathrm{res}})$: In this case, we have
		\begin{equation}\notag
		(M+N)_{ij} = M_{ij} = \tilde{\Sigma}_{ij}-\tilde{\lambda}_{ij}\times \mathrm{sign}(\tilde{\Sigma}_{ij}).
		\end{equation}
		Note that since $|\Sigma_{ij}|>\lambda$, we have that $\mathrm{sign}(M_{ij}) = \mathrm{sign}(\tilde{\Sigma}_{ij})$. On the other hand, due to the sign-consistency of $M$, we have $\mathrm{sign}(M_{ij}) = -\mathrm{sign}\left(\left((M+N)^{-1}\right)_{ij}\right)$. This implies that
		\begin{equation}\notag
		(M+N)_{ij} = M_{ij} = \tilde{\Sigma}_{ij}+\tilde{\lambda}_{ij}\times \mathrm{sign}((M+N)^{-1}).
		\end{equation}
		\item[3)] $(i,j)\not\in \mathrm{supp}(\tilde{\Sigma}^{\mathrm{res}})$: One can verify that $(M+N)_{ij} = N_{ij}$. Therefore, due to Condition (1-iii), we have
		\begin{equation}\label{eq39}
		\begin{aligned}
		|(M+N)_{ij}|&\leq \beta\left(\mathrm{supp}(\tilde{\Sigma}^{\mathrm{res}}),\|\tilde{\Sigma}^{\mathrm{res}}\|_{\max}\right)\\
		&\leq\underset{
			\begin{subarray}{c}
			k\not=l\\
			(k,l)\not\in\mathrm{supp}(\Sigma^{\mathrm{res}})
			\end{subarray}
		} {\min}\frac{\lambda-|\Sigma_{kl}|}{\sqrt{\Sigma_{kk}\Sigma_{ll}}}\\
		&=\underset{
			\begin{subarray}{c}
			k\not=l\\
			(k,l)\not\in\mathrm{supp}(\Sigma^{\mathrm{res}})
			\end{subarray}
		} {\min}{\tilde\lambda_{kl}-|\tilde\Sigma_{kl}|}.
		\end{aligned}
		\end{equation}
		This leads to
		\begin{equation}\label{eq44}
		|(M+N)_{ij}-\tilde{\Sigma}_{ij}|\leq |(M+N)_{ij}|+|\tilde{\Sigma}_{ij}|\leq \underset{
			\begin{subarray}{c}
			k\not=l\\
			(k,l)\not\in\mathrm{supp}(\Sigma^{\mathrm{res}})
			\end{subarray}
		} {\min}\left({\tilde\lambda_{kl}-|\tilde\Sigma_{kl}|}\right)+|\tilde{\Sigma}_{ij}|\leq \tilde{\lambda}_{ij}.
		\end{equation}
	\end{itemize}
	Therefore, it can be concluded that $(M+N)^{-1}$ satisfies the KKT conditions for~\eqref{eq_corr}\footnote{The KKT conditions for the normalized GL are equivalent to~\eqref{optcon} after replacing $\lambda$ with $\tilde{\lambda}_{ij}$}. On the other hand, note that $\rm{supp}((M+N)^{-1})= \mathrm{supp}(\tilde{\Sigma}^{\mathrm{res}})$. This concludes the proof.~\hfill$\blacksquare$}

\vspace{2mm}

To proceed with the proof of Lemma~\ref{thm:tt4}, we need the following lemma.

\begin{lemma} \label{lemma:ll2}
	
	Consider a {matrix  $M\in\mathbb S^d$ with positive-definite completion}. Assume that 
	$\|M^{(c)}\|_1\leq \eta \|M-I_d\|_1$ and $\|M-I_d\|_1<\frac{1}{\eta+1}$, for some number $\eta$. The relation
	\begin{equation}\notag
	\|M^{(c)}\|_1\leq (1+\eta)^2\frac{\|M-I_d\|_1^2}{1-(\eta+1)\|M-I_d\|_1}
	\end{equation}
	holds.
\end{lemma}

\vspace{2mm}

\begin{proof}
	{Note that $M\in\mathbb S^d$ has a positive-definite completion and hence, is inverse-consistent due to Lemma~\ref{lemma:ll1}.} One can write
	\begin{equation}
	\begin{aligned}\notag
	\|(M-I_d)+M^{(c)}\|_1&\leq \|M-I_d\|_1+\|M^{(c)}\|_1\leq (\eta+1)\|M-I_d\|_1<1.
	\end{aligned}
	\end{equation}
	Therefore,
	\begin{equation}
	\begin{aligned}\notag
	(M+M^{(c)})^{-1}&=(I_d+(M-I_d+M^{(c)}))^{-1}+I_d- (M-I_d+M^{(c)})\\
	&+(M-I_d+M^{(c)})^2 \times \sum_{i=0}^{\infty} (-M+I_d-M^{(c)})^i.
	\end{aligned}
	\end{equation}
	Since $\mathrm{supp}((M+M^{(c)})^{-1})\subseteq\mathrm{supp}(M)$, 
	it can be concluded that   the  $(i,j)$ entries of $M^{(c)}$ and
	\begin{equation}\notag
	(M-I_d+M^{(c)})^2 \times \sum_{i=0}^{\infty} (-M+I_d-M^{(c)})^i
	\end{equation}
	are equal for every $(i,j)\in\mathrm{supp}(M^{(c)})$.
	Since  the $(i,j)$ entry of $M^{(c)}$ is zero if $(i,j)\not\in\mathrm{supp}(M^{(c)})$, we have 
	\begin{equation}\notag
	\|M^{(c)}\|_1\leq \left\|(M-I_d+M^{(c)})^2 \sum_{i=0}^{\infty} (M-I_d+M^{(c)})^i\right\|_1.
	\end{equation}
	Since 1-norm is sub-multiplicative, 
	the above inequality can be simplified as
	\begin{equation}
	\begin{aligned}\notag
	\|M^{(c)}\|_1&\leq (\|M-I_d\|_1+\|M^{(c)}\|_1)^2\times  \sum_{i=0}^{\infty} (\|M-I_d\|_1+\|M^{(c)}\|_1)^i\\
	&=\frac{ (\|M-I_d\|_1+\|M^{(c)}\|_1)^2 }{1-\|M-I_d\|_1-\|M^{(c)}\|_1}\\
	&\leq \frac{ (\|M-I_d\|_1+\eta\|M-I_d\|\|_1)^2 }{1-\|M-I_d\|_1-\eta\|M-I_d\|\|_1}\\
	&=(1+\eta)^2\frac{\|M-I_d\|_1^2}{1-(\eta+1)\|M-I_d\|_1}.
	\end{aligned}
	\end{equation}
	This completes the proof.\end{proof}

\vspace{2mm}

\noindent \textit{\bf Proof of Lemma~\ref{thm:tt4}}
Given an arbitrary graph $\mathcal G$, consider a matrix variable $M$ with 1's on the diagonal such that $\mathrm{supp}(M)\subseteq\mathcal G$.  The first objective is to find a  matrix in terms of $M$, denoted by the matrix function $N(M)$, satisfying the following properties
\begin{subequations}
	\begin{align}
	&\mathrm{supp}\left((M+N(M))^{-1}\right)\subseteq\mathcal G,\notag\\
	&\mathrm{supp}(N(M))\subseteq \mathcal G^{(c)}.\notag
	\end{align}
\end{subequations}
To this end, define the matrix function $A(M)$ as
\begin{equation}\notag
A(M)=(M+N(M))^{-1}.
\end{equation}
Observe that
\begin{itemize}
	\item As long as $A(M)$ exists and $\mathrm{supp}(A(M))\subseteq\mathcal G$, there is a continuously differentiable mapping from $A(M)$ to $M$  because $M$ can be found by setting those entries of $A(M)^{-1}$ corresponding to the edges of $\mathcal G^{(c)}$ to zero. Moreover, the Jacobian of this mapping has full rank at $M=I_d$. Due to the inverse function theorem, the mapping from $M$ to $A(M)$ exists and is continuously differentiable. 
	\item Similarly, as long as $A(M)$ exists and $\mathrm{supp}(A(M))\subseteq\mathcal G$, there is a continuously differentiable mapping from $A(M)$ to $N(M)$.
	\item If $M=I_d$, then $N(M)=0$.
\end{itemize}
It follows from the above properties that if $M$ is sufficiently small, the function $N(M)$ exists and satisfies the following properties: (i) $0=N(I_d)$, and (ii) $N(\cdot)$ is differentiable at $M=I_d$. This implies that there are sufficiently small nonzero numbers $\eta$ and $\alpha_0$ such that $\|N(M)\|_1\leq \eta \|M-I_d\|_1$ whenever $\|M\|_{\max}\leq \alpha_0$. Now, it follows from Lemma~\ref{lemma:ll2} that
\begin{equation}\notag
\|N(M)\|_1\leq (1+\eta)^2\frac{\|M-I_d\|_1^2}{1-(\eta+1)\|M-I_d\|_1},
\end{equation}
or
\begin{equation}\notag
\|N(M)\|_{\max}\leq \frac{(1+\eta)^2\times ( \text{deg}(\mathcal G))^2}{1-(\eta+1)\alpha_0\times \text{deg}(\mathcal G)} \|M\|_{\max}^2,
\end{equation}
if $\|M\|_{\max}\leq \alpha_0$. The inequality~\eqref{eq_p5} is satisfied for the number $\zeta$ defined as the maximum of 
\begin{equation}\notag
\frac{(1+\eta)^2\times ( \text{deg}(\mathcal G))^2}{1-(\eta+1)\alpha_0\times \text{deg}(\mathcal G)}
\end{equation}
and the finite number
\begin{equation}\notag
\max\left\{\frac{\beta(\mathcal G,\alpha)}{\alpha^2}\bigg| \alpha\in(\alpha_0,1)\right\}.
\end{equation}
This completes the proof.
~\hfill$\blacksquare$

\vspace{2mm}

{\noindent \textit{\bf Proof of Lemma~\ref{l_sign}} It can be easily verified that 
	\begin{equation}\notag
	(M+M^{(c)})^{-1} = I-(M+M^{(c)}-I)+(M+M^{(c)})^{-1}(M+M^{(c)}-I)^2.
	\end{equation}
	This implies that, for a given pair $(i,j)\in\mathcal{G}$, one can write
	\begin{equation}\label{sign5}
	\left((M+M^{(c)})^{-1}\right)_{ij} = -M_{ij}+\left((M+M^{(c)})^{-1}\right)_{i:}\left((M+M^{(c)}-I)^2\right)_{:j},
	\end{equation}
	where $\left((M+M^{(c)})^{-1}\right)_{i:}$ and $\left((M+M^{(c)}-I)^2\right)_{:j}$ are the $i^{\text{th}}$ row and $j^{\text{th}}$ column of $(M+M^{(c)})^{-1}$ and $(M+M^{(c)}-I)^2$, respectively. Based on~\eqref{sign5}, the $(i,j)$ entries of $M$ and $(M+M^{(c)})^{-1}$ have opposite signs if 
	\begin{equation}\label{sign}
	|M_{ij}|> \left|\left((M+M^{(c)})^{-1}\right)_{i:}\left((M+M^{(c)}-I)^2\right)_{:j}\right|.
	\end{equation}
	To streamline the presentation, $\|M\|_{\max}$ is redefined as $\max_{i,j}|M_{ij}|$ in the rest of the proof. One can write
	\begin{align}\label{sign2}
	\left\|(M\!\!+\!\!M^{(c)}\!-\!I)^2\right\|_{\max}\!\!&\leq\! \left\|(M\!-\!I)^2\right\|_{\max}\!\!+\!\left\|\left(M^{(c)}\right)^2\right\|_{\max}\!\!\!\!\!+\!\left\|M^{(c)}(M\!-\!I)\right\|_{\max}\!\!\!\!+\!\left\|(M\!-\!I)M^{(c)}\right\|_{\max}\nonumber\\
	&\leq \mathrm{deg}(\mathcal{G})\alpha^2+(d-\mathrm{deg}(\mathcal{G}))\zeta(\mathcal{G})^2\alpha^4+2\mathrm{deg}(\mathcal{G})\zeta(\mathcal{G})\alpha^3\nonumber\\
	&\leq 3\mathrm{deg}(\mathcal{G})\max\{\alpha^2,\zeta(\mathcal{G})\alpha^3\}+(d-\mathrm{deg}(\mathcal{G}))\zeta(\mathcal{G})^2\alpha^4\nonumber\\
	&\leq K\alpha^2,
	\end{align}
	for some $K$ that only depends on $\mathrm{deg}(\mathcal{G}), \zeta(\mathcal{G})$, and $d$. Furthermore, assume that 
	\begin{equation}\label{assum}
	\alpha\leq\frac{1}{2\mathrm{deg}(\mathcal{G})\sqrt{\zeta(\mathcal{G})}} = \alpha_0(\mathcal{G}).
	\end{equation}
	Note that
	\begin{equation}\notag
	(M+M^{(c)})^{-1} = I-(M+M^{(c)}-I)(M+M^{(c)})^{-1},
	\end{equation}
	which implies that
	\begin{equation}\label{sign6}
	\left\|(M+M^{(c)})^{-1}\right\|_{\max} = 1+\mathrm{deg}(\mathcal{G})\max\{\alpha,\zeta(\mathcal{G})\alpha^2\}\left\|(M+M^{(c)})^{-1}\right\|_{\max},
	\end{equation}
	where we have used the fact that $\mathrm{supp}((M+M^{(c)})^{-1})\subseteq\mathcal{G}$ and hence, its maximum degree is upper bounded by $\mathrm{deg}(\mathcal{G})$.~\eqref{sign6}, together with the assumption~\eqref{assum} implies that
	\begin{equation}\label{sign4}
	\left\|(M+M^{(c)})^{-1}\right\|_{\max}\leq\frac{1}{1-\mathrm{deg}(\mathcal{G})\max\{\alpha,\zeta(\mathcal{G})\alpha^2\}}\leq 2.
	\end{equation}
	Combining~\eqref{sign2} and~\eqref{sign4} with~\eqref{sign} completes the proof.~\hfill$\blacksquare$}

\vspace{2mm}

\noindent \textit{\bf Proof of Lemma~\ref{l_beta}}
Without loss of generality,   assume that $\mathcal{G}$ is a tree. Note that if there are disjoint components, the argument made in the sequel can be applied to each connected component of $\mathcal{G}$ separately. 
Let $d_{ij}$ denote the unique path between every two disparate nodes $i$ and $j$ in $\mathcal{G}$. Furthermore, define $\mathcal N(i)$ as the set of all neighbors of node $i$ in $\mathcal{G}$. Consider a {matrix $M$ with positive-definite completion} and with diagonal elements  equal to 1 such that  $\|M\|_{\max}\leq\alpha$ and  $\text{supp}(M) = \mathcal{G}$. Let $N$ be a matrix with the following entries
\begin{equation}\label{N}
N_{ij} = \left\{
\begin{array}{ll}
\prod_{(m,t)\in d_{ij}}M_{mt} & \text{if}\quad (i,j)\in(\mathrm{supp}(M))^{(c)},\\
0 & \text{otherwise}.
\end{array} 
\right.
\end{equation}
Moreover, define
\begin{equation}\label{matA}
A_{ij} = \left\{
\begin{array}{ll}
1+\sum_{m\in \mathcal N(i)}\frac{M_{mi}^2}{1-M_{mi}^2} & \text{if}\quad i=j,\\
\frac{-M_{ij}}{1-M_{ij}^2} & \text{if}\quad (i,j)\in \mathrm{supp}(M),\\
0 & \text{otherwise}.
\end{array} 
\right.
\end{equation}
The goal is to show that the  matrix $N$ is the unique inverse-consistent complement of $M$. First, note that $\text{supp}(N) = (\text{supp}(M))^{(c)}$ and $\text{supp}(M) = \text{supp}(A)$.
Next, it is desirable to prove that $(M+N)^{-1} = A$ or equivalently $(M+N)A = I$. Upon defining $T = (M+N)A$, one can write
\begin{equation}\notag
T_{ii} = \sum_{m=1}^{d}(M_{im}+\mathcal N_{im})A_{mi} = 1+\sum_{m\in \mathcal N(i)}\frac{M_{mi}^2}{1-M_{mi}^2}-\sum_{m\in \mathcal N(i)}\frac{M_{mi}^2}{1-M_{mi}^2} = 1.
\end{equation}
Moreover, for every pair of nodes $i$ and $j$, define $D_{ij} $ as $ \prod_{(k,t)\in d_{ij}}M_{kt}$ if $i\not=j$ and  as $1$ if $i=j$. 

Consider a pair of distinct nodes $i$ and $j$. Let $t$ denote the node adjacent to $j$ in $d_{ij}$ (note that we may have $t = i$). It can be verified that
\begin{align}\label{tij}
T_{ij} =& \sum_{m=1}^{d}(M_{im}+N_{im})A_{mj} = D_{ij}\left(1+\sum_{m\in\mathcal N(j)}\frac{M_{mj}^2}{1-M_{mj}^2}\right)-D_{it}\left(\frac{M_{tj}}{1-M_{tj}^2}\right)\nonumber \\
& -\underset{
	\begin{subarray}{c}
	m\in\mathcal N(j)\\
	m\not=t
	\end{subarray}
}{\sum}D_{im}\frac{M_{mj}}{1-M_{mj}^2}.
\end{align}
Furthermore, 
\begin{align}\label{decomposition}
& D_{ij} = D_{it} M_{tj},\nonumber\\
& D_{im} = D_{it} M_{tj} M_{jm},\qquad \forall \ m\in\mathcal N(j),\ m\not=t.
\end{align}
Plugging \eqref{decomposition} into \eqref{tij} yields that
\begin{equation}\notag
T_{ij} = D_{it}M_{tj}\left(\frac{1}{1-M_{tj}^2}+\underset{
	\begin{subarray}{c}
	m\in\mathcal N(j)\\
	m\not=t
	\end{subarray}
}{\sum}\frac{M_{mj}^2}{1-M_{mj}^2}\right)-D_{it}\left(\frac{M_{tj}}{1-M_{tj}^2}\right) - D_{it}M_{tj}\underset{
	\begin{subarray}{c}
	m\in\mathcal N(j)\\
	m\not=t
	\end{subarray}
}{\sum}\frac{M_{mj}^2}{1-M_{mj}^2} = 0.
\end{equation}
Hence, $T = I$. Finally, we need to show that $M+N\succ 0$. To this end, it suffices to prove that $A\succ 0$. 
Note that $A$ can be written as $I+\sum_{(i,j)\in\mathcal{G}}L^{(i,j)}$, where $L^{(i,j)}$ is defined as
\[  L^{(i,j)}_{rl} = \left\{
\begin{array}{ll}
\frac{M_{ij}^2}{1-M_{ij}^2} &  \text{if}\quad r=l = i\ \text{or}\ j,\\
\frac{-M_{ij}}{1-M_{ij}^2} & \text{if}\quad (r,l) = (i,j),\\
0 & \text{otherwise}.
\end{array} 
\right. \]
Consider the term $x^TAx$ for an arbitrary vector $x\in\mathbb{R}^d$. One can verify that
\begin{align}\label{A}
x^TAx =& \sum_{i=1}^{d}x_i^2\! +\!\!\!\!\sum_{(i,j)\in\mathcal{G}}x^TL^{(i,j)}x \nonumber \\
=& \sum_{i=1}^{d}x_i^2+\!\!\!\!\sum_{(i,j)\in\mathcal{G}}\left(\frac{M_{ij}^2}{1-M_{ij}^2}\right)x_i^2\!+\!\left(\frac{M_{ij}^2}{1-M_{ij}^2}\right)x_j^2\!-\! \left(\frac{2M_{ij}}{1-M_{ij}^2}\right)x_ix_j.
\end{align}
Without loss of generality, assume that the graph is a rooted tree with the root at node $d$. Assume that each edge $(i,j)$ defines a direction that is toward the root. Then, it follows from \eqref{A} that
\begin{align}
x^TAx = & x_d^2 + \sum_{(i,j)\in\mathcal{G}}x_i^2+\left(\frac{M_{ij}^2}{1-M_{ij}^2}\right)x_i^2+\left(\frac{M_{ij}^2}{1-M_{ij}^2}\right)x_j^2
- \left(\frac{2M_{ij}}{1-M_{ij}^2}\right)x_ix_j\nonumber\\
= & x_d^2 + \sum_{(i,j)\in\mathcal{G}}\left(\frac{1}{1-M_{ij}^2}\right)x_i^2+\left(\frac{M_{ij}^2}{1-M_{ij}^2}\right)x_j^2
- \left(\frac{2M_{ij}}{1-M_{ij}^2}\right)x_ix_j\nonumber\\
= & x_d^2+\sum_{(i,j)\in\mathcal{G}}\frac{(x_i-M_{ij}x_j)^2}{1-M_{ij}^2}\geq 0.\notag
\end{align}
Therefore, $M+N\succeq 0$ and subsequently $M+N\succ 0$ (because it is invertible). Hence, according to Definition \ref{def:dd1} and Lemma \ref{lemma:ll1}, the matrix $N$ is the unique inverse-consistent compliment of $M$. On the other hand, it follows from the definition of $N$ that $\|N\|_{\max}\leq \alpha^2$ and consequently $\beta(\mathcal{G},\alpha) \leq \alpha^2$. Now, suppose that $\mathcal{G}$ includes a path of length at least 2, e.g., the edges $(1,2)$ and $(2,3)$ belong to $\mathcal{G}$. By setting $M_{12} = M_{23} = \alpha$ and choosing sufficiently small values for those entries of $M$ corresponding to the remaining  edges in $\mathcal{G}$, the matrix $M$ becomes positive-definite {with a trivial positive-definite completion} and we obtain $\|N\|_{\max} = \alpha^2$. This completes the proof. ~\hfill$\blacksquare$

\vspace{2mm}

\noindent \textit{\bf Proof of Theorem \ref{thm4}} {To prove this theorem, first consider the following matrix
	\begin{equation}\label{S_normal}
	\hat{S}_{ij} = \left\{
	\begin{array}{ll}
	1+\underset{
		(i,m)\in\mathcal{E}^{\text{opt}}
	}{\sum}\frac{({\tilde\Sigma^{\mathrm{res}}_{im}})^2}{1-({\tilde\Sigma^{\mathrm{res}}_{im}})^2} & \text{if}\quad i=j,\\
	\frac{-\tilde\Sigma^{\mathrm{res}}_{ij}}{1-({\tilde\Sigma^{\mathrm{res}}_{ij}})^2} & \text{if}\quad(i,j)\in \mathcal{E}^{\text{opt}},\vspace{1mm}\\
	
	0 & \text{otherwise}.
	\end{array} 
	\right.
	\end{equation}
	In what follows, we will show that $\hat{S} = \tilde{S}$, where $\tilde{S}$ is the optimal solution for the normalized GL. This, together with Lemma~\ref{l26} implies that~\eqref{S_opt2} is indeed optimal for the GL.}

First, note that there exists a matrix $N$ such that $\tilde{S}^{-1} = M+N$, where $M$ is defined as 
\begin{equation}\label{M}
M_{ij} = \left\{
\begin{array}{ll}
\tilde\Sigma_{ij}+\tilde\lambda_{ij}\times \text{sign}(\tilde{S}_{ij}) & \text{if}\quad (i,j)\in\text{supp}(\tilde{S}),\\
1 & \text{if}\quad  i=j,\\
0 & \text{otherwise}.
\end{array} 
\right.
\end{equation}
Clearly, $\text{supp}(\tilde{S}) = \text{supp}(M)$. Furthermore, $M = I_d+\tilde T(\lambda)$, where $(i,j)^{\text{th}}$ entry of $\tilde T(\lambda)$ is equal to $\tilde{\Sigma}_{ij}+\tilde{\lambda}_{ij}\text{sign}(S^{\mathrm{opt}}_{ij})$ for every $(i,j)\in\text{supp}(S^{\mathrm{opt}})$ and it is equal to zero otherwise. Subsequently, $M = D^{-1/2}(D+T(\lambda))D^{-1/2}$ and hence, $D+T(\lambda)\succ 0$ implies $M\succ 0$. By combining $N = (\tilde{S})^{-1}-M$ with \eqref{M} and exploiting the optimality conditions in \eqref{optcon}, one can verify that $\text{supp}(N)\subseteq (\text{supp}(M))^{(c)}$ and $\text{supp}(\tilde{S})=\text{supp}\left((M+N)^{-1}\right)\subseteq \text{supp}(M)$. Therefore, according to Lemma \ref{lemma:ll1}, the matrix $N$ is the unique inverse-complement of $M$. Moreover, since $M$ is sign-consistent, the equation  $\text{sign}(M_{ij})=-\text{sign}(\tilde{S}_{ij})$ holds for every $(i,j)\in\text{supp}(\tilde{S})$. This leads to the relations $\text{sign}(\Sigma_{ij}) = -\text{sign}(\tilde{S}_{ij})$ and
\begin{subequations}
	\begin{align}
	& M_{ij} = \tilde\Sigma^{\text{res}}_{ij},\\
	& |\tilde\Sigma_{ij}|> \tilde\lambda_{ij},\label{Sigma}
	\end{align}
\end{subequations}
for every $(i,j)\in\text{supp}(\tilde{S})$. Part 1 of the theorem is an immediate consequence of \eqref{Sigma}. On the other hand, based on the argument made in  the proof of Lemma \ref{l_beta}, the matrix $N$ can be obtained as
\begin{equation}\label{N2}
N_{ij} = \left\{
\begin{array}{ll}
\prod_{(m,t)\in d_{ij}}M_{mt} & \text{if}\quad  d_{ij}\not=\emptyset \ \text{and}\ (i,j)\in\left(\text{supp}(M)\right)^{(c)},\\
0 & \text{otherwise},
\end{array} 
\right.
\end{equation}
where $d_{ij}$ denotes the unique path between nodes $i$ and $j$ in $\text{supp}(\tilde{S})$ if they belong to the same connected component in $\text{supp}(\tilde{S})$, and $d_{ij}$ is empty if there is no path between nodes $i$ and $j$. Similar to the proof of Lemma \ref{l_beta}, one can show that \eqref{S_opt2} is equal to $(M+N)^{-1}$. This completes the proof of the second part of the theorem. 
~\hfill$\blacksquare$

\vspace{2mm}

\noindent \textit{\bf Proof of Theorem \ref{th2}} Based on Lemmas  \ref{l_beta} and \ref{thm:tt2}, the conditions introduced in Theorem~\ref{thm:tt1} can be reduced to conditions (2-ii) and (2-iii) in Theorem \ref{th2} if $\text{supp}(\Sigma^{\text{res}})$ is acyclic and therefore, $\mathcal{E}^{\mathrm{opt}} = \mathcal{E}^{\mathrm{res}}$. Moreover, suppose that $M$ is set to $I_d+\tilde\Sigma^{\text{res}}$, and that the matrices $N$ and $A$ are defined as \eqref{N} and \eqref{matA}, respectively. {Similar to the proof of Theorem~\ref{thm:tt1}, it can be verified that~\eqref{S_normal} satisfies all the KKT conditions for the normalized GL~\eqref{eq_corr}. Therefore, due to Lemma~\ref{l26}, ~\eqref{S_opt2} is the unique solution of the GL. The details are omitted for brevity.~\hfill$\blacksquare$}

\vspace{2mm}

%

\noindent \textit{\bf Proof of Corollary \ref{cor2}} Given $\Sigma$ and $\lambda$, the matrix $\Sigma^{\text{res}}$ can be computed in $\mathcal{O}(d^2)$. Moreover, Condition (2-i) in Theorem \ref{th2} can be checked using the Depth-First-Search algorithm, which has the time complexity of $\mathcal{O}(d^2)$ in the worst case~\cite{Ravindra93}. If the graph is cyclic, Theorem \ref{th2} cannot be used. Otherwise, we consider Condition (2-ii). For matrices with acyclic support graphs, the Cholesky Decomposition can be computed in $\mathcal{O}(d)$, from which the positive-definiteness of the matrix can be checked~\cite{Lieven15}. {The complexity of checking Condition (2-iii) is equivalent to that of finding its left and right hand sides, which can be done in $\mathcal{O}(d)$ and $\mathcal{O}(d^2)$, respectively.} Finally, since \eqref{S_opt2} can be used only if the support graph of $\Sigma^{\text{res}}$ is acyclic, one can easily verify that the complexity of obtaining $S^{\text{opt}}$ using \eqref{S_opt2} is at most $\mathcal{O}(d)$. This  completes the proof of Corollary \ref{cor2}.~\hfill$\blacksquare$

\vspace{2mm}

{The remainder of this section is devoted to proving approximation bounds for the derived closed-form solution when the acyclic assumption on the support graph of the thresholded sample covariance matrix is not necessarily acyclic.}
The shorthand notations $c$, $\text{\rm deg}$, $\mathcal{P}_{ij}$ and $P_{\max}$ will be used instead of $c(\text{\rm supp}(\Sigma^{\text{res}}))$, $\text{\rm deg}(\text{\rm supp}(\Sigma^{\text{res}}))$, $\mathcal{P}_{ij}(\text{\rm supp}(\Sigma^{\text{res}}))$ and $P_{\max}(\text{\rm supp}(\Sigma^{\text{res}}))$, respectively. {First, the approximation error of the closed-form solution for the normalized GL will be analyzed. Then, the result will be generalized to the GL via the key equality in Lemma~\ref{l26}.}
To prove Theorem~\ref{thm:approx}, the first step is to generalize the definition of the matrix $N$ in \eqref{N2} and show that this generalized matrix is an approximate inverse-consistent complement of $I_d+\tilde\Sigma^{\text{res}}$. Without loss of generality,  assume that $\text{supp}(\Sigma^{\text{res}})$ is connected. If there are disjoint components in $\text{supp}(\Sigma^{\text{res}})$, the argument made in the sequel can be used for every connected component {due to the decomposition rule for the GL (see~\cite{Mazumdar12})}. Let $M$ be equal to $I_d+\tilde\Sigma^{\text{res}}$. Consider the matrix $N$ as
\begin{equation}\label{N_new}
N_{ij} = \left\{
\begin{array}{ll}
\sum_{d_{ij}\in\mathcal{P}_{ij}}\prod_{(m,t)\in d_{ij}}M_{mt} & \text{if}\quad (i,j)\in\left(\text{supp}(M)\right)^{(c)},\\
\sum_{d_{ij}\in\mathcal{P}_{ij}\backslash\{(i,j)\}}\prod_{(m,t)\in d_{ij}}M_{mt} & \text{if}\quad (i,j)\in\left(\text{supp}(M)\right),\\
0 & \text{otherwise}.
\end{array} 
\right.
\end{equation}
It can be verified that $M+N = R$, where 
\begin{equation}\label{D}
R_{ij} = \left\{
\begin{array}{ll}
\sum_{d_{ij}\in\mathcal{P}_{ij}}\prod_{(m,t)\in d_{ij}}M_{mt} & \text{if}\quad i\not=j,\\
1 & \text{if}\quad i=j.
\end{array} 
\right.
\end{equation}
For each simple path between the pair of nodes $i$ and $j$, define its length as the multiplication of the entries of $M$ corresponding to the edges of the path. Based on this definition, $R_{ij}$ is equal to the sum of the lengths of all nonidentical simple paths between nodes $i$ and $j$ in $\text{supp}(M)$.
Denote $d_{ij}^s$ as any shortest path between nodes $i$ and $j$ in $\text{supp}(M)$ (recall that $\text{supp}(M)$ is unweighted), and let $R^s$ be given by
\begin{equation}\notag
R^s_{ij} = \left\{
\begin{array}{ll}
\prod_{(m,t)\in d^s_{ij}}M_{mt} & \text{if}\quad i\not=j,\\
1 & \text{if}\quad i=j.
\end{array} 
\right.
\end{equation}
Note that $R^s$ collects the length of the shortest path between  every two nodes in $\text{supp}(M)$.
The following lemmas are crucial to prove Theorem~\ref{thm:approx}.


\begin{lemma}\label{lemma:cycle}
	Given two nodes $i$ and $j$ in $\mathrm{supp}(\Sigma^{\mathrm{res}})$, suppose that $\mathcal{P}_{ij}\backslash d_{ij}^s$ is non-empty. Then, the length of every path $d_{ij}$ in $\mathcal{P}_{ij}\backslash d_{ij}^s$ is at least $\lceil c/2\rceil$.
\end{lemma}

\begin{proof}
	Consider a path $d_{ij}$ in $\mathcal{P}_{ij}\backslash d_{ij}^s$. The subgraph $d_{ij}\cup d_{ij}^s$ has a cycle. Since the length of this cycle is at least $c$, the segment of this cycle that resides in $d_{ij}$ should have the length of at least $\lceil c/2\rceil$; otherwise $d_{ij}^s$ is not the shortest path between the nodes $i$ and $j$. This implies that the length of $d_{ij}$ is at least $\lceil c/2\rceil$.
\end{proof}

\begin{lemma}
	Let $M$ be equal to $I_d+\tilde\Sigma^{\mathrm{res}}$. The inequalities
	\begin{subequations}
		\begin{align}
		& \left|R_{ij}-R^s_{k'j}M_{ik'}\right|\leq \left(|\mathcal{P}_{ij}|_0-1\right)\left({\|\tilde{\Sigma}^{\mathrm{res}}\|_{\max}}\right)^{\lceil \frac{c}{2}\rceil},\label{la}\\
		& \left|R_{kj}-R^s_{k'j}M_{ik'}M_{ik}\right|\leq \left(|\mathcal{P}_{kj}|_0\!-\!1\right)\left({\|\tilde{\Sigma}^{\mathrm{res}}\|_{\max}}\right)^{\lceil \frac{c}{2}\rceil-1}\label{lb}
		\end{align}
	\end{subequations}
	hold if $i\not=j$,  where $k'$ is the node adjacent to $i$ in $d^s_{ij}$ and $k\in\mathcal N(i)\backslash k'$.
\end{lemma}
\begin{proof}
	First, we show the validity of \eqref{la}. Due to \eqref{D}, one can write
	\begin{equation}\label{eq70}
	R_{ij} = R_{ij}^s+\sum_{d_{ij}\in\mathcal{P}_{ij}\backslash d_{ij}^s}\prod_{(m,t)\in d_{ij}}M_{mt}.
	\end{equation}
	If $\mathcal{P}_{ij}\backslash d_{ij}^s$ is empty,  then the equation $R_{ij}=R^s_{k'j}M_{ik'}$ and therefore \eqref{la} hold. Now, assume that $\mathcal{P}_{ij}\backslash d_{ij}^s$ is not empty. Due to Lemma \ref{lemma:cycle}, we have
	{\begin{equation}\notag
		-\left({\|\tilde{\Sigma}^{\mathrm{res}}\|_{\max}}\right)^{\lceil \frac{c}{2} \rceil}\leq\prod_{(m,t)\in d_{ij}}M_{mt}\leq \left({\|\tilde{\Sigma}^{\mathrm{res}}\|_{\max}}\right)^{\lceil \frac{c}{2} \rceil},
		\end{equation}}
	for every $d_{ij}\in \mathcal{P}_{ij}\backslash d_{ij}^s$. The above inequalities, together with \eqref{eq70} and the equation $R_{ij}^s = R^s_{k'j}M_{ik'}$, result in \eqref{la}. To prove \eqref{lb}, define $\hat d_{kj}$ as $d_{ij}^s\cup \{(i,k)\}$ (note that $\hat d_{kj}$ is not necessarily equal to $d_{kj}^s$). It yields that
	\begin{equation}\label{eq72}
	R_{kj} = R_{ij}^sM_{ik}+\sum_{d_{kj}\in\mathcal{P}_{kj}\backslash \hat d_{kj}}\prod_{(m,t)\in d_{kj}}M_{mt}.
	\end{equation}
	In light of Lemma~\ref{lemma:cycle}, the length of every path $d_{kj}\in\mathcal{P}_{kj}\backslash \hat d_{kj}$ is lower bounded by $\lceil c/2\rceil-1$. This implies that 
	{\begin{equation}\label{eq73}
		-\left({\|\tilde{\Sigma}^{\mathrm{res}}\|_{\max}}\right)^{\lceil \frac{c}{2} \rceil-1}\leq\prod_{(m,t)\in d_{ij}}M_{mt}\leq \left({\|\tilde{\Sigma}^{\mathrm{res}}\|_{\max}}\right)^{\lceil \frac{c}{2} \rceil-1},
		\end{equation} }
	for every $d_{kj}\in\mathcal{P}_{kj}\backslash \hat d_{kj}$. Combining $R_{ij}^sM_{ik} = R^s_{k'j}M_{ik'}M_{ik}$ with \eqref{eq72} and \eqref{eq73} leads to the inequality \eqref{lb}.\end{proof}

\begin{lemma}\label{deg_ineq}
	The following inequality holds
	{\begin{equation}\notag
		\frac{\mathrm{deg}}{1-{\|\tilde{\Sigma}^{\mathrm{res}}\|_{\max}^2}}\leq \delta,
		\end{equation}}
	where $\delta$ defined as \eqref{delta}.
\end{lemma}
\begin{proof}
	The proof is straightforward and is omitted for brevity.
\end{proof}
\vspace*{2mm}

\noindent \textit{\bf Proof of Theorem~\ref{thm:approx}} 
{Consider the normalized GL and define the following explicit formula for $\tilde{A}$
	\begin{equation}\label{A_normal}
	\tilde{A}_{ij} = \left\{
	\begin{array}{ll}
	1+\underset{
		(i,m)\in\mathcal{E}^{\text{opt}}
	}{\sum}\frac{({\tilde\Sigma^{\mathrm{res}}_{im}})^2}{1-({\tilde\Sigma^{\mathrm{res}}_{im}})^2} & \text{if}\quad i=j,\\
	\frac{-\tilde\Sigma^{\mathrm{res}}_{ij}}{1-({\tilde\Sigma^{\mathrm{res}}_{ij}})^2} & \text{if}\quad(i,j)\in \mathcal{E}^{\mathrm{res}},\vspace{1mm}\\
	
	0 & \text{otherwise}.
	\end{array} 
	\right.
	\end{equation}
	Let $M$ be equal to $I_d+\tilde\Sigma^{\text{res}}$. Furthermore, define
	\begin{equation}\notag
	\tilde{\epsilon} = \delta\cdot (P_{\max}(\text{\rm supp}(\Sigma^{\mathrm{res}}))-1)\cdot \left(\|\tilde{\Sigma}^{\mathrm{res}}\|_{\max}\right)^{\left\lceil\frac{c(\text{\rm supp}(\Sigma^{\mathrm{res}}))}{2}\right\rceil}.
	\end{equation}
	In order to prove the theorem, we use the matrix $N$ defined in \eqref{N_new}, and first show that $M+N$ is an $\tilde\epsilon$-relaxed inverse of $\tilde{A}$ and  that the pair $(\tilde{A}, M+N)$ satisfies the $\tilde\epsilon$-relaxed KKT conditions.
	
	Consider the matrix $T$ defined as $T = \tilde{A}(M+N)$ and recall that $M+N = R$. One can write
	\begin{equation}\label{Tii}
	T_{ii} = \sum_{m = 1}^{d} \tilde{A}_{im}R_{mi} = \left(1+\sum_{m\in\mathcal N(i)}\frac{{M_{im}}^2}{1-{M_{im}}^2}\right)-\sum_{m\in\mathcal N(i)}\frac{{M_{im}}}{1-{M_{im}}^2}R_{mi}.
	\end{equation}
	Note that since $\{(m,i)\}\in\mathcal{P}_{mi}$ for every $m\in\mathcal N(i)$, we have
	\begin{equation}\notag
	R_{mi} = M_{mi}+\sum_{d_{mi}\in\mathcal{P}_{mi}\backslash \{(m,i)\}}\prod_{(r,t)\in d_{mi}}M_{rt}.
	\end{equation}
	If $\mathcal{P}_{mi}\backslash \{(m,i)\}$ is empty, then $R_{mi} = M_{mi}$ and $T_{ii} = 1$. Otherwise, since the length of the minimum-length cycle is $c$, the length of every path $d_{mi}\in\mathcal{P}_{mi}\backslash \{(m,i)\}$ is at least $c-1$. This yields that
	\begin{equation}\label{c-1}
	M_{mi} - (|\mathcal{P}_{mi}|_0-1)\left({\|\tilde{\Sigma}^{\mathrm{res}}\|_{\max}}\right)^{c-1} \leq R_{mi} \leq M_{mi} + (|\mathcal{P}_{mi}|_0-1)\left({\|\tilde{\Sigma}^{\mathrm{res}}\|_{\max}}\right)^{c-1}.
	\end{equation}
	Combining \eqref{c-1} and \eqref{Tii} leads to
	\begin{equation}\label{Tii_relax}
	|T_{ii}-1|\!\leq\!\! \left(|\mathcal{P}_{mi}|_0-1\!\right)\!\left({\|\tilde{\Sigma}^{\mathrm{res}}\|_{\max}}\right)^{c-1}\!\left(\sum_{m\in\mathcal N(i)}\frac{{M_{im}}}{1-{M_{im}}^2}\right) \!\leq\! \text{\rm deg}(P_{\max}-1)\frac{{\|\tilde{\Sigma}^{\mathrm{res}}\|_{\max}^{c}}}{1-\|\tilde{\Sigma}^{\mathrm{res}}\|_{\max}^{2}}\leq \tilde\epsilon,
	\end{equation}
	where the last inequality is due to Lemma~\ref{deg_ineq} and the fact that $\lceil \frac{c}{2}\rceil\leq c$ for $c\geq 3$. Now, consider $T_{ij}$ for a pair $(i,j)$ such that $i\not=j$. We have
	\begin{equation}\label{Tij}
	T_{ij} = \sum_{m = 1}^{d} \tilde{A}_{im}R_{mj} = \left(1+\sum_{m\in\mathcal N(i)}\frac{{M_{im}}^2}{1-{M_{im}}^2}\right)R_{ij}-\sum_{m\in\mathcal N(i)}\frac{{M_{im}}}{1-{M_{im}}^2}R_{mj}.
	\end{equation}
	According to Lemma~\ref{lemma:cycle}, one can write
	\begin{subequations}
		\begin{align}
		& R^s_{m'j}M_{im'}-\left(|\mathcal{P}_{ij}|_0-1\right)\left({\|\tilde{\Sigma}^{\mathrm{res}}\|_{\max}}\right)^{\lceil \frac{c}{2}\rceil}\leq R_{ij}\leq R^s_{m'j}M_{im'}+\left(|\mathcal{P}_{ij}|_0-1\right)\left({\|\tilde{\Sigma}^{\mathrm{res}}\|_{\max}}\right)^{\lceil \frac{c}{2}\rceil},\label{la2}\\
		&R^s_{m'j}M_{im'}M_{im}\!-\!\left(|\mathcal{P}_{mj}|_0\!-\!1\right)\!\left({\|\tilde{\Sigma}^{\mathrm{res}}\|_{\max}}\right)^{\lceil \frac{c}{2}\rceil-1} \!\leq\! R_{mj}\nonumber\\
		&\hspace{7.1cm}\!\leq\! R^s_{m'j}M_{im'}M_{im}\!+\!\left(|\mathcal{P}_{mj}|_0\!-\!1\right)\!\left({\|\tilde{\Sigma}^{\mathrm{res}}\|_{\max}}\right)^{\lceil \frac{c}{2}\rceil-1},\label{lb2}
		\end{align}
	\end{subequations}
	where $m'$ is the node adjacent to $i$ in $d_{ij}^s$ and $m\in\mathcal N(i)\backslash m'$. Note that if $\mathcal N(i)\backslash m'$ is empty, then $R_{ij} = R^s_{m'j}M_{im'}$ and $R_{mj} = R^s_{m'j}M_{im'}\tilde\Sigma^{\text{res}}_{im}$. In this case, an argument similar to the proof of Lemma \ref{l_beta} can be made to show that $T_{ij} = 0$. Now, assume that $\mathcal N(i)\backslash m'$ is not empty. One can write
	\begin{align}\label{Tij_relax}
	\left|T_{ij} - F_{ij} \right|
	\overset{(a)}{=} |T_{ij}|&\overset{(b)}{\leq} \tilde\epsilon,
	\end{align}
	where
	\begin{align}
	F_{ij} =& \left(\frac{1}{1-{M_{im'}}^2}+\sum_{m\in\mathcal N(i)\backslash m'}\frac{{M_{im}}^2}{1-{M_{im}}^2}\right)R^s_{m'j}M_{im'}-\frac{{M_{im'}}}{1-{M_{im'}}^2}R_{m'j}^s
	\nonumber\\
	&-\sum_{m\in\mathcal N(i)\backslash m'}\frac{{M_{im}}^2}{1-{M_{im}}^2}R^s_{m'j}M_{im'}M_{im}.\notag
	\end{align}
	%
	Note that the relation (a) can be verified by the fact that $F_{ij} = 0$ and the inequality (b) is obtained by combining \eqref{Tij} with \eqref{la2} and \eqref{lb2}. The inequalities
	\eqref{Tii_relax} and \eqref{Tij_relax} imply that $M+N$ is an $\tilde\epsilon$-relaxed inverse of $\tilde{A}$.
	
	Now, it will be shown that the pair $(\tilde{A}, M+N)$ satisfies the $\tilde\epsilon$-relaxed KKT conditions. Note that $M_{ii}+N_{ii} = M_{ii} = \tilde\Sigma_{ii}$ and, hence, \eqref{eps1} is satisfied. To prove \eqref{eps2},  since $\text{sign}(\tilde{A}_{ij}) = -\text{sign}(M_{ij}) = -\text{sign}(\tilde\Sigma_{ij})$, it can be concluded that
	\begin{equation}\notag
	M_{ij}+N_{ij} = (\tilde\Sigma_{ij}-\tilde\lambda_{ij}\times\text{sign}(\Sigma_{ij}))+N_{ij} = (\tilde\Sigma_{ij}+\tilde\lambda_{ij}\times\text{sign}(\tilde{A}_{ij}))+N_{ij},
	\end{equation}
	for every $(i,j)$ such that $i\not=j$ and $\tilde{A}_{ij}\not=0$. Due to the definition of $N$ and the fact that $(i,j)\in\text{supp}(M)$, we have $|N_{ij}|\leq (P_{\max}-1)\left({\|\tilde{\Sigma}^{\mathrm{res}}\|_{\max}}\right)^{c-1}$. Hence,
	\begin{equation}\notag
	|M_{ij}+N_{ij}-(\tilde\Sigma_{ij}+\tilde\lambda_{ij}\times\text{sign}(\tilde{A}_{ij}))|\leq \epsilon,
	\end{equation}
	for every $(i,j)$ such that $i\not=j$ and $\tilde{A}_{ij}\not=0$. Therefore, the pair $(\tilde{A}, M+N)$ satisfies \eqref{eps2}. Finally, consider a pair $(i,j)$ such that $i\not=j$ and $\tilde{A}_{ij}=0$. One can write
	\begin{equation}\notag
	M_{ij}+N_{ij} = R_{ij}^s+\sum_{d_{ij}\in\mathcal{P}_{ij}\backslash d_{ij}^s}\prod_{(m,t)\in d_{ij}}\tilde\Sigma^{\text{res}}_{mt}.
	\end{equation}
	If $\mathcal{P}_{ij}\backslash d_{ij}^s$ is empty, a set of inequalities similar to \eqref{eq39} and \eqref{eq44} can be obtained to prove  \eqref{eps3}. Now, assume that $\mathcal{P}_{ij}\backslash d_{ij}^s$ is not empty. The length of $d_{ij}^s$ is at least 2 since there is no direct edge between nodes $i$ and $j$. Hence, $|R_{ij}^s|\leq {\|\tilde{\Sigma}^{\mathrm{res}}\|_{\max}^2}$. Furthermore, due to Lemma \eqref{lemma:cycle}, the length of every path $d_{ij}\in\mathcal{P}_{ij}\backslash d_{ij}^s$ is at least $\lceil c/2\rceil$. This leads to
	\begin{align}
	|M_{ij}+N_{ij}|\leq {\|\tilde{\Sigma}^{\mathrm{res}}\|_{\max}^2}&+(P_{\max}-1)\left({\|\tilde{\Sigma}^{\mathrm{res}}\|_{\max}}\right)^{\lceil\frac{c}{2}\rceil}\nonumber\\
	&\leq \underset{
		\begin{subarray}{c}
		k\not=l\\
		(k,l)\not\in\mathrm{supp}(\Sigma^{\mathrm{res}})
		\end{subarray}
	} {\min}(\tilde{\lambda}_{kl}-|\tilde{\Sigma}^{\mathrm{res}}_{kl}|)+(P_{\max}-1)\left({\|\tilde{\Sigma}^{\mathrm{res}}\|_{\max}}\right)^{\lceil\frac{c}{2}\rceil}\nonumber\\
	&\leq \tilde{\lambda}_{ij}-|\tilde{\Sigma}^{\mathrm{res}}_{ij}|+(P_{\max}-1)\left({\|\tilde{\Sigma}^{\mathrm{res}}\|_{\max}}\right)^{\lceil\frac{c}{2}\rceil},\notag
	\end{align}
	where the last inequality follows from Condition (2-ii) in the theorem. Therefore,
	\begin{equation}
	\begin{aligned}
	|M_{ij}\!+\!N_{ij}\!-\!\tilde\Sigma_{ij}|\!\leq\! |M_{ij}\!+\!N_{ij}|\!+\!|\tilde\Sigma_{ij}|&\leq\tilde{\lambda}_{ij}-|\tilde{\Sigma}^{\mathrm{res}}_{ij}|+|\tilde{\Sigma}^{\mathrm{res}}_{ij}|+(P_{\max}-1)\left({\|\tilde{\Sigma}^{\mathrm{res}}\|_{\max}}\right)^{\lceil\frac{c}{2}\rceil}\\
	&\leq\tilde{\lambda}_{ij}+\tilde{\epsilon}.\notag
	\end{aligned}
	\end{equation} 
	This shows that $(\tilde{A}, M+N)$ indeed satisfies the $\tilde{\epsilon}$-relaxed KKT conditions for the normalized GL. Finally, we consider the explicit solution $A$ defined as~\eqref{S_opt3}. The following statements hold:
	\begin{itemize}
		\item[1.] the matrix $D^{1/2}(M+N)D^{1/2}$ is $\epsilon$-relaxed inverse of $A$. To see this, note that
		\begin{align}
		A\left(D^{1/2}(M+N)D^{1/2}\right) &= D^{-1/2}\tilde{A}D^{-1/2}D^{1/2}(M+N)D^{1/2}\nonumber\\
		& = D^{-1/2}TD^{1/2}\nonumber\\
		& = I_d+E,\notag
		\end{align}
		where $\|E\|_{\max}\leq \sqrt{\frac{\Sigma_{\max}}{\Sigma_{\min}}}\tilde{\epsilon}\leq\epsilon$.
		\item[2.] The pair $(A,D^{1/2}(M+N)D^{1/2})$ satisfies the $\epsilon$-relaxed KKT conditions. Note that it is already shown that $(\tilde{A},M+N)$ satisfies the following inequalities
		\begin{subequations}\label{optcon4}
			\begin{align}
			& (M+N)_{ij} = \tilde\Sigma_{ij}\hspace{3cm} &&\text{if}\quad i=j,\label{eps14}\\
			& \left|(M+N)_{ij} - \left(\tilde\Sigma_{ij}+\tilde\lambda_{ij}\times \mathrm{sign}(\tilde{A}_{ij})\right)\right|\leq \tilde\epsilon&& \text{if}\quad \tilde{A}_{ij}\not=0,\label{eps24}\\
			& \left|(M+N)_{ij}-\tilde\Sigma_{ij}\right| \leq \tilde\lambda_{ij}+\tilde\epsilon&& \text{if}\quad \tilde{A}_{ij}=0.\label{eps34}
			\end{align}
		\end{subequations}
		Replacing $M+N$ with $D^{1/2}(M+N)D^{1/2}$ and modifying~\eqref{optcon4} accordingly, one can verify that $(A,D^{1/2}(M+N)D^{1/2})$ satisfies $\epsilon$-relaxed KKT conditions for the GL, where
		\begin{equation}\notag
		\epsilon = \max\left\{\Sigma_{\max},\sqrt{\frac{\Sigma_{\max}}{\Sigma_{\min}}}\right\}\tilde\epsilon.
		\end{equation}
	\end{itemize}
	This completes the proof.~\hfill$\blacksquare$}

\vspace*{2mm}
\noindent \textit{\bf Proof of Theorem~\ref{cor:perturb}}
{Due to Theorem \ref{thm:approx}, the equation
	\begin{equation}\label{eq:app_inverse}
	D^{1/2}(M+N)D^{1/2} = A^{-1}+A^{-1}E
	\end{equation}
	holds for every $\lambda$ greater than or equal to $\lambda_0$, where $\|E\|_{\max}\leq \epsilon$. Since the pair $(A, D^{1/2}(M+N)D^{1/2})$ satisfies the $\epsilon$-relaxed KKT conditions, it follows from  \eqref{eq:app_inverse} that
	\begin{subequations}\label{optcon3}
		\begin{align}
		& (A)^{-1}_{ij} = \Sigma_{ij}-(A^{-1}E)_{ij} = \hat{\Sigma}_{ij}\hspace{3cm} &&\text{if}\quad i=j,\label{eps12}\\
		& (A)^{-1}_{ij} = \underbrace{\Sigma_{ij}+t_{ij}\epsilon-(A^{-1}E)_{ij}}_{\hat{\Sigma}_{ij}}+\lambda\times \text{sign}(A_{ij}) && \text{if}\quad A_{ij}\not=0,\label{eps22}\\
		& \underbrace{\Sigma_{ij}+s_{ij}\epsilon\!-\!(A^{-1}E)_{ij}}_{\hat{\Sigma}_{ij}}-\lambda\leq (A)^{-1}_{ij} \leq \underbrace{\Sigma_{ij}+s_{ij}\epsilon\!-\!(A^{-1}E)_{ij}}_{\hat{\Sigma}_{ij}}+\lambda&& \text{if}\quad A_{ij}=0,\label{eps32}
		\end{align}
	\end{subequations}
	for some numbers $t_{ij}$ and $s_{ij}$ in the interval $[-1,1]$. To complete the proof, it suffices to show that the matrix $F$ defined as
	\begin{equation}\label{F}
	\Sigma_{ij}-\hat{\Sigma}_{ij} = F_{ij} = \left\{
	\begin{array}{ll}
	-(A^{-1}E)_{ij} & \text{if}\quad i=j,\\
	t_{ij}\epsilon\!-\!(A^{-1}E)_{ij} & \text{if}\quad A_{ij} \not= 0,\\
	s_{ij}\epsilon\!-\!(A^{-1}E)_{ij} & \text{if}\quad A_{ij} = 0
	\end{array} 
	\right.
	\end{equation}
	satisfies the inequality $\|F\|_{2}\leq {d_{\max}}\left(1/{\mu_{\min}(A)}+1\right)\epsilon$. To this end, it is enough to prove that $\|A^{-1}E\|_2\leq ({d_{\max}}/{\mu_{\min}(A)})\epsilon$, since $\|F-A^{-1}E\|_2\leq d_{\max}(A)\epsilon$. 
	One can write
	\begin{equation}
	\begin{aligned}\notag
	\|A^{-1}E\|_2\leq \|A^{-1}\|_2 \|E\|_2\leq d_{\max}(A)\|A^{-1}\|_2 \|E\|_{\max} = \left(\frac{d_{\max}(A)}{\mu_{\min}(A)}\right)\epsilon,
	\end{aligned}
	\end{equation}
	which shows the validity of~(\ref{eq:perturb}).
	
	Next, we prove the inequality~\eqref{optgap}. The following chain of inequalities hold
	\begin{align}
	-\log\det(A)+\mathrm{trace}(\hat{\Sigma}A)+\lambda\|A\|_{1,\mathrm{off}}&=\underbrace{-\log\det(A)+\mathrm{trace}({\Sigma}A)+\lambda\|A\|_{1,\mathrm{off}}}_{f(A)}\nonumber\\
	&\ \ \ + \mathrm{trace}((\hat{\Sigma}-\Sigma)A)\nonumber\\
	&\overset{(a)}{\leq} -\log\det(S^{\mathrm{opt}})+\mathrm{trace}(\hat{\Sigma}S^{\mathrm{opt}})+\lambda\|S^{\mathrm{opt}}\|_{1,\mathrm{off}}\nonumber\\
	& = \underbrace{-\log\det(S^{\mathrm{opt}})+\mathrm{trace}({\Sigma}S^{\mathrm{opt}})+\lambda\|S^{\mathrm{opt}}\|_{1,\mathrm{off}}}_{f^*}\nonumber\\
	&\ \ \ + \mathrm{trace}((\hat{\Sigma}-\Sigma)S^{\mathrm{opt}}),\notag
	\end{align}
	where $(a)$ is due to the fact that $A$ is optimal for the GL with the perturbed sample covariances. This implies that
	\begin{align}
	f(A)-f^*&\leq \mathrm{trace}(({\tilde\Sigma-\Sigma})(S^{\mathrm{opt}}-A))\nonumber\\
	&\leq \|{\tilde\Sigma-\Sigma}\|_2(\|S^{\mathrm{opt}}\|_2+\|A\|_2)\nonumber\\
	&\leq \left(\mu_{\max}(A)+\mu_{\max}(S^{\mathrm{opt}})\right)d_{\max}(A)\left(\frac{1}{\mu_{\min}(A)}+1\right)\epsilon\notag.
	\end{align}
	~\hfill$\blacksquare$}

\end{document}